\documentclass{article}
\PassOptionsToPackage{hyphens}{url}
\usepackage{iclr2027_conference}
\usepackage[T1]{fontenc}
\usepackage{newtxtext}

\usepackage{hyperref}
\usepackage{graphicx}
\usepackage{adjustbox}
\usepackage{caption}
\usepackage{subcaption}
\usepackage{multirow}
\usepackage{xcolor}
\usepackage{algorithm}
\usepackage{algorithmicx}
\usepackage{algpseudocode}
\usepackage{amsthm}
\usepackage{amssymb}
\usepackage{amsmath}
\newtheorem{theorem}{Theorem}
\newtheorem{definition}{Definition}

\usepackage{wrapfig}   % 放在导言区

\usepackage{booktabs}
\renewcommand{\thesection}{\arabic{section}}
\renewcommand{\thesubsection}{\thesection.\arabic{subsection}}

\title{Convex-Hull-Neighborhood Smooth Dual Generalization:
Controlling Local Correction Propagation in Offline RL}

\author{%
\parbox{0.98\textwidth}{%
\centering
{\normalfont\bfseries\small
Yi Yang\textsuperscript{1},
Zhennan Chen\textsuperscript{2},
Mingfeng Lv\textsuperscript{1},\\[2pt]
HanleiLi\textsuperscript{1},
Zhengsen Ruan\textsuperscript{3,*},
Lvqing Yang\textsuperscript{1,*}
}\\[7pt]
{\normalfont\small
\textsuperscript{1}Xiamen University \quad
\textsuperscript{2}Nanjing University \quad
\textsuperscript{3}Zhejiang Gongshang University
}
}}

\hypersetup{
  pdftitle={Convex-Hull-Neighborhood Smooth Dual Generalization: Controlling Local Correction Propagation in Offline RL},
  pdfauthor={Yi Yang, Zhennan Chen, Mingfeng Lv, HanleiLi, Zhengsen Ruan, Lvqing Yang}
}

\iclrfinalcopy

\begin{document}

\maketitle
\lhead{Preprint}

\begingroup
\renewcommand{\thefootnote}{\fnsymbol{footnote}}
\footnotetext[1]{Corresponding authors}
\endgroup

\begin{abstract}
Offline reinforcement learning (offline RL) can benefit from nearby
out-of-distribution (OOD) actions, but estimation errors at these actions may be
amplified by bootstrapping. Existing regularization and local-generalization
methods control either the admissible OOD region or the influence of generalized
targets, often through separate mechanisms. We propose Convex Hull Neighborhood
Smooth Dual Generalization (CSDG), which expresses the Bellman backup as an
in-sample value target plus a CHN-local correction. This formulation makes the
generalized contribution explicit and separates it from the in-sample reference
path. The correction is obtained by smoothing in-sample-oriented and
OOD-oriented candidates sampled at different perturbation radii. A mixture
coefficient $\lambda$ scales its contribution to each backup, while the recursive
discount remains $\gamma$. Under boundedness and fixed perturbation kernels,
we derive an exact one-step correction identity, a time-varying iterate bound,
and a fixed-point bound that depends only on the branch discrepancy at the
fixed point. We further characterize the implicit policies induced by the
idealized operators and give a conditional non-degradation criterion. The
practical algorithm approximates these quantities using asymmetric bounded noise
and expectile regression, without exact support classification or an additional
pessimistic OOD penalty. Experiments on Gym-MuJoCo and AntMaze show strong
aggregate performance and stable value estimation. Code is available at: \url{https://github.com/YOUNG-fnxm/CSDG}
\end{abstract}


\section{Introduction}

While reinforcement learning (RL) has achieved success in sequential decision-making \citep{schrittwieser2020mastering}, standard approaches rely on active interaction. To address this limitation, offline RL has emerged as a data-driven paradigm that learns optimal policies entirely from previously gathered datasets \citep{levine2020offline}. This approach decouples learning from online interaction, enabling the direct exploitation of large-scale logs of past experience \citep{qu2023hokoff}. However, the distribution shift between the behavior and target policies presents a major challenge. This shift causes overestimation of OOD actions, leading to extrapolation errors that are exacerbated by bootstrapping. Consequently, these errors distort value estimates and prevent convergence to the optimal policy \citep{fujimoto2019off}.

%Standard
Prior offline RL methods mitigate OOD value errors through policy constraints
~\citep{wu2019behavior,kumar2019stabilizing,chen2020bail,
peng2019advantage,ghasemipour2021emaq,kumar2020conservative},
value penalization
~\citep{ma2021conservative,bai2022pessimistic,yang2022rorl},
or in-sample learning
~\citep{kostrikov2021offline,xiao2023sample,haarnoja2018soft,
xu2023offline}. More recent work seeks a middle ground through controlled OOD
generalization, including MCQ~\citep{lyu2022mildly},
DOGE~\citep{li2022data}, and SQOG~\citep{yao2025offline}, or through local
generalization and propagation control, such as
HUBL~\citep{geng2023improving} and DMG~\citep{DMG}. These methods regulate where generalization occurs or how it propagates. 
We instead approach this from a different perspective: 
isolating the local correction from the in-sample target and 
explicitly controlling how much of it enters each Bellman backup. 
This leads to the following question:
\textbf{\emph{Can we exploit locally generalized values while explicitly 
controlling their contribution to each recursive update under geometric 
approximation bounds?}}
\begin{table*}[t]
  \centering
  \small
  \renewcommand{\arraystretch}{1.3}
  \setlength{\tabcolsep}{8pt}
  \resizebox{0.9\textwidth}{!}{%
    \begin{tabular}{l c c c}
      \toprule
      Method &
      OOD Exploitation &
      Propagation Mode &
      Control Mechanism \\
      \midrule
      $\mathcal{X}$QL / IQL &
      None &
      In-Sample &
      In-Sample Reference \\
      CQL &
      Penalized &
      Standard &
      Pessimistic Penalty \\
      SQOG &
      CHN-Guided &
      Standard &
      CHN Regularization \\
      \midrule
      CSDG (Ours) &
      \textbf{Smoothed Local Values} &
      \textbf{Dual Target} &
      \textbf{Correction Scaling} \\
      \bottomrule
    \end{tabular}%
  }
\caption{Comparison of representative offline RL methods by OOD use, backup structure, and control mechanism.}
  \label{tab:comparison_ood_safety}
\end{table*}
\begin{wrapfigure}{r}{0.45\textwidth}   % r表示靠右，l表示靠左；第二个参数是环绕宽度
    \centering
    \includegraphics[width=\linewidth]{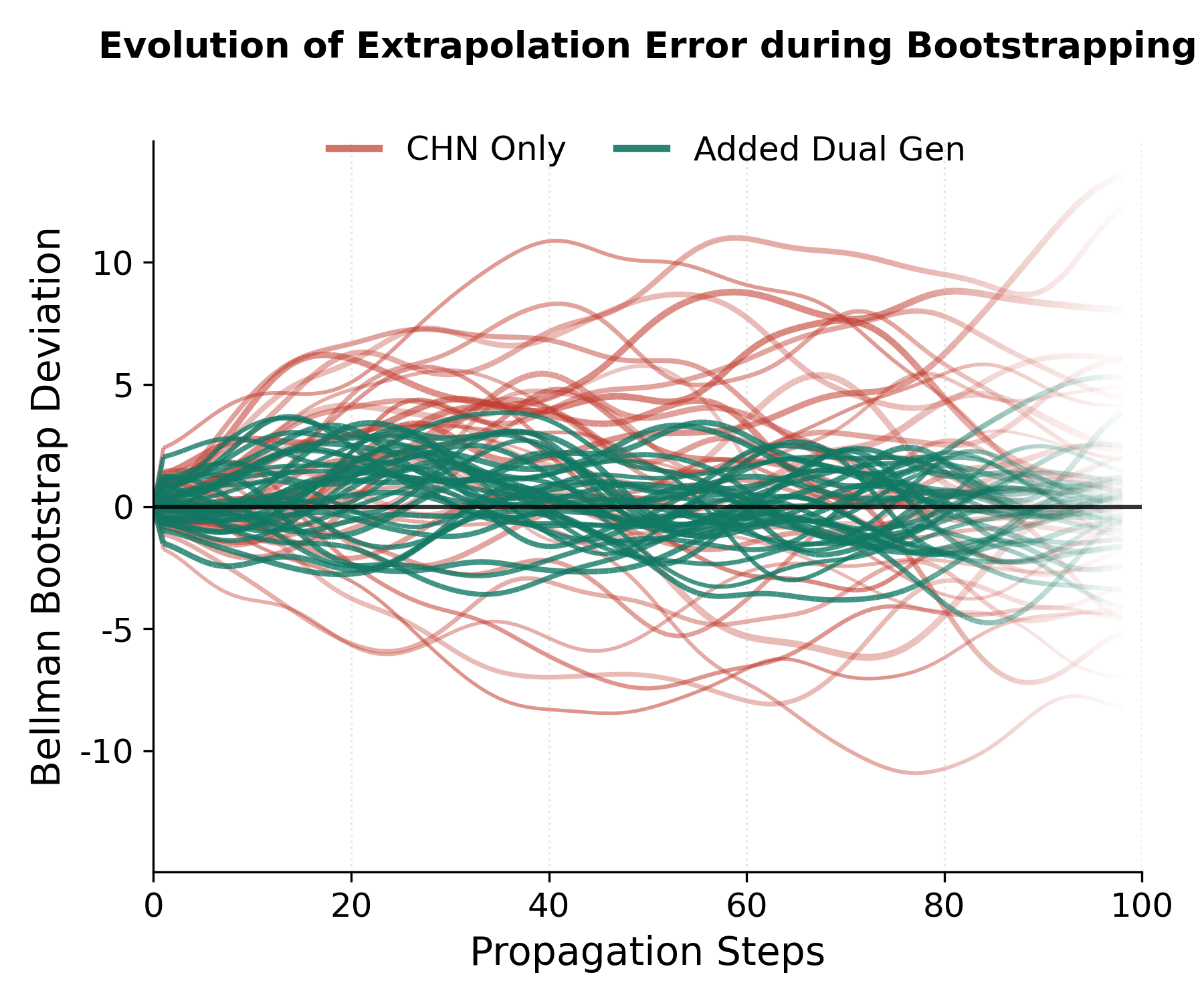}
    \caption{Schematic illustration of discrepancy propagation under repeated
    bootstrapping. The CSDG correction reduces the contribution of locally
    generalized values while retaining the same recursive discount.}
    \label{fig:int}
    \vspace{-1em}
\end{wrapfigure}
To answer this question, we build on the Convex Hull and its Neighborhood
(CHN)~\citep{yao2025offline} and introduce CHN Smooth Dual Generalization
(CSDG). CSDG organizes local smoothing and propagation control through one
correction to the in-sample Bellman target. Asymmetric perturbations first
produce an in-sample-oriented candidate and an OOD-oriented candidate at two
sampling radii. Their smoothed value defines a locally generalized target. The
difference between this target and the in-sample value is then treated as a
local correction, and $\lambda$ determines how much of it enters each recursive
update. In short, CHN specifies the local geometry, smoothing constructs the
generalized value, and $\lambda$ gates its contribution to bootstrapping.
Table~\ref{tab:comparison_ood_safety} contrasts this design with representative
approaches.

Our analysis characterizes the perturbation-averaged correction at the
operator, iterative, and policy levels. We first derive an exact one-step
decomposition and a time-varying bound that tracks the branch discrepancy along
the update sequence. We then establish fixed-point and a posteriori residual
bounds, with a nearest-anchor decomposition linking the branch discrepancy to
geometric and reference components. Finally, an implicit-policy representation
yields a signed-advantage criterion for comparing CSDG with the in-sample
reference. Empirically, CSDG performs strongly
on standard offline RL benchmarks~\citep{fu2020d4rl}, including Gym--MuJoCo
locomotion and the challenging AntMaze tasks. The two benchmark families cover
rather different control problems, providing a broader view of the method beyond
a single domain.

To summarize, our main contributions are as follows:

\begin{itemize}
    \item We formulate CSDG as an in-sample Bellman backup augmented by a
    CHN-local correction, with $\lambda$ directly controlling how much of the
    locally generalized value enters each recursive update.
    \item We derive exact one-step, time-varying iterate, fixed-point, and
a posteriori residual bounds, together with a conditional performance
criterion for the implicit policies induced by the idealized operators.
    \item We develop a practical algorithm that combines asymmetric two-scale
    smoothing with an expectile-referenced correction in the Bellman objective,
    without an additional pessimistic OOD penalty.
\end{itemize}
\section{Preliminaries}
\paragraph{RL.}
An RL problem is modeled as a Markov decision process (MDP)
$\mathcal M=(\mathcal S,\mathcal A,P,R,\gamma,d_0)$, where $\mathcal S$ and
$\mathcal A$ are the state and action spaces,
$P:\mathcal S\times\mathcal A\to\Delta(\mathcal S)$ is the transition kernel,
$R:\mathcal S\times\mathcal A\to[0,R_{\max}]$ is the bounded reward function,
$\gamma\in[0,1)$ is the discount factor, and $d_0$ is the initial-state
distribution~\citep{sutton1998reinforcement}. A policy
$\pi:\mathcal S\to\Delta(\mathcal A)$ seeks to maximize
\begin{equation}
J(\pi)=
\mathbb{E}_{\substack{
s_0\sim d_0,\,
a_t\sim\pi(\cdot|s_t),\,
s_{t+1}\sim P(\cdot|s_t,a_t)}}
\left[
\sum_{t=0}^{\infty}\gamma^{t}R(s_t,a_t)
\right].
\end{equation}
The corresponding value functions are
$V^\pi(s)=
\mathbb E_\pi[
\sum_{t=0}^\infty\gamma^tR(s_t,a_t)\mid s_0=s]$
and
$Q^\pi(s,a)=
\mathbb E_\pi[
\sum_{t=0}^\infty\gamma^tR(s_t,a_t)
\mid s_0=s,a_0=a]$.

\paragraph{Offline RL.}
Offline RL learns from a fixed dataset
$\mathcal D=\{(s_i,a_i,r_i,s'_i,d_i)\}_{i=0}^{n-1}$ collected by an unknown
behavior policy $\beta(\cdot|s)$, without further environment interaction
~\citep{lange2012batch,levine2020offline}. Because exact state repetitions are
rare in continuous control, for $1\le k\le n$ let $\mathcal N_k(s)$ index the
$k$ nearest dataset states, with distance ties resolved by the smallest dataset
index. Define the nonempty local action reference
$\mathcal A_{\mathcal D,k}(s)=\{a_i:i\in\mathcal N_k(s)\}$ and its
sample-weighted empirical distribution
$\hat\beta_k(\cdot|s)=k^{-1}\sum_{i\in\mathcal N_k(s)}\delta_{a_i}$.
Actions outside its effective support are locally OOD. A standard critic
minimizes the temporal-difference (TD) loss~\citep{sutton1998reinforcement}:
\begin{equation}
\begin{split}
\mathcal{L}_{\text{TD}}(\theta)
= \mathbb{E}_{(s,a,s')\sim\mathcal D}
\Bigg[ \Big( Q_\theta(s,a) &- R(s,a) - \gamma\max_{a'}Q_{\theta'}(s',a') \Big)^2 \Bigg].
\end{split}
\end{equation}
where $Q_\theta$ is the learned critic and $Q_{\theta'}$ is its target network,
updated by Polyak averaging~\citep{mnih2015human}.

\section{CHN Smooth Dual Generalization for Offline RL}
This section presents CSDG, which uses CHN as the geometric reference for local
generalization. Figure~\ref{zongtu} gives an overview. We first define the asymmetric smoothing target and its expectile-referenced
correction, analyze it at the one-step, iterative, fixed-point, and policy
levels, and then describe the practical algorithm.

\begin{wrapfigure}{r}{0.45\textwidth}   % r: 靠右, l: 靠左；宽度可调
    \centering
    \vspace{-3em}
    \includegraphics[width=\linewidth]{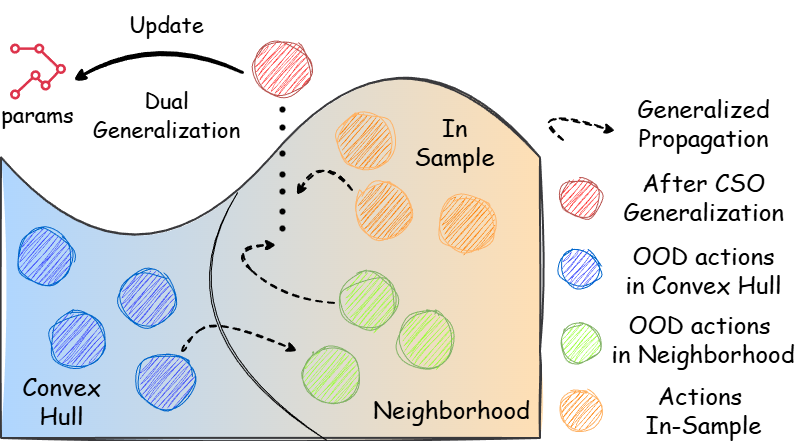}  % 原0.9\linewidth也可保留
    \caption{CSDG combines a locally smoothed Q-value target with an in-sample value
    target. Generated actions are associated with observed-action anchors through
    an effective geometric radius, and $\lambda$ scales the generalized correction
    during bootstrapping.}
    \label{zongtu}
    \vspace{-2em}
\end{wrapfigure}
\subsection{CHN Smooth Dual Generalization}
To formalize local generalization, we first specify the state-dependent action
region used by CSDG. We adopt Convex Hull and its Neighborhood
(CHN)~\citep{yao2025offline} as the geometric reference for this construction.

\begin{definition}[State-conditional CHN]
\label{de1}
Let $\mathcal S\subseteq\mathbb R^{d_s}$, and let
$\mathcal A\subseteq\mathbb R^{d_a}$ be a closed convex action box. Given a state
$s\in\mathcal S$, the local empirical action reference
$\mathcal A_{\mathcal D,k}(s)$, and a radius $\delta$, the
\emph{state-conditional CHN} is
\begin{equation}
\small
\operatorname{CHN}_{\delta}(s)
=
\left\{
a\in\mathcal A:
\min_{a'\in\operatorname{Conv}(\mathcal A_{\mathcal D,k}(s))}
\|a-a'\|_2\le\delta
\right\}.
\end{equation}
Here $\operatorname{Conv}(\mathcal A_{\mathcal D,k}(s))$ is the convex hull of
the local action reference.
\end{definition}

\paragraph{Interpretation.} This state-conditional construction provides a continuous local reference
without requiring exact state repetitions. CSDG uses it to describe the region
associated with the local correction and its geometric error; the corresponding
anchoring analysis is given in Appendix~\ref{chn}. For the operator analysis, equip $\mathcal S$ and $\mathcal A$ with their Borel
$\sigma$-algebras and assume that the reward, transition kernel, target selector,
perturbation kernels, and deterministic tie-breaking rules are measurable. We
work on the Banach space $\mathbb{B}_b(\mathcal S\times\mathcal A)$ of bounded
measurable functions with the supremum norm, and write
$\operatorname{clip}=\Pi_{\mathcal A}$ for Euclidean projection onto the action
box. 

We now define the two-scale smoothing mechanism used to construct the
generalized target.

\begin{definition}[CHN smooth operator]
\label{defcso}
Given a fixed deterministic target action selector
$\pi':\mathcal S\to\mathcal A$ and two bounded perturbation kernels
$\nu_{\mathrm{In}}$ and $\nu_{\mathrm{OOD}}$, draw $\eta_j\sim\nu_j$ and set
\begin{equation}
\begin{aligned}
a_{\mathrm{In}}
&=\operatorname{clip}\bigl(\pi'(s)+\eta_{\mathrm{In}}\bigr),
&\|\eta_{\mathrm{In}}\|_2&\le\delta_{\mathrm{In}},\\
a_{\mathrm{OOD}}
&=\operatorname{clip}\bigl(\pi'(s)+\eta_{\mathrm{OOD}}\bigr),
&\|\eta_{\mathrm{OOD}}\|_2&\le\delta_{\mathrm{OOD}},
\end{aligned}
\label{eq:asymmetric_candidates}
\end{equation}
where $0\le\delta_{\mathrm{In}}<\delta_{\mathrm{OOD}}$. The smaller-radius
action $a_{\mathrm{In}}$ is the in-sample-oriented candidate, and the
larger-radius action $a_{\mathrm{OOD}}$ is the OOD-oriented candidate. For
$\mu\in[0,1]$, define the expected CHN-smoothed value as
\begin{equation}
\begin{aligned}
\mathcal S_{\mu}Q(s)
=&\;
\mu\,\mathbb E_{\eta_{\mathrm{In}}\sim\nu_{\mathrm{In}}}
\left[Q(s,a_{\mathrm{In}})\right] +
(1-\mu)\,\mathbb E_{\eta_{\mathrm{OOD}}\sim\nu_{\mathrm{OOD}}}
\left[Q(s,a_{\mathrm{OOD}})\right].
\end{aligned}
\label{eq:theoretical_smooth_target}
\end{equation}
\end{definition}

The practical update draws one Monte Carlo sample from each kernel, yielding a
conditionally unbiased estimator of $\mathcal S_\mu Q(s)$. Our analysis uses the
perturbation-averaged operator.

\paragraph{Interpretation.}
The asymmetric perturbations avoid explicit CHN membership tests and
nearest-neighbor projection. The smaller-radius candidate provides a local
reference, whereas the larger-radius candidate introduces controlled
extrapolation. The two labels describe their sampling radii, not exact support
membership.

\paragraph{Local generalization in offline RL.}
Value updates can generalize to actions near the dataset, although their errors
may propagate through bootstrapping~\citep{DMG}. CSDG constructs the generalized
target directly with the two-scale smoother in
Eq.~\eqref{eq:theoretical_smooth_target}, rather than maximizing over an expanded
policy support. We distinguish the CHN radius from the effective
observed-action anchoring radius because distance to a convex hull need not equal
distance to an observed action.

\begin{definition}[CHN smooth dual-generalization operator]
\label{de4}
For $\tau\in(0,1)$, let
$L_2^\tau(u)=|\tau-\mathbf{1}\{u<0\}|u^2$. The in-sample expectile functional is
\begin{equation}
V_\tau^Q(s)
=
\arg\min_{v\in\mathbb R}
\mathbb E_{a\sim\hat\beta_k(\cdot|s)}
\left[L_2^\tau\bigl(Q(s,a)-v\bigr)\right].
\label{eq:theoretical_expectile}
\end{equation}
The objective is strictly convex for the nonempty empirical kernel and therefore
has a unique minimizer. Define the CHN-local correction by
$
\mathcal B_{\mathrm{CHN}}^\tau Q(s)
=
\mathcal S_\mu Q(s)-V_\tau^Q(s).
$
For $\lambda\in[0,1]$, the CSDG operator is
\begin{equation}
\begin{split}
\mathcal{T}_{\text{CSDG}}Q(s,a)
:=&\; R(s,a) 
   + \gamma\,\mathbb{E}_{s'\sim P(\cdot|s,a)}
     \Bigl[ V_\tau^Q(s')
     + \lambda\mathcal B_{\mathrm{CHN}}^\tau Q(s') \Bigr].
\end{split}
\label{eq:csdg_bias_injection}
\end{equation}
Equivalently, the target term is
$\lambda\mathcal S_\mu Q(s')+(1-\lambda)V_\tau^Q(s')$.
\end{definition}

\paragraph{Interpretation.} Equation~\eqref{eq:csdg_bias_injection} shows that
$\lambda$ scales the CHN-local correction
$\mathcal B_{\mathrm{CHN}}^\tau Q$ added to the in-sample backup, while the
recursive propagation factor remains $\gamma$. The analysis below first bounds
the direct branch discrepancy and then separates its geometric and reference
components through a nearest-anchor decomposition.

CSDG introduces controlled generalization at two complementary levels of the Bellman update.

\paragraph{Smooth OOD action generalization (first level).}
CSDG uses asymmetric perturbation kernels to generate actions in a controlled
local region. For each realized candidate, the nearest action among the local
dataset transitions serves as its anchor, and the corresponding value
difference contributes to $e_{\mathrm{geo}}(Q,s)$.

\paragraph{Smooth generalization propagation (second level).}
By combining the generalized target with an in-sample expectile target, CSDG
scales the local correction in each Bellman backup while retaining the recursive
discount factor $\gamma$.

Together, these two levels regulate the local correction through the Bellman
target without an auxiliary pessimistic penalty. We next analyze the correction at the one-step, iterative, fixed-point, and
policy levels.

\subsection{Analysis of the Generalized Correction}
\label{sec:theory_smoothness}

The branch discrepancy admits a geometric and reference decomposition.

\begin{definition}[Nearest-anchor decomposition]
\label{def:nearest_anchor_decomposition}
Let $\mathcal J=\{\mathrm{In},\mathrm{OOD}\}$,
$\alpha_{\mathrm{In}}=\mu$, and
$\alpha_{\mathrm{OOD}}=1-\mu$. For each $j\in\mathcal J$, let
$i_j(s,\eta_j)\in\mathcal N_k(s)$ minimize $\|a_j-a_i\|_2$, and write
$(s_{j,\mathrm{anc}},a_{j,\mathrm{anc}})=(s_{i_j},a_{i_j})$.
Distance ties are resolved by the smallest dataset index. For any branchwise quantity $g_j$, write
$\mathbb E_\alpha[g_j]
=\sum_{j\in\mathcal J}\alpha_j\mathbb E_{\eta_j}[g_j]$.
Define
\[
\begin{aligned}
e_{\mathrm{geo}}(Q,s)
&=
\mathbb E_\alpha
\left[
\left|
Q(s,a_j)
-
Q(s_{j,\mathrm{anc}},a_{j,\mathrm{anc}})
\right|
\right],\\
e_{\mathrm{ref}}(Q,s)
&=
\left|
\mathbb E_\alpha
\left[
Q(s_{j,\mathrm{anc}},a_{j,\mathrm{anc}})
\right]
-
V_\tau^Q(s)
\right|.
\end{aligned}
\]
The triangle inequality gives
\begin{equation}
\left|
\mathcal S_\mu Q(s)-V_\tau^Q(s)
\right|
\le
e_{\mathrm{geo}}(Q,s)
+
e_{\mathrm{ref}}(Q,s).
\label{eq:nearest_anchor_decomposition}
\end{equation}
\end{definition}

\paragraph{Interpretation.}
The selector associates each realized candidate with a local dataset
transition. The term $e_{\mathrm{geo}}(Q,s)$ measures the candidate-to-anchor
value difference, and $e_{\mathrm{ref}}(Q,s)$ measures the difference between
the averaged anchor value and the expectile reference. 

We next use the direct branch discrepancy to compare the generalized and
in-sample Bellman operators.

\begin{definition}[In-sample reference and branch discrepancy]
\label{def:branch_discrepancy}
The expectile-based in-sample operator is
$
\mathcal T_{\mathrm{In}}^\tau Q(s,a)
=
R(s,a)
+
\gamma
\mathbb E_{s'\sim P(\cdot|s,a)}
\left[
V_\tau^Q(s')
\right].
$
For every bounded $Q$, define the direct branch discrepancy as
\begin{equation}
e(Q)
=
\left\|
\mathcal S_\mu Q-V_\tau^Q
\right\|_\infty.
\label{eq:direct_branch_discrepancy}
\end{equation}
\end{definition}

\paragraph{Interpretation.}
The quantity $e(Q)$ measures the value difference between the smoothed and
expectile branches at $Q$. The nearest-anchor decomposition gives
$
e(Q)
\le
\sup_s
\left[
e_{\mathrm{geo}}(Q,s)
+
e_{\mathrm{ref}}(Q,s)
\right].
$

The following result characterizes the contribution of $e(Q)$ to one Bellman
update.

\begin{theorem}[Exact one-step correction]
\label{th:unified_bound}
For every bounded $Q$, the difference between the CSDG and in-sample Bellman operators is
$
\mathcal T_{\mathrm{CSDG}}Q(s,a)
-
\mathcal T_{\mathrm{In}}^\tau Q(s,a)
=
\gamma\lambda
\mathbb E_{s'\sim P(\cdot|s,a)}
\left[
\mathcal S_\mu Q(s')
-
V_\tau^Q(s')
\right].
$
Consequently,
\begin{equation}
\left\|
\mathcal T_{\mathrm{CSDG}}Q
-
\mathcal T_{\mathrm{In}}^\tau Q
\right\|_\infty
\le
\gamma\lambda e(Q).
\label{eq:one_step_correction_bound}
\end{equation}
\end{theorem}

\paragraph{Proof sketch.}
Expanding the two Bellman operators and subtracting their common terms gives
the stated identity. The triangle inequality and the definition of $e(Q)$ give
Eq.~\eqref{eq:one_step_correction_bound}.

\paragraph{Interpretation.}
The identity gives the signed contribution of the generalized branch to one
Bellman update. The parameter $\lambda$ scales this contribution, while
$\gamma$ propagates it through the transition dynamics.

We next track the branch discrepancy along the Bellman iterates and at the
fixed point.

\subsection{Iterate and Fixed-Point Guarantees}
\label{sec:robustness}

For the fixed target action selector and perturbation kernels,
$\mathcal S_\mu$ and $V_\tau^Q$ are non-expansive in the supremum norm.
Hence, $\mathcal T_{\mathrm{CSDG}}$ and
$\mathcal T_{\mathrm{In}}^\tau$ are $\gamma$-contractions with unique fixed
points.

\begin{theorem}[Iterate, fixed-point, and residual bounds]
\label{th:deviation}
Let
$Q_{\mathrm{CSDG}}^{k+1}
=\mathcal T_{\mathrm{CSDG}}Q_{\mathrm{CSDG}}^k$
and
$Q_{\mathrm{In}}^{k+1}
=\mathcal T_{\mathrm{In}}^\tau Q_{\mathrm{In}}^k$,
with $Q_{\mathrm{CSDG}}^0=Q_{\mathrm{In}}^0$. At iteration $k$, write
$e_k=e(Q_{\mathrm{CSDG}}^k)
=\|\mathcal S_\mu Q_{\mathrm{CSDG}}^k
-V_\tau^{Q_{\mathrm{CSDG}}^k}\|_\infty$.
Then, for every $k\ge1$,
\begin{equation}
\left\|
Q_{\mathrm{CSDG}}^k-Q_{\mathrm{In}}^k
\right\|_\infty
\le
\gamma\lambda
\sum_{t=0}^{k-1}
\gamma^{k-1-t}e_t.
\label{eq:deviation_bound}
\end{equation}

Let $Q_{\mathrm{CSDG}}^*$ and $Q_{\mathrm{In}}^*$ be the corresponding
fixed points, and write
$e_*=e(Q_{\mathrm{CSDG}}^*)$. For any bounded $\bar Q$, define
$r_{\mathrm{CSDG}}(\bar Q)
=\|\bar Q-\mathcal T_{\mathrm{CSDG}}\bar Q\|_\infty$.
Then
\begin{equation}
\begin{gathered}
\left\| Q_{\mathrm{CSDG}}^*-Q_{\mathrm{In}}^* \right\|_\infty
\le \frac{\gamma\lambda e_*}{1-\gamma},\\[4pt]
\left\| \bar Q-Q_{\mathrm{In}}^* \right\|_\infty
\le \frac{ r_{\mathrm{CSDG}}(\bar Q) + \gamma\lambda e(\bar Q) }{ 1-\gamma }.
\end{gathered}
\label{eq:fixed_point_residual_bounds}
\end{equation}
\end{theorem}

\paragraph{Proof sketch.}
Let
$D_k=\|Q_{\mathrm{CSDG}}^k-Q_{\mathrm{In}}^k\|_\infty$.
Theorem~\ref{th:unified_bound} and the contraction of
$\mathcal T_{\mathrm{In}}^\tau$ give
$D_{k+1}\le\gamma\lambda e_k+\gamma D_k$.
Unrolling this recurrence from $D_0=0$ gives
Eq.~\eqref{eq:deviation_bound}. Applying the same inequality to the two fixed
points gives the first bound in
Eq.~\eqref{eq:fixed_point_residual_bounds}. For a bounded $\bar Q$, inserting
$\mathcal T_{\mathrm{CSDG}}\bar Q$ and
$\mathcal T_{\mathrm{In}}^\tau\bar Q$ between $\bar Q$ and
$Q_{\mathrm{In}}^*$ gives the second bound.

\paragraph{Interpretation.}
Equation~\eqref{eq:deviation_bound} accumulates the observed discrepancies
$e_t$ with their remaining discount factors. The fixed-point bound depends on
$e_*$, and the residual bound combines the Bellman residual of $\bar Q$ with
its branch discrepancy.

We next characterize the policies represented by the two fixed points.

\begin{definition}[Induced fixed-point policies]
\label{def:induced_fixed_point_policies}
Let $w_\tau(u)=\tau$ for $u\ge0$ and $w_\tau(u)=1-\tau$ otherwise.
Define $\pi_\tau^Q$ as the normalized reweighting
$
\pi_\tau^Q(da|s)
\propto
w_\tau\bigl(Q(s,a)-V_\tau^Q(s)\bigr)
\hat\beta_k(da|s).
$
The expectile first-order condition gives
$V_\tau^Q(s)=\mathbb E_{a\sim\pi_\tau^Q(\cdot|s)}[Q(s,a)]$.
Let $\kappa_\mu(\cdot|s)$ be the mixture distribution of
$a_{\mathrm{In}}$ and $a_{\mathrm{OOD}}$ with weights $\mu$ and $1-\mu$.
Then
$\mathcal S_\mu Q(s)
=\mathbb E_{a\sim\kappa_\mu(\cdot|s)}[Q(s,a)]$.
At the fixed points, define
\begin{equation}
\begin{gathered}
\pi_{\mathrm{In}}^* = \pi_\tau^{Q_{\mathrm{In}}^*},\\
\pi_{\mathrm{CSDG}}^* = (1-\lambda)\pi_\tau^{Q_{\mathrm{CSDG}}^*} + \lambda\kappa_\mu.
\end{gathered}
\label{eq:induced_policy_system}
\end{equation}
\end{definition}

\paragraph{Interpretation.}
The two branches induce a reweighted empirical policy and a perturbation
kernel, whose mixture defines $\pi_{\mathrm{CSDG}}^*$.

\begin{theorem}[Induced-policy performance]
\label{th:policy_performance_comparison}
The fixed points satisfy
$Q_{\mathrm{In}}^*=Q^{\pi_{\mathrm{In}}^*}$ and
$Q_{\mathrm{CSDG}}^*=Q^{\pi_{\mathrm{CSDG}}^*}$.
Let $A^\pi(s,a)=Q^\pi(s,a)-V^\pi(s)$, and let $d_\pi^\gamma$ denote the
normalized discounted state occupancy. For an action kernel $\rho$, let
$\bar G(\rho)$ be the expected value of
$A^{\pi_{\mathrm{In}}^*}(s,a)$ under
$s\sim d_{\pi_{\mathrm{CSDG}}^*}^\gamma$ and
$a\sim\rho(\cdot|s)$. Write
$\bar G_\tau=\bar G(\pi_\tau^{Q_{\mathrm{CSDG}}^*})$ and
$\bar G_\mu=\bar G(\kappa_\mu)$. If $\bar G_\tau\ge-\xi$ and $\bar G_\mu\ge g_\mu$, where
$\xi,g_\mu\ge0$, then
\begin{equation}
\begin{aligned}
J(\pi_{\mathrm{CSDG}}^*)-J(\pi_{\mathrm{In}}^*)
&=
\frac{
(1-\lambda)\bar G_\tau+\lambda\bar G_\mu
}{
1-\gamma
}\\
&\ge
\frac{
\lambda g_\mu-(1-\lambda)\xi
}{
1-\gamma
},\\
\lambda g_\mu\ge(1-\lambda)\xi
&\Longrightarrow
J(\pi_{\mathrm{CSDG}}^*)
\ge
J(\pi_{\mathrm{In}}^*).
\end{aligned}
\label{eq:policy_performance_result}
\end{equation}
\end{theorem}

\paragraph{Proof sketch.}
The induced-policy identities turn the fixed-point equations into the Bellman
evaluation equations of $\pi_{\mathrm{In}}^*$ and
$\pi_{\mathrm{CSDG}}^*$. The performance-difference lemma and the mixture form
of $\pi_{\mathrm{CSDG}}^*$ give the first equality in
Eq.~\eqref{eq:policy_performance_result}; the bounds on $\bar G_\tau$ and
$\bar G_\mu$ give the remaining statements.

\paragraph{Interpretation.}
The averaged advantages separate the expectile-policy and generalized-branch
contributions. Their weighted balance determines the non-degradation
condition.

This completes the idealized operator analysis. We next present the practical
optimization procedure.

\subsection{Practical Algorithm}
The practical CSDG algorithm uses a deterministic actor $\pi_\phi$ and target
actor $\pi_{\phi'}$, twin critics $Q_{\theta_1},Q_{\theta_2}$ and their target
copies $Q_{\theta_1'},Q_{\theta_2'}$, and a value network $V_\psi$. We write
$Q_\theta=\min_iQ_{\theta_i}$ and $Q_{\theta'}=\min_iQ_{\theta_i'}$ for the
online and target clipped double-Q estimates, respectively, and update all
target networks by Polyak averaging.
\paragraph{Policy learning.}
The actor objective combines TD3BC-style Q maximization with
advantage-weighted behavior cloning. Define
\begin{equation}
w(s,a)
=
\operatorname{clip}
\left(
\exp
\left[
\alpha
\left(
Q_{\theta'}(s,a)-V_\psi(s)
\right)
\right],
w_{\max}
\right),
\label{eq:awr_weight}
\end{equation}
where $\alpha$ is the inverse temperature and $w_{\max}$ clips large advantage
weights. The adaptive Q-scaling coefficient is
\begin{equation}
\omega_Q
=
\left(
\mathbb{E}_{s\sim\mathcal D}
\left[
\left|
Q_{\theta}(s,\pi_\phi(s))
\right|
\right]
\right)^{-1}.
\label{eq:q_scale}
\end{equation}
The actor loss is
\begin{equation}
\begin{split}
L_\pi(\phi)
=
&-\omega_Q
\mathbb{E}_{s\sim\mathcal D}
\left[
Q_{\theta}(s,\pi_\phi(s))
\right] +
\nu
\mathbb{E}_{(s,a)\sim\mathcal D}
\left[
w(s,a)
\left\|
\pi_\phi(s)-a
\right\|_2^2
\right].
\end{split}
\label{eqpi}
\end{equation}
The first term favors high-value actions, while the second keeps the actor near
advantage-weighted dataset actions. Thus, behavior regularization is applied in
the actor update, whereas the critic uses no additional CQL-style OOD penalty.

\begin{table*}[htb]
  \footnotesize
  \small
  \centering
    \caption{Average normalized scores on Gym locomotion and AntMaze tasks. CSDG results are averaged over the final ten evaluations and 10 random seeds.
  m = medium, m-r = medium-replay, m-e = medium-expert, e = expert, r = random; u = umaze, u-d = umaze-diverse, m-p = medium-play, m-d = medium-diverse, l-p= large-play, l-d = large-diverse.
  }
  \label{table2}
  \setlength{\tabcolsep}{3.7pt}
  \begin{adjustbox}{max width=\textwidth}
  \begin{tabular}{@{}l rrrrrrrrrr r@{}}
    \toprule
    Dataset-v2 & BC & BCQ & BEAR & AWAC & TD3BC & CQL & IQL & CPI & CQL+C4 & UNIQ & CSDG(Ours) \\ \midrule
    halfcheetah-m & 42.0 & 46.6 & 43.0 & 47.9 & 48.3 & 47.0 & 47.4 & \textbf{64.4} & 48.5 & 48.9 & 55.3$\pm$0.4 \\ 
    hopper-m & 56.2 & 59.4 & 51.8 & 59.8 & 59.3 & 53.0 & 66.2 & 98.5 & 85.9 & 75.6 & \textbf{101.2$\pm$1.6} \\ 
    walker2d-m & 71.0 & 71.8 & -0.2 & 83.1 & 83.7 & 73.3 & 78.3 & 85.8 & 81.8 & 85.5 & \textbf{93.1$\pm$1.8} \\ 
    halfcheetah-m-r & 36.4 & 42.2  & 36.6 & 44.8 & 44.6 & 45.5 & 44.2 & \textbf{54.6} & 46.7 & 46.0 & 51.5$\pm$0.2 \\ 
    hopper-m-r & 21.8 & 60.9 & 52.2 & 69.8 & 60.9 & 88.7 & 94.7 & 101.7 & 95.0 & 101.6 & \textbf{102.5$\pm$1.4} \\ 
    walker2d-m-r & 24.9 & 57.0 & 7.0 & 78.1 & 81.8 & 81.8 & 73.8 & 91.8 & 90.9 & 89.4 & \textbf{96.3$\pm$2.5} \\ 
    halfcheetah-m-e & 59.6 & \textbf{95.4} & 46.0 & 64.9 & 90.7 & 75.6 & 86.7 & 94.7 & 91.6 & 94.8 & 94.1$\pm$1.7 \\ 
    hopper-m-e & 51.7 & 106.9 & 50.6 & 100.1 & 98.0 & 105.6 & 91.5 & 106.4 & 89.4 & \textbf{111.8} & 111.3$\pm$1.6 \\ 
    walker2d-m-e & 101.2 & 107.7 & 22.1 & 110.0 & 110.1 & 107.9 & 109.6 & 110.9 & 108.6 & 112.9 & \textbf{114.8$\pm$0.8} \\ 
    halfcheetah-e & 92.9 & 89.9 & 92.7 & 81.7 & \textbf{96.7} & 96.3 & 95.0 & 96.5 & 93.2 & -- & 95.6$\pm$1.4 \\ 
    hopper-e & 110.9 & 109.0 & 54.6 & 109.5 & 107.8 & 96.5 & 109.4 & 112.2 & 110.1 & -- & \textbf{113.3$\pm$3.2} \\ 
    walker2d-e & 107.7 & 106.3 & 106.6 & 110.1 & 110.2 & 108.5 & 109.9 & 110.6 & 109.7 & -- & \textbf{114.9$\pm$0.3} \\ 
    halfcheetah-r & 2.6 & 2.2 & 2.3 & 6.1 & 11.0 & 17.5 & 13.1 & 29.7 & -- & -- & \textbf{33.7$\pm$6.8} \\ 
    hopper-r & 4.1 & 7.8 & 3.9 & 9.2 & 8.5 & 7.9 & 7.9 & \textbf{29.5} & -- & -- & 8.2$\pm$1.1\\ 
    walker2d-r & 1.2 & 4.9 & 12.8 & 0.2 & 1.6 & 5.1 & 5.4 & 5.9 & -- & -- & \textbf{13.6$\pm$1.9} \\ \midrule
    locomotion total & 784.2 & 968.0 & 581.7 & 975.3 & 1013.2 & 1010.2 & 1033.1 & 1193.2 & -- & -- & \textbf{1199.4} \\ 
    \midrule
    antmaze-u & 66.8 & 78.9 & 73.0 & 80.0 & 73.0 & 82.6 & 89.6 & \textbf{98.8} & -- & -- & 92.6$\pm$1.9 \\ 
    antmaze-u-d & 56.8 & 55.0 & 61.0 & 52.0 & 47.0 & 10.2 & 65.6 & \textbf{88.6} & -- & -- & 80.1$\pm$7.1 \\ 
    antmaze-m-p & 0.0 & 0.0 & 0.0 & 0.0 & 0.0 & 59.0 & 76.4 & \textbf{82.4} & -- & -- & 81.6$\pm$4.1 \\ 
    antmaze-m-d & 0.0 & 0.0 & 0.0 & 0.6 & 0.2 & 46.6 & 72.8 & \textbf{80.4} & -- & -- & 74.2$\pm$4.7 \\ 
    antmaze-l-p & 0.0 & 6.7 & 0.0 & 0.0 & 0.0 & 16.4 & 42.0 & 20.6 & -- & -- & \textbf{56.1$\pm$4.3} \\ 
    antmaze-l-d & 0.0 & 2.2 & 0.0 & 0.0 & 0.0 & 3.2 & 46.0 & 45.2 & -- & -- & \textbf{61.2$\pm$3.7} \\ \midrule
    antmaze total & 123.6 & 142.8 & 142.0 & 132.2 & 120.2 & 218.0 & 392.4 & 416.0 & -- & -- & \textbf{445.8} \\ 
    \bottomrule
  \end{tabular}
  \end{adjustbox}
\end{table*}

\paragraph{Value learning.}
The value network is trained by expectile regression over the clipped target
critic:
\begin{equation}
L_V(\psi)
=
\mathbb{E}_{(s,a)\sim\mathcal D}
\left[
L_2^\tau
\left(
Q_{\theta'}(s,a)-V_\psi(s)
\right)
\right],
\label{eqv}
\end{equation}
where $L_2^\tau$ denotes the expectile loss. This learned value network is a
function-approximation surrogate for the ideal reference
$V_\tau^{Q_{\theta'}}$ in the theoretical operator. The idealized fixed-point
result analyzes the exact reference; practical performance also depends on the
residual $\sup_s|V_\psi(s)-V_\tau^{Q_{\theta'}}(s)|$ and optimization error.

\paragraph{Smoothed target construction.}
For each next state $s'$, CSDG samples two clipped target actions,
$\tilde a_j=\operatorname{clip}(\pi_{\phi'}(s')+\eta_j,
-a_{\max},a_{\max})$, where $j\in\{\mathrm{In},\mathrm{OOD}\}$.
The OOD-oriented branch uses a larger noise scale and clipping radius than the
in-sample-oriented branch. Their clipped double-Q values form the smoothed
target
\begin{equation}
\widetilde Q_{\theta'}(s')
=
\mu Q_{\theta'}(s',\tilde a_{\mathrm{In}})
+
(1-\mu)Q_{\theta'}(s',\tilde a_{\mathrm{OOD}}).
\label{eq:smooth_target}
\end{equation}
Each sampled pair gives a Monte Carlo estimate of
$\mathcal S_\mu Q_{\theta'}(s')$. The critic target adds the corresponding
local correction to the in-sample value:
\begin{equation}
y
=
R(s,a)
+
\gamma(1-d)
\left[
V_\psi(s')
+
\lambda
\left(
\widetilde Q_{\theta'}(s')-V_\psi(s')
\right)
\right].
\label{eq:critic_target}
\end{equation}
Both critics minimize
\begin{equation}
L_Q(\theta_i)
=
\mathbb E_{(s,a,s',d)\sim\mathcal D}
\left[
\left(
Q_{\theta_i}(s,a)-y
\right)^2
\right],
i\in\{1,2\}.
\label{eqq}
\end{equation}
Clipped double-Q estimation limits critic overestimation, while $\lambda$
scales the local correction.

\begin{algorithm}[htb]
\caption{Practical CSDG}
\label{alg:csdg}
\begin{algorithmic}[1]
  \State Initialize $\pi_\phi$, $Q_{\theta_1},Q_{\theta_2}$, $V_\psi$,
  and the actor/critic target networks
  \For{each gradient step}
    \State Sample $(s,a,r,s',d)\sim\mathcal D$
    \State Update $V_\psi$ using Eq.~\eqref{eqv}
    \State Generate $\tilde a_{\mathrm{In}},\tilde a_{\mathrm{OOD}}$ and
    compute $\widetilde Q_{\theta'}$ and $y$ using
    Eqs.~\eqref{eq:smooth_target}--\eqref{eq:critic_target}
    \State Update both critics using Eq.~\eqref{eqq}
    \If{delayed update}
      \State Update $\pi_\phi$ and all target networks
    \EndIf
  \EndFor
\end{algorithmic}
\end{algorithm}
\paragraph{Practical noise injection.}
Asymmetric bounded noise sets two action radii around the target action. Let
$C_k(s')=\operatorname{Conv}(\mathcal A_{\mathcal D,k}(s'))$ and
$\rho_k(s')=\operatorname{dist}(\pi_{\phi'}(s'),C_k(s'))$. For
$\|\eta_j\|_2\le\delta_j$, the perturbed candidate satisfies
\[
\operatorname{dist}(\tilde a_j,C_k(s'))
\le
\rho_k(s')+\delta_j.
\]
Thus, $\delta_{\mathrm{In}}<\delta_{\mathrm{OOD}}$ gives a tighter locality
bound for the in-sample-oriented candidate and a wider search region for the
OOD-oriented candidate. Appendix~\ref{PFNI} relates these perturbation radii to
the anchoring radii used in the analysis.

\paragraph{Overall algorithm.}
Algorithm~\ref{alg:csdg} summarizes the complete training procedure.
% \section{Experiments}
% In this section, we conduct several experiments to justify the validity of the proposed method CSDG. Experimental details and extended results are provided in Appendices D and E, respectively.

\section{Experiments}
We evaluate CSDG on the D4RL benchmark and then examine its main design choices
through ablations. Experimental details and extended results are provided in
Appendices~\ref{appd} and~\ref{appe}, respectively.
\subsection{D4RL Benchmark Results}

We compare CSDG with three groups of baselines on 21 D4RL
datasets~\citep{fu2020d4rl}: constraint-based and bootstrapping methods
(BCQ~\citep{fujimoto2019off},
BEAR~\citep{kumar2019stabilizing}, AWAC~\citep{nair2020awac},
TD3BC~\citep{fujimoto2021minimalist}, and
CQL~\citep{kumar2020conservative}), in-sample methods
(BC~\citep{pomerleau1988alvinn}, IQL~\citep{kostrikov2021offline}, and
$\mathcal{X}$QL~\citep{garg2023extreme}), and recent approaches including
CPI~\citep{cpi1}, C4~\citep{c4}, and
UNIQ~\citep{upadhyay2026uniq}. As shown in Table~\ref{table2}, CSDG achieves the
highest aggregate score in both benchmark groups, with 1199.4 on Gym--MuJoCo
and 445.8 on AntMaze. For a more detailed comparison of recent methods, please refer to Appendices~B and~C.

\subsection{Performance Improvement over In-sample Learning Approaches}

CSDG consistently improves both tested in-sample backbones across all D4RL v2
locomotion tasks. We integrate CSDG with $\mathcal{X}$QL~\citep{garg2023extreme}
and IQL~\citep{kostrikov2021offline}, with the task-level and aggregate results
reported in Table~\ref{table3}. Across the medium, medium-replay, and
medium-expert datasets, the total score increases from 725.3 to 779.3 for
$\mathcal{X}$QL and from 692.4 to 756.9 for IQL. The improvements across
different dataset qualities indicate that the two-scale target is not tied to
a particular in-sample learner. Instead, it incorporates locally generalized
OOD values while retaining the in-sample reference path.

\begin{table}[htp]  % 放在栏顶部，避免乱飘
\small
\centering
\caption{Average scores over five seeds when CSDG is combined with different in-sample methods.}
\label{table3}
\begin{tabular}{lrr}
\toprule
Dataset-v2 & $\mathcal{X}$QL (+CSDG) & IQL (+CSDG) \\
\midrule
halfcheetah-m & 47.7 $\rightarrow$ \textbf{51.3} & 47.4 $\rightarrow$ \textbf{48.3} \\
hopper-m & 71.1 $\rightarrow$ \textbf{94.6} & 66.2 $\rightarrow$ \textbf{76.3} \\
walker2d-m & 81.5 $\rightarrow$ \textbf{84.7} & 78.3 $\rightarrow$ \textbf{85.1} \\
halfcheetah-m-r & 44.8 $\rightarrow$ \textbf{48.3} & 44.2 $\rightarrow$ \textbf{45.1} \\
hopper-m-r & 97.3 $\rightarrow$ \textbf{101.6} & 94.7 $\rightarrow$ \textbf{100.6} \\
walker2d-m-r & 75.9 $\rightarrow$ \textbf{83.8} & 73.8 $\rightarrow$ \textbf{88.5} \\
halfcheetah-m-e & 89.8 $\rightarrow$ \textbf{94.7} & 86.7 $\rightarrow$ \textbf{91.2} \\
hopper-m-e & 107.1 $\rightarrow$ \textbf{109.6} & 91.5 $\rightarrow$ \textbf{109.4} \\
walker2d-m-e & 110.1 $\rightarrow$ \textbf{110.7} & 109.6 $\rightarrow$ \textbf{112.4} \\
\midrule
total & 725.3 $\rightarrow$ \textbf{779.3} & 692.4 $\rightarrow$ \textbf{756.9} \\
\bottomrule
\end{tabular}
\end{table}
\begin{table*}[t]
\small
\centering
\caption{Normalized scores of CSDG under different Gaussian noise scales and
clipping values on Gym locomotion and AntMaze tasks, averaged over five seeds.
Best scores are shown in bold.}
\label{tab:sqog_noise_clip}
\begin{tabular}{lcccc}
\toprule
\textbf{Dataset} &
\small scale{=}0.2, clip{=}0.3 &
\small scale{=}0.6, clip{=}0.5 &
\small scale{=}1.0, clip{=}0.7 &
\small scale{=}2.0, clip{=}1.0 \\
\midrule
halfcheetah-medium & \textbf{55.8\,$\pm$\,0.5} & 55.3\,$\pm$\,0.4 & 54.1\,$\pm$\,0.4 & 53.8\,$\pm$\,0.9 \\
hopper-medium                &  100.7\,$\pm$\,0.6 & \textbf{101.2\,$\pm$\,1.6} & 99.5\,$\pm$\,2.6 & 94.5\,$\pm$\,4.5 \\
walker2d-medium              &  92.1\,$\pm$\,2.3 &  \textbf{93.1\,$\pm$\,1.8} & 90.2\,$\pm$\,0.5 & 87.2\,$\pm$\,0.8 \\
halfcheetah-medium-expert   &  92.2\,$\pm$\,1.9 & \textbf{94.1\,$\pm$\,1.7} & 91.3\,$\pm$\,1.2 & 88.2\,$\pm$\,0.8 \\
hopper-medium-expert        & 87.4\,$\pm$\,3.8 & \textbf{111.3\,$\pm$\,1.6} & 108.3\,$\pm$\,1.9 & 74.1\,$\pm$\,5.8\\
walker2d-medium-expert      & 111.2\,$\pm$\,0.3 & \textbf{114.8\,$\pm$\,0.8} & 111.2\,$\pm$\,1.2 & 110.3\,$\pm$\,2.1 \\
\midrule
\textit{MuJoCo Average}      & 89.9 & \textbf{94.9} & 92.4 & 84.7 \\
\midrule
antmaze-large-play                  & 35.7\,$\pm$\,3.8 & \textbf{56.1\,$\pm$\,4.3} & 37.3\,$\pm$\,4.3 & 32.8\,$\pm$\,4.3 \\
antmaze-large-diverse                   & 56.1\,$\pm$\,3.2 & \textbf{61.2\,$\pm$\,3.7} & 44.3\,$\pm$\,3.1 & 47.4\,$\pm$\,3.0 \\
\midrule
\textit{AntMaze Average}      & 45.9 & \textbf{58.6} & 40.8 & 40.1 \\
\bottomrule
\end{tabular}
\end{table*}
\begin{figure*}[tb]
  \centering
  \includegraphics[width=\textwidth]{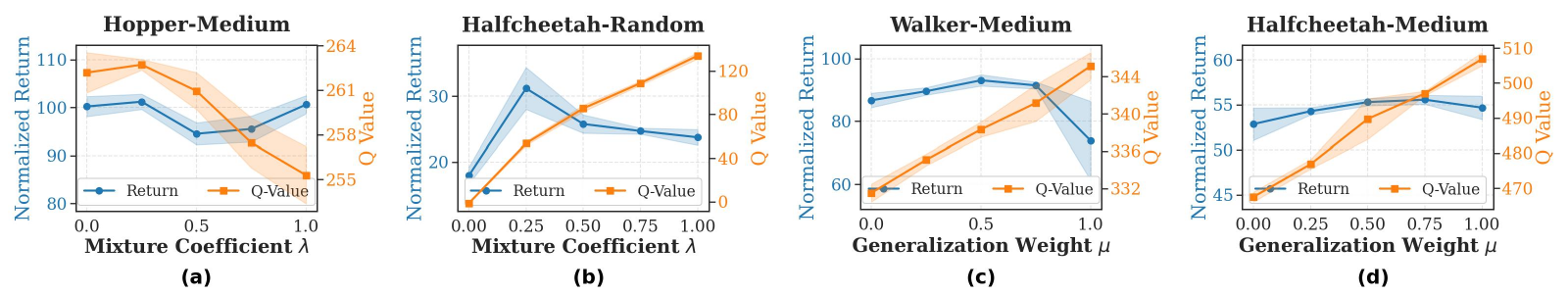}
  \caption{Performance and Q-value estimates of CSDG under different values of
the propagation weight $\lambda$ and OOD generalization weight $\mu$, averaged
over five random seeds.}
  \label{fig:ablation-mu-lambda1}
\end{figure*}
\subsection{Ablation Study}

\paragraph{Mixture coefficient $\lambda$.}
Intermediate values of $\lambda$ perform best. With $\mu=0.5$, we vary
$\lambda$ from 0 to 1 in Figures~\ref{fig:ablation-mu-lambda1}(a)
and~\ref{fig:ablation-mu-lambda1}(b). Small values underuse the local correction,
whereas large values introduce excessive smoothing. The middle range retains
useful generalization without either extreme.

\paragraph{OOD action generalization weight $\mu$.}
A moderate value of $\mu$ also gives the strongest results. With
$\lambda=0.25$, Figures~\ref{fig:ablation-mu-lambda1}(c)
and~\ref{fig:ablation-mu-lambda1}(d) sweep $\mu$ over $[0,1]$. Smaller values
place more weight on the expanded candidate and produce higher-variance
estimates, while larger values overemphasize the local candidate and constrain
the Q-function.

\paragraph{Additional ablations and sanity checks.}
The OOD-noise sweep shows that moderate perturbations (scale $=0.6$, clip $=0.5$) yield the strongest and most consistent performance, ranking best on seven of eight tasks and attaining the highest averages on both MuJoCo (94.9) and AntMaze (58.6); only halfcheetah-medium favors the smaller perturbation. Performance degrades under larger noise, especially on AntMaze, indicating that local OOD generalization benefits from a bounded perturbation radius. CSDG further maintains the lowest absolute Q-value bias throughout training. A full analysis is provided in Appendix~\ref{jianquanx}.

\section{Related Work}

Distribution shift makes OOD value estimation unreliable in offline RL.
Policy-constrained and pessimistic methods address this problem by restricting
policy deviation or suppressing uncertain OOD values
~\citep{kumar2020conservative,zhou2021plas}. In-sample methods, including
IQL~\citep{kostrikov2021offline} and $\mathcal{X}$QL~\citep{garg2023extreme}, take
a different route and avoid explicit maximization over OOD actions.

Several methods retain some form of controlled generalization. MCQ
~\citep{lyu2022mildly} combines an auxiliary generative model with mildly
conservative value learning. DOGE~\citep{li2022data} and SQOG
~\citep{yao2025offline} use dataset geometry to guide policy optimization and
Q-value control, respectively, while DMG~\citep{DMG} exploits mild local
generalization and limits its influence during bootstrapping. More recent work
uses clustered TD updates in C4~\citep{c4}, flow-based behavior modeling in
FAC~\citep{chae2026flow}, and uncertainty-adaptive in-sample learning in
UNIQ~\citep{upadhyay2026uniq}.

CSDG draws its CHN geometry from SQOG and shares DMG's motivation of controlling
generalized values during bootstrapping. The distinction lies in the target
construction: CSDG forms a two-scale smoothed target, measures it against an
expectile in-sample target, and inserts their difference as a local correction
in each Bellman update. In practice, asymmetric noise approximates the local and
expanded CHN candidates, while $\lambda$ scales the resulting correction.

\section{Conclusion and Limitations}
We presented CSDG, which combines a two-scale smoothed Q-value target with an
expectile in-sample target through a weighted local correction. The analysis
connects the anchor-based branch discrepancy to one-step, iterative, and
fixed-point behavior, and further gives a conditional performance comparison
for the induced policies. On 21 D4RL datasets, CSDG achieves the highest
aggregate scores on both Gym--MuJoCo and AntMaze, reaching 1199.4 and 445.8,
respectively. Moreover, it consistently improves both $\mathcal{X}$QL and IQL
across all evaluated D4RL v2 locomotion tasks, indicating that the two-scale
target generalizes across different in-sample backbones.

The current implementation uses task-specific noise scales and mixture
coefficients, and it does not track the effective anchoring geometry during
training. The ablations provide practical settings for the tested domains,
while adaptive tuning, geometry-aware diagnostics, and higher-dimensional
action spaces remain useful directions for future work.
\newpage
\bibliographystyle{iclr2027_conference}
\bibliography{iclr2027_conference}

\clearpage
\appendix
\setcounter{secnumdepth}{2}

\begin{center}
{\LARGE\bfseries Appendix}
\end{center}
\vspace{0.5em}

\section{Experimental Details}
\label{appd}

\subsection{Experimental Details in Offline Experiments}

We follow the standard D4RL evaluation protocol~\citep{fu2020d4rl}. For Gym
locomotion, each evaluation uses 10 trajectories, whereas each AntMaze
evaluation uses 100 trajectories. The main benchmark results are averaged over
the final 10 evaluations and 10 random seeds. The backbone comparison and
ablation studies use five random seeds, as stated in their captions. We report
the normalized D4RL score
\[
100\times
\frac{
\text{learned-policy return}-\text{random-policy return}
}{
\text{expert-policy return}-\text{random-policy return}
}.
\]

\begin{table}[htb]
\small
\centering
\caption{Hyperparameters used for CSDG.}
\label{tab:app_hyperparameters}
\begin{tabular}{lll}
\toprule
 & \textbf{Hyperparameter} & \textbf{Value} \\
\midrule
\multirow{11}{*}{CSDG}
  & Optimizer & Adam \\
  & Critic learning rate & $3\times10^{-4}$ \\
  & Actor learning rate & $3\times10^{-4}$ with cosine schedule \\
  & Batch size & 256 \\
  & Discount factor & 0.99 \\
  & Number of gradient steps & $10^6$ \\
  & Target update rate & 0.005 \\
  & Correction weight $\lambda$ & 0.25 \\
  & Candidate mixture weight $\mu$ & 0.5 \\
  & Behavior-cloning coefficient $\nu$ & 0.1 or 10 for Gym locomotion \\
  & & 0.5 for AntMaze \\
\midrule
\multirow{4}{*}{IQL specific}
  & Expectile $\tau$ & 0.7 for Gym locomotion \\
  & & 0.9 for AntMaze \\
  & Inverse temperature $\alpha$ & 3.0 for Gym locomotion \\
  & & 10.0 for AntMaze \\
\midrule
$\mathcal X$QL specific
  & Temperature $\alpha$ & 5.0 \\
\midrule
\multirow{2}{*}{Architecture}
  & Actor & input--256--256--output \\
  & Critic and value networks & input--256--256--1 \\
\bottomrule
\end{tabular}
\end{table}

The IQL implementation uses $\tau=0.7$ and $\alpha=3$ on Gym locomotion and
$\tau=0.9$ and $\alpha=10$ on AntMaze, following the original
configuration~\citep{kostrikov2021offline}. For the $\mathcal X$QL backbone, we
use temperature $5.0$ on the medium, medium-replay, and medium-expert datasets,
following~\citet{garg2023extreme}. CSDG uses $\lambda=0.25$ and $\mu=0.5$
throughout the main comparison. The coefficient $\nu$ is $0.1$ on medium,
medium-replay, and random datasets, $10$ on expert and medium-expert datasets,
and $0.5$ on AntMaze. Table~\ref{tab:app_hyperparameters} summarizes the full
configuration.

\section{Extended Related Work and Discussions}

\noindent\textbf{Model-free offline RL.}
Model-free offline RL methods commonly control distribution shift by limiting
the policy or modifying the learned values. Representative approaches include
importance weighting~\citep{precup2001off,sutton2016emphatic,liu2019off,
nachum2019dualdice,gelada2019off}, explicit policy constraints
~\citep{kumar2019stabilizing,wu2019behavior,ghasemipour2021emaq,
fujimoto2021minimalist,fakoor2021continuous}, and latent-action
representations~\citep{zhou2021plas,ajay2020opal}. Other methods apply
conservative value objectives~\citep{kumar2020conservative,ma2021conservative},
uncertainty-aware updates~\citep{wu2021uncertainty,zanette2021provable,
bai2022pessimistic}, or in-sample learning
~\citep{wang2018exponentially,chen2020bail,kostrikov2021offline}. These
approaches differ in how much they rely on values outside the empirical action
support. CSDG retains an in-sample expectile reference and adds a bounded local
correction constructed from two perturbation branches.

\medskip
\noindent\textbf{Model-based offline RL.}
Model-based methods first learn a transition model and then optimize a policy
with synthetic rollouts. Existing methods use model uncertainty
~\citep{yu2020mopo,kidambi2020morel,diehl2021umbrella}, conservative value
learning~\citep{yu2021combo}, representation learning
~\citep{lee2021representation,rafailov2021offline}, behavior regularization
~\citep{matsushima2020deployment}, or sequence modeling
~\citep{chen2021decision,janner2021offline,meng2023offline}. Their performance
depends on both policy learning and model accuracy, and the learned model adds
an additional training component. CSDG is model-free and modifies only the
value target and actor objective.

\subsection{Discussion of Recent Local-Generalization Methods}
\label{morework}

MCQ~\citep{lyu2022mildly} uses a conditional generative model to produce OOD
actions and constructs a mildly conservative critic objective. The generated
actions provide a learned approximation to the behavior distribution, while
the conservative coefficient determines their contribution to value learning.

SQOG~\citep{yao2025offline} introduces CHN geometry and connects perturbed
actions with neighboring dataset values through an auxiliary smoothing loss.
Its locality control is therefore implemented as an additional critic
regularizer.

DMG~\citep{DMG} defines a mildly generalized policy and blends its value with an
in-sample value inside the Bellman target. The mixture coefficient directly
controls the amount of generalized information propagated by bootstrapping.

CSDG combines the CHN geometric reference with an expectile-referenced target.
Two bounded perturbation branches form the smoothed value, and their difference
from the in-sample expectile defines the correction. The correction weight
$\lambda$ then determines how much of this difference enters each Bellman
update. Table~\ref{tab:app_mechanistic_comparison} summarizes these distinctions.

\begin{table}[t]
\centering
\caption{Mechanistic comparison of MCQ, SQOG, DMG, and CSDG.}
\label{tab:app_mechanistic_comparison}
\resizebox{\textwidth}{!}{%
\begin{tabular}{lcccc}
\toprule
\textbf{Property} & \textbf{MCQ} & \textbf{SQOG} & \textbf{DMG} & \textbf{CSDG} \\
\midrule
Generalized-action source
& Conditional VAE & Bounded noise & Generalized policy & Two-scale bounded noise \\
Geometric reference
& Learned support & CHN & Local policy support & State-conditional CHN \\
Control location
& Critic objective & Auxiliary critic loss & Bellman target & Bellman target \\
In-sample reference
& Conservative target & Neighboring values & In-sample maximum & Expectile value \\
Additional behavior model
& Yes & No & No & No \\
Propagation control
& Conservative coefficient & Smoothing regularizer & Mixture coefficient & Correction weight $\lambda$ \\
\bottomrule
\end{tabular}%
}
\end{table}

The target construction also distinguishes CSDG from DMG. DMG evaluates
actions sampled from its generalized policy, whereas CSDG evaluates two
perturbation scales around the target action. Figure~\ref{fig:app_dmg_csdg}
shows how this difference affects the return and learned Q-values as
$\lambda$ increases.

\section{Additional Experimental Results}
\label{appe}

\subsection{Learning Curves of CSDG during Offline Training}

The learning curves remain stable across the evaluated Gym locomotion and
AntMaze tasks. Figures~\ref{fig:dataset1} and~\ref{fig:dataset2} report the
mean normalized score over 10 random seeds, with shaded regions showing one
standard deviation. The curves also show that the final-score averages in the
main paper are representative of the late-training behavior rather than an
isolated evaluation point.
We also compare the training time of CQL, IQL, MCQ, SQOG, and CSDG for one
million gradient steps. Figure~\ref{fig:app_accuracy}(b) reports the runtime on
the same hardware configuration. CSDG does not train an additional behavior
model and remains close to the in-sample implementation in computational cost.

\subsection{Comparisons with Local-Generalization Methods}

CSDG is competitive with MCQ, SQOG, and DMG on the 12 Gym tasks for which all
four methods report results. Table~\ref{tab:app_local_methods} shows that CSDG
obtains the highest aggregate score on this shared subset. The largest gains
occur on walker2d-medium-replay and the medium-expert tasks, while the random
datasets remain more variable.

\begin{table}[t]
\small
\centering
\caption{Normalized scores on the 12-task Gym subset shared by MCQ, SQOG, DMG,
and CSDG. CSDG results are averaged over 10 random seeds.}
\label{tab:app_local_methods}
\begin{tabular}{lrrrr}
\toprule
Dataset-v2 & MCQ & SQOG & DMG & CSDG (Ours) \\
\midrule
halfcheetah-m   & 58.3 & \textbf{59.2} & 54.9 & 55.3$\pm$0.4 \\
hopper-m        & 73.6 & 100.8 & 100.6 & \textbf{101.2$\pm$1.6} \\
walker2d-m      & 88.4 & 82.9 & 92.4 & \textbf{93.1$\pm$1.8} \\
halfcheetah-m-r & \textbf{51.5} & 46.4 & 51.4 & \textbf{51.5$\pm$0.2} \\
hopper-m-r      & 99.6 & 100.9 & 101.9 & \textbf{102.5$\pm$1.4} \\
walker2d-m-r    & 83.3 & 88.3 & 89.7 & \textbf{96.3$\pm$2.5} \\
halfcheetah-m-e & 85.4 & 92.6 & 91.1 & \textbf{94.1$\pm$1.7} \\
hopper-m-e      & 106.2 & 109.2 & 110.4 & \textbf{111.3$\pm$1.6} \\
walker2d-m-e    & 110.3 & 109.0 & 114.4 & \textbf{114.8$\pm$0.8} \\
halfcheetah-r   & 23.6 & 25.6 & 28.8 & \textbf{33.7$\pm$6.8} \\
hopper-r        & \textbf{31.0} & 15.6 & 20.4 & 8.2$\pm$1.1 \\
walker2d-r      & 10.3 & \textbf{17.7} & 4.8 & 13.6$\pm$1.9 \\
\midrule
Total & 821.5 & 848.2 & 860.8 & \textbf{875.6} \\
\bottomrule
\end{tabular}
\end{table}

\subsection{Additional Baseline Comparisons}
\label{ABC}

Table~\ref{tab:app_recent_baselines} extends the comparison to IAC, SPOT,
SCAS, STR, CPI, CPED, and ANQ. CSDG achieves the largest aggregate score on
both benchmark groups among the methods shown. The AntMaze results are
particularly consistent on the two large-maze tasks, where CSDG obtains 56.1
and 61.2.

\begin{table}[t]
\centering
\caption{Normalized scores for additional recent baselines on Gym locomotion
and AntMaze. CSDG results are averaged over 10 random seeds.}
\label{tab:app_recent_baselines}
\resizebox{\textwidth}{!}{%
\begin{tabular}{lrrrrrrrr}
\toprule
Dataset-v2 & IAC & SPOT & SCAS & STR & CPI & CPED & ANQ & CSDG (Ours) \\
\midrule
halfcheetah-m   & $51.6\pm0.3$ & $58.4\pm1.0$ & $46.6\pm0.2$ & $51.8\pm0.3$ & \textbf{$64.4\pm1.3$} & $61.8\pm1.6$ & $61.8\pm1.4$ & $55.3\pm0.4$ \\
hopper-m        & $74.6\pm11.5$ & $86.0\pm8.7$ & \textbf{$102.5\pm0.3$} & $101.3\pm0.4$ & $98.5\pm3.0$ & $100.1\pm2.8$ & $100.9\pm0.6$ & $101.2\pm1.6$ \\
walker2d-m      & $85.2\pm0.4$ & $86.4\pm2.7$ & $82.3\pm3.0$ & $85.9\pm1.1$ & $85.8\pm0.8$ & $90.2\pm1.7$ & $82.9\pm1.5$ & \textbf{$93.1\pm1.8$} \\
halfcheetah-m-r & $47.2\pm0.3$ & $52.2\pm1.2$ & $44.0\pm0.3$ & $47.5\pm0.2$ & $54.6\pm1.3$ & \textbf{$55.8\pm2.9$} & $55.5\pm1.4$ & $51.5\pm0.2$ \\
hopper-m-r      & \textbf{$103.2\pm1.0$} & $100.2\pm1.9$ & $101.6\pm1.0$ & $100.0\pm1.2$ & $101.7\pm1.6$ & $98.1\pm2.1$ & $101.5\pm2.7$ & $102.5\pm1.4$ \\
walker2d-m-r    & $93.2\pm1.8$ & $91.6\pm2.8$ & $78.1\pm4.5$ & $85.7\pm2.2$ & $91.8\pm2.9$ & $91.9\pm0.9$ & $92.7\pm3.8$ & \textbf{$96.3\pm2.5$} \\
halfcheetah-m-e & $92.9\pm0.7$ & $86.9\pm4.3$ & $91.7\pm2.7$ & \textbf{$94.9\pm1.6$} & $94.7\pm1.1$ & $85.4\pm10.9$ & $94.2\pm0.8$ & $94.1\pm1.7$ \\
hopper-m-e      & $109.3\pm4.0$ & $99.3\pm7.1$ & $109.7\pm3.5$ & \textbf{$111.9\pm0.6$} & $106.4\pm4.3$ & $95.3\pm13.5$ & $107.0\pm4.9$ & $111.3\pm1.6$ \\
walker2d-m-e    & $110.1\pm0.1$ & $112.0\pm0.5$ & $108.4\pm3.7$ & $110.2\pm0.1$ & $110.9\pm0.4$ & $113.0\pm1.4$ & $111.7\pm0.2$ & \textbf{$114.8\pm0.8$} \\
halfcheetah-e   & $94.5\pm0.5$ & -- & -- & $95.2\pm0.3$ & \textbf{$96.5\pm0.2$} & -- & $95.9\pm0.4$ & $95.6\pm1.4$ \\
hopper-e        & $110.6\pm1.9$ & -- & -- & $111.2\pm0.3$ & $112.2\pm0.5$ & -- & $111.4\pm2.5$ & \textbf{$113.3\pm1.6$} \\
walker2d-e      & $114.8\pm1.2$ & -- & -- & $110.1\pm0.1$ & $110.6\pm0.1$ & -- & $111.8\pm0.1$ & \textbf{$114.9\pm0.3$} \\
halfcheetah-r   & $20.9\pm1.2$ & -- & $12.2\pm3.2$ & $20.6\pm1.1$ & $29.7\pm1.1$ & -- & $24.9\pm1.0$ & \textbf{$33.7\pm6.8$} \\
hopper-r        & $31.3\pm0.3$ & -- & \textbf{$31.4\pm0.1$} & $31.3\pm0.3$ & $29.5\pm3.7$ & -- & $31.1\pm0.2$ & $8.2\pm1.1$ \\
walker2d-r      & $3.0\pm1.3$ & -- & $1.4\pm1.1$ & $4.7\pm3.8$ & $5.9\pm1.7$ & -- & $11.2\pm9.5$ & \textbf{$13.6\pm1.9$} \\
\midrule
Locomotion total & 1142.4 & -- & -- & 1162.2 & 1193.2 & -- & 1194.5 & \textbf{1199.4} \\
\midrule
antmaze-u   & $77.6\pm3.8$ & $93.5\pm2.4$ & $90.4\pm4.3$ & $93.6\pm4.0$ & \textbf{$98.8\pm1.1$} & $96.8\pm2.6$ & $96.0\pm1.6$ & $92.6\pm1.9$ \\
antmaze-u-d & $71.2\pm8.6$ & $40.7\pm5.1$ & $63.8\pm16.7$ & $77.4\pm7.2$ & \textbf{$88.6\pm5.7$} & $55.6\pm2.2$ & $80.2\pm1.8$ & $80.1\pm7.1$ \\
antmaze-m-p & $72.0\pm7.6$ & $74.7\pm4.6$ & $76.6\pm3.9$ & $82.6\pm5.4$ & $82.4\pm5.8$ & \textbf{$85.1\pm3.4$} & $76.2\pm3.3$ & $81.6\pm4.1$ \\
antmaze-m-d & $74.2\pm4.1$ & $79.1\pm5.6$ & $80.4\pm5.4$ & \textbf{$87.0\pm4.2$} & $80.4\pm8.9$ & $72.1\pm2.9$ & $77.2\pm6.1$ & $74.2\pm4.7$ \\
antmaze-l-p & \textbf{$57.0\pm7.4$} & $35.3\pm8.3$ & $49.0\pm4.0$ & $42.8\pm8.7$ & $20.6\pm16.3$ & $34.9\pm5.3$ & $56.2\pm4.9$ & $56.1\pm4.3$ \\
antmaze-l-d & $47.2\pm9.4$ & $36.3\pm13.7$ & $50.6\pm7.2$ & $46.8\pm7.6$ & $45.2\pm6.9$ & $32.3\pm7.4$ & $55.8\pm4.0$ & \textbf{$61.2\pm3.7$} \\
\midrule
AntMaze total & 399.2 & 359.6 & 410.8 & 430.2 & 416.0 & 376.8 & 441.6 & \textbf{445.8} \\
\bottomrule
\end{tabular}%
}
\end{table}

\subsection{Comparison with DMG}

CSDG and DMG both control the contribution of generalized values, but they
construct the generalized branch differently. DMG evaluates actions from a
generalized policy, whereas CSDG averages two bounded perturbation branches
around the target action. Figure~\ref{fig:app_dmg_csdg} compares their return
and value estimates as $\lambda$ varies. The two methods are similar for small
correction weights, while their behavior separates as the generalized branch
receives more weight.

\begin{figure}[t]
\small
  \centering
  \begin{subfigure}[b]{0.48\textwidth}
    \centering
    \includegraphics[width=\linewidth]{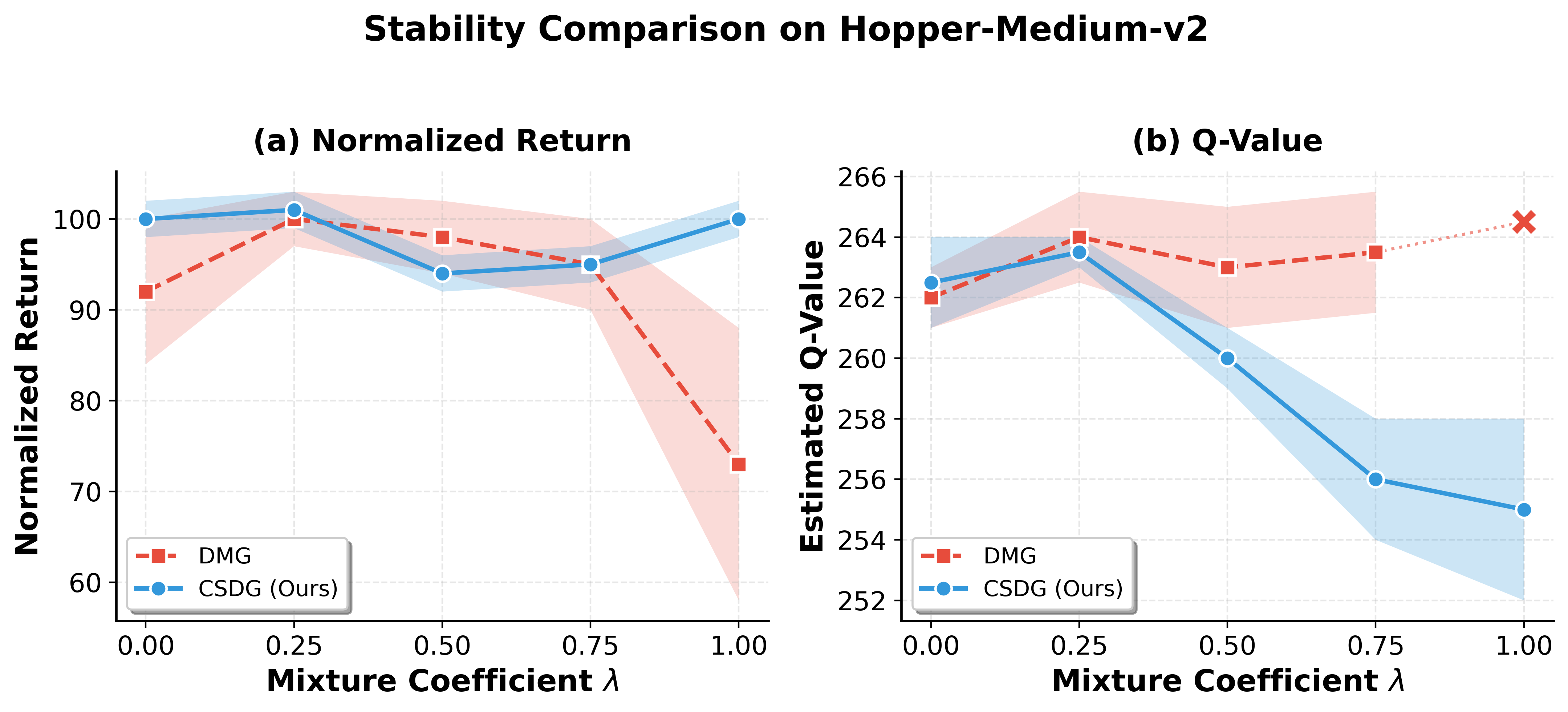}
  \end{subfigure}
  \hfill
  \begin{subfigure}[b]{0.48\textwidth}
    \centering
    \includegraphics[width=\linewidth]{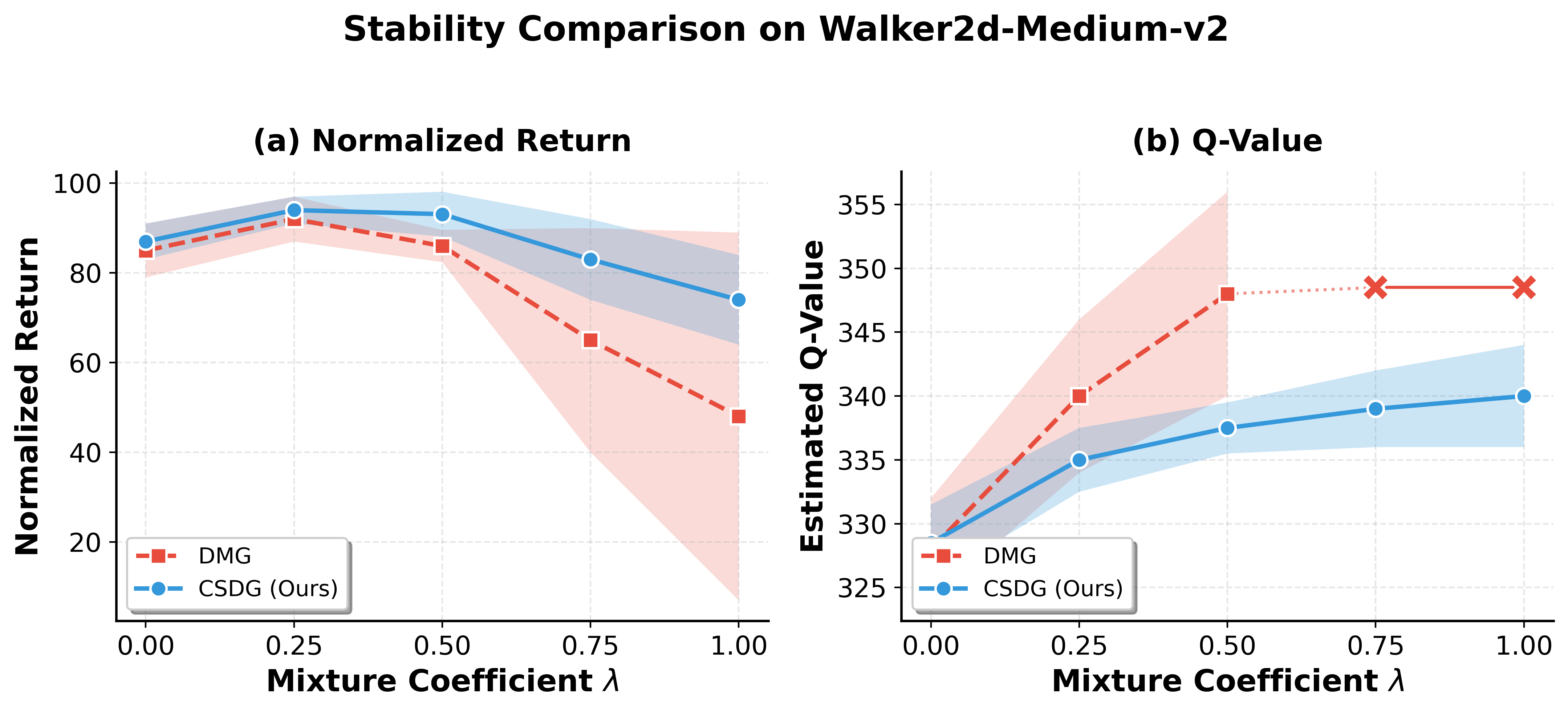}
  \end{subfigure}
  \caption{Return and Q-value estimates of DMG and CSDG under different values
  of the generalized-branch weight $\lambda$.}
  \label{fig:app_dmg_csdg}
\end{figure}

\subsection{Additional Hyperparameter Sensitivity and Ablation Studies}
\label{jianquanx}

\paragraph{Correction and candidate-mixture weights.}

Intermediate values of both weights produce the strongest results. We vary
$\lambda$ and $\mu$ over $\{0,0.25,0.5,0.75,1\}$ while fixing the other at its
default value. Figure~\ref{fig:app_lambda_mu} shows that small $\lambda$ values
underuse the smoothed correction, whereas large values place most of the target
weight on the generalized branch. For $\mu$, smaller values emphasize the
OOD-oriented candidate and larger values emphasize the in-sample-oriented
candidate. The middle range balances the two branches across both tasks.

\begin{figure}[htb]
\small
  \centering
  \begin{tabular}{cccc}
    \includegraphics[width=0.23\textwidth]{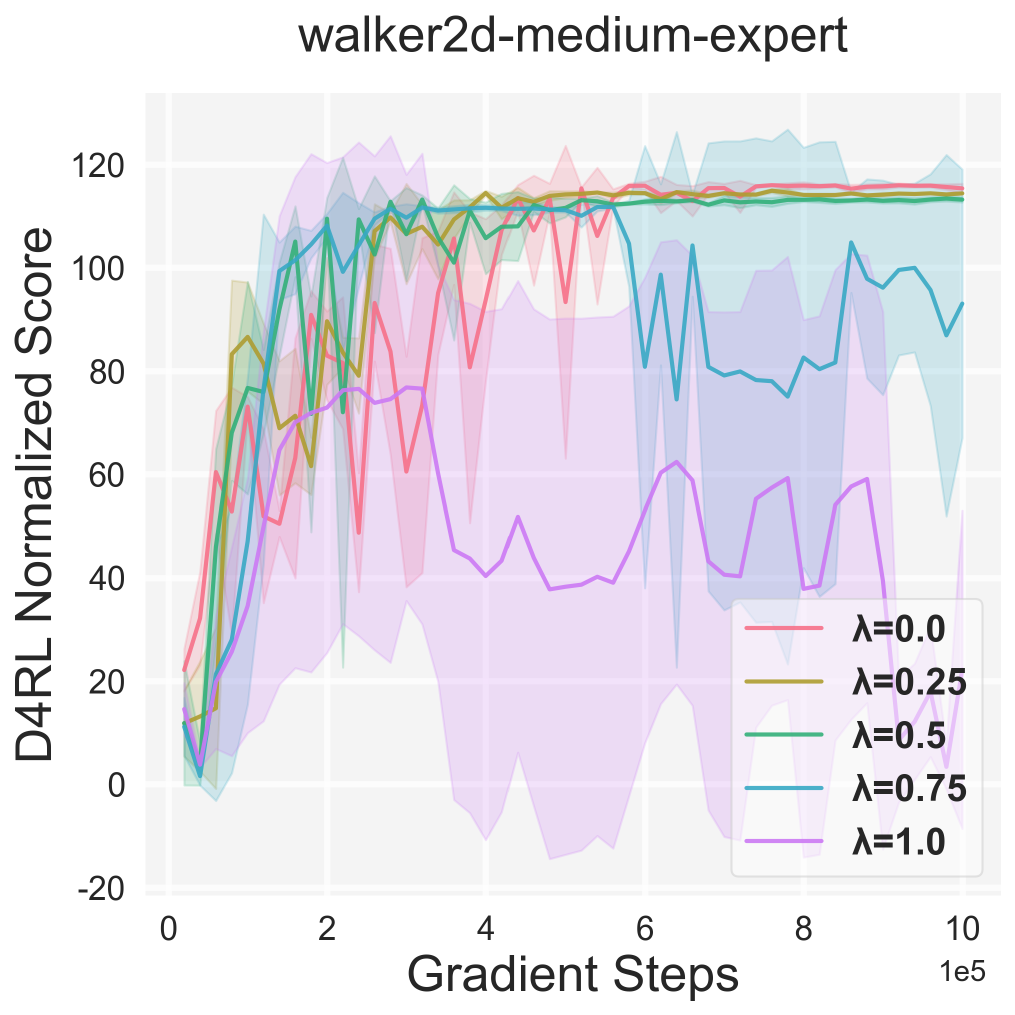} &
    \includegraphics[width=0.23\textwidth]{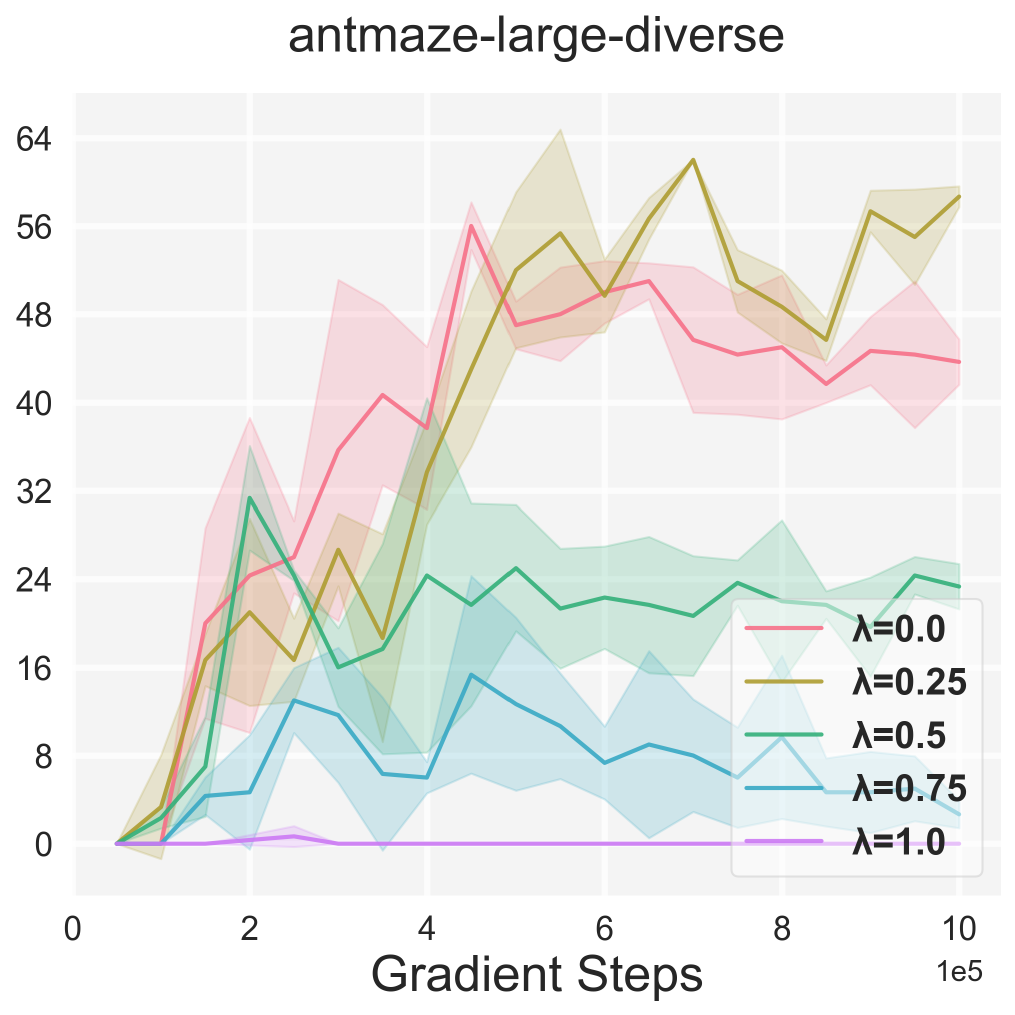} &
    \includegraphics[width=0.23\textwidth]{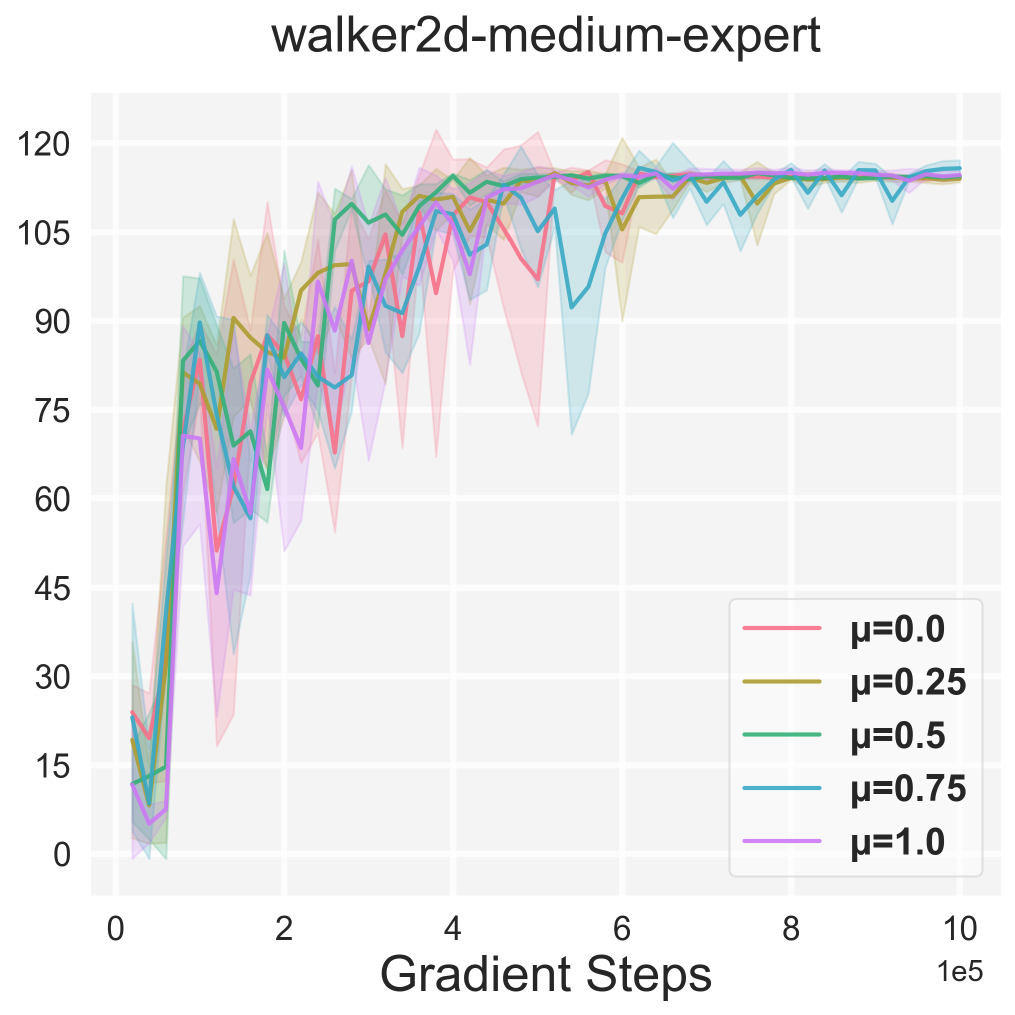} &
    \includegraphics[width=0.23\textwidth]{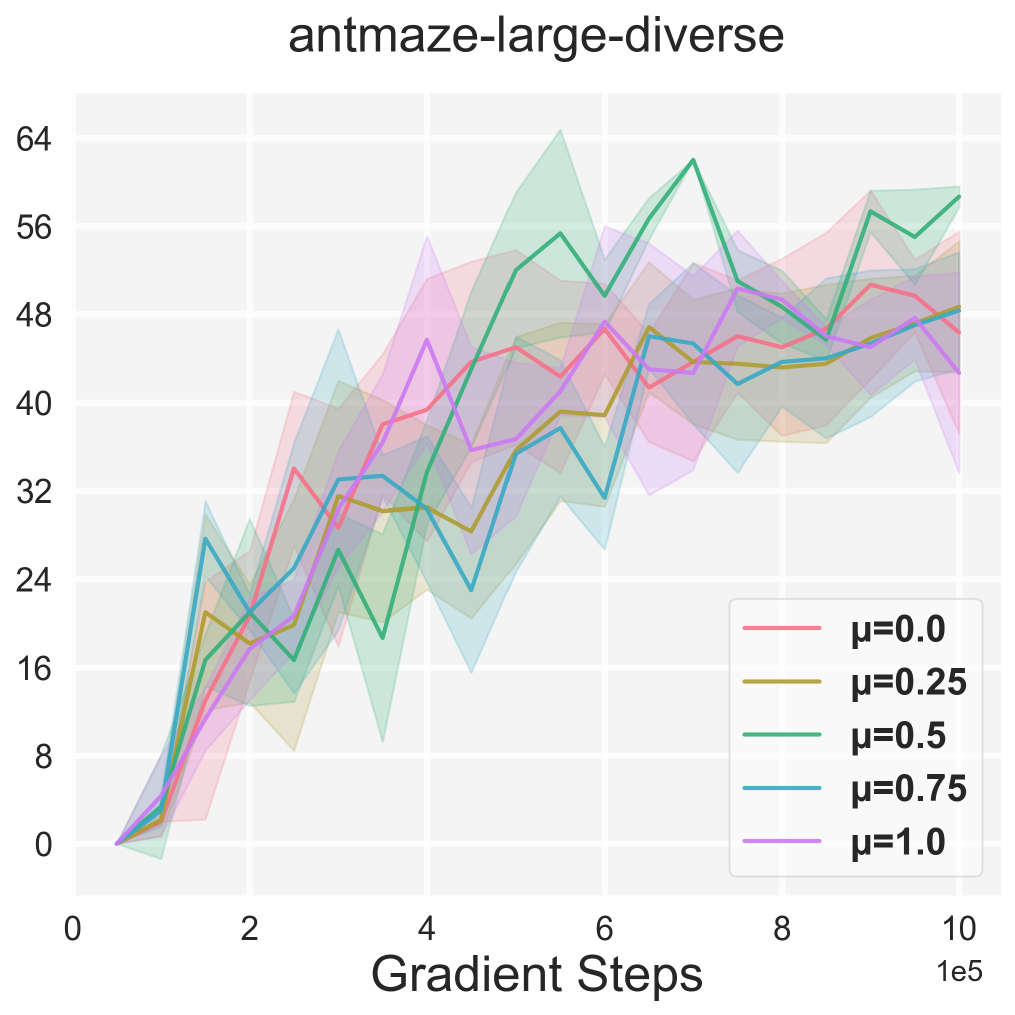} \\
    (a) & (b) & (c) & (d)
  \end{tabular}
  \caption{Sensitivity to the correction weight $\lambda$ (a--b) and the
  candidate mixture weight $\mu$ (c--d) on walker2d-medium-expert and
  antmaze-large-diverse. Results are averaged over five random seeds.}
  \label{fig:app_lambda_mu}
\end{figure}

\begin{figure}[htb]
\small
  \centering
  \includegraphics[width=0.4\textwidth]{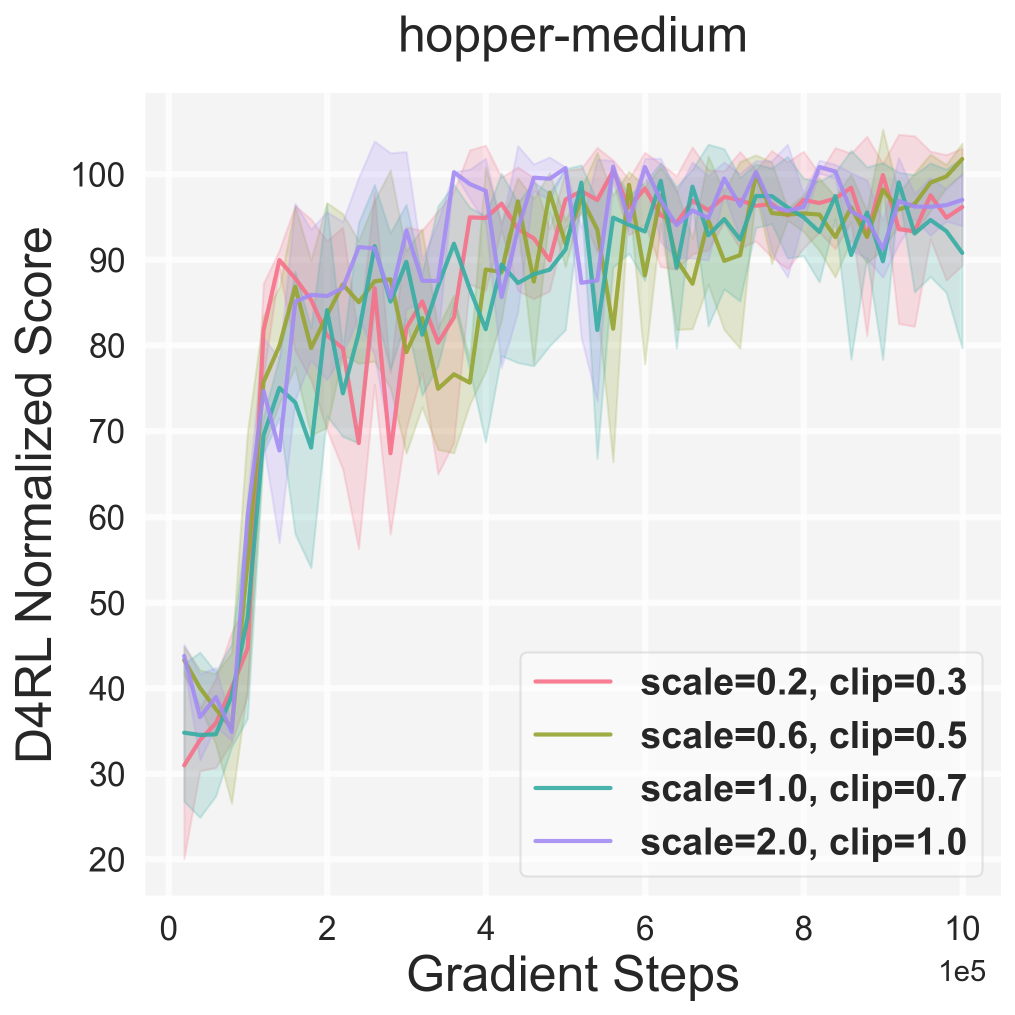}
  \includegraphics[width=0.4\textwidth]{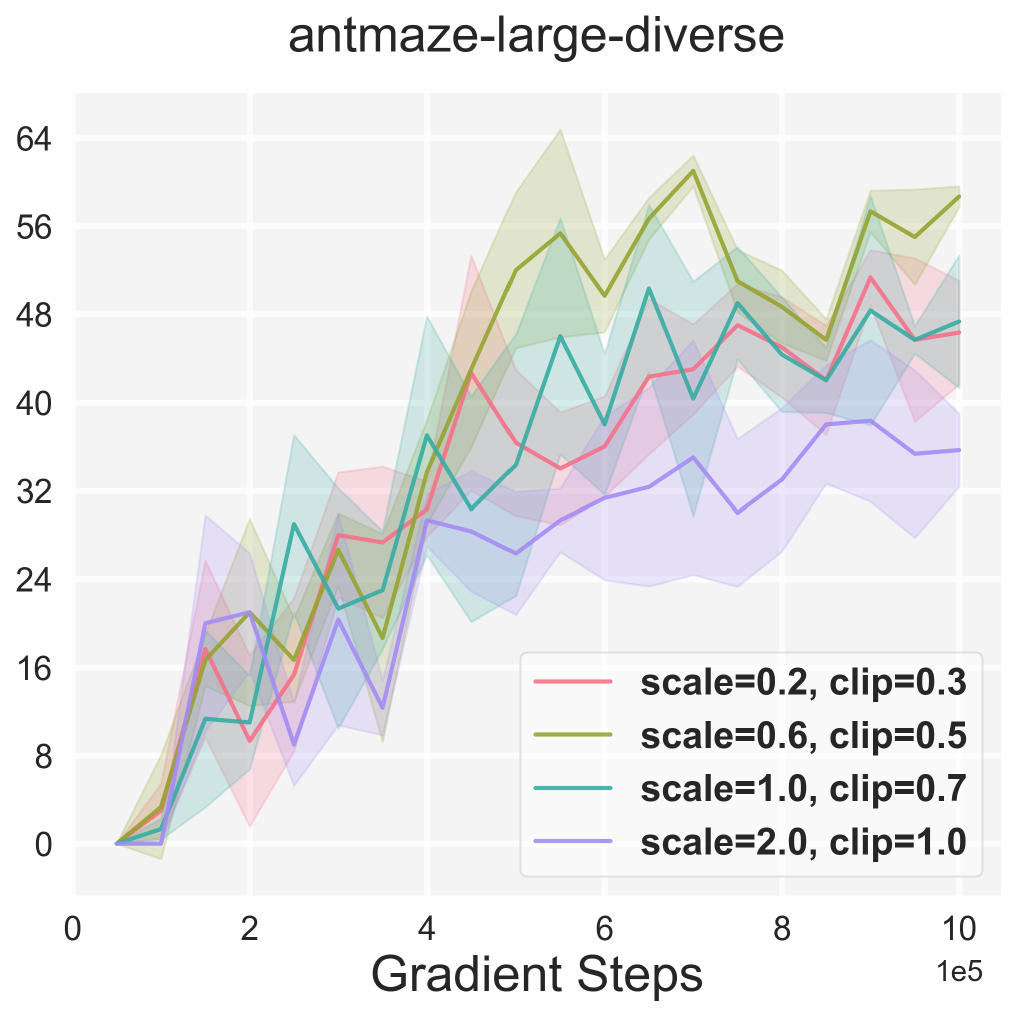}
  \caption{Performance under different Gaussian noise scales and clipping
  values on hopper-medium-v2 and antmaze-large-diverse-v2.}
  \label{fig:app_noise}
\end{figure}

\subsection{Additional Sanity Checks: Accuracy}

\paragraph{Q-value estimation accuracy and runtime.}

CSDG gives the smallest absolute Q-value bias among the four evaluated methods
on hopper-medium-v2. We estimate reference returns with Monte Carlo rollouts
and compare them with the learned Q-values in Figure~\ref{fig:app_accuracy}(a).
The corresponding runtime comparison in Figure~\ref{fig:app_accuracy}(b) shows
that CSDG remains close to the in-sample baseline and avoids the additional
behavior-model training used by MCQ.
\begin{figure}[htb]
\small
  \centering
  \begin{tabular}{cc}
    \includegraphics[width=0.45\textwidth]{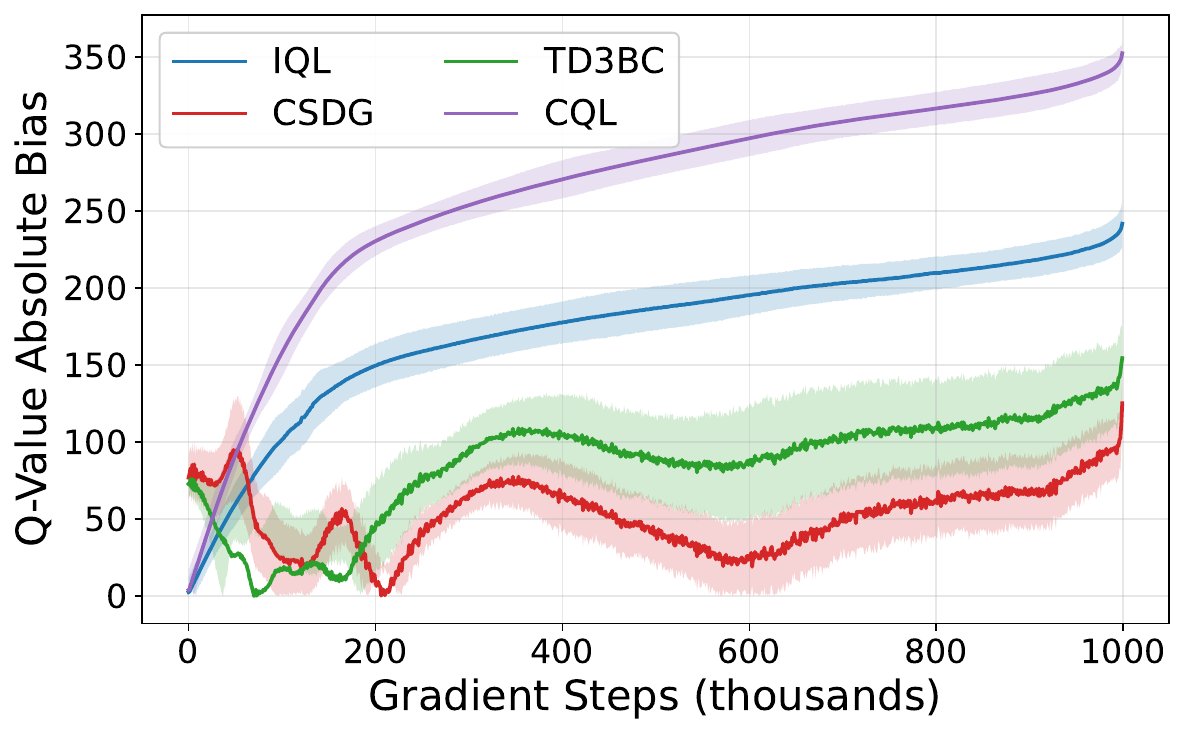} &
    \includegraphics[width=0.48\textwidth]{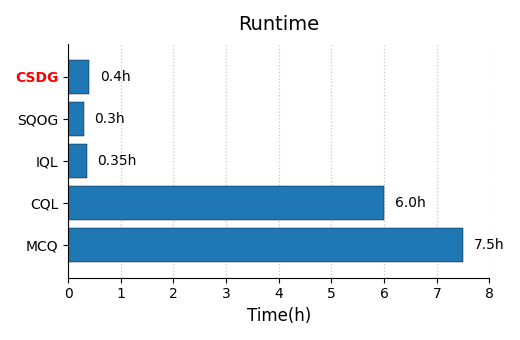} \\
    (a) & (b)
  \end{tabular}
  \caption{(a) Absolute Q-value bias on hopper-medium-v2. Shaded regions show
  one standard deviation. (b) Average training time on Gym locomotion using a
  GeForce RTX 4090.}
  \label{fig:app_accuracy}
\end{figure}
\begin{figure}[htb]
\small
  \centering
  \begin{subfigure}[b]{\linewidth}
    \centering
    \includegraphics[width=\linewidth]{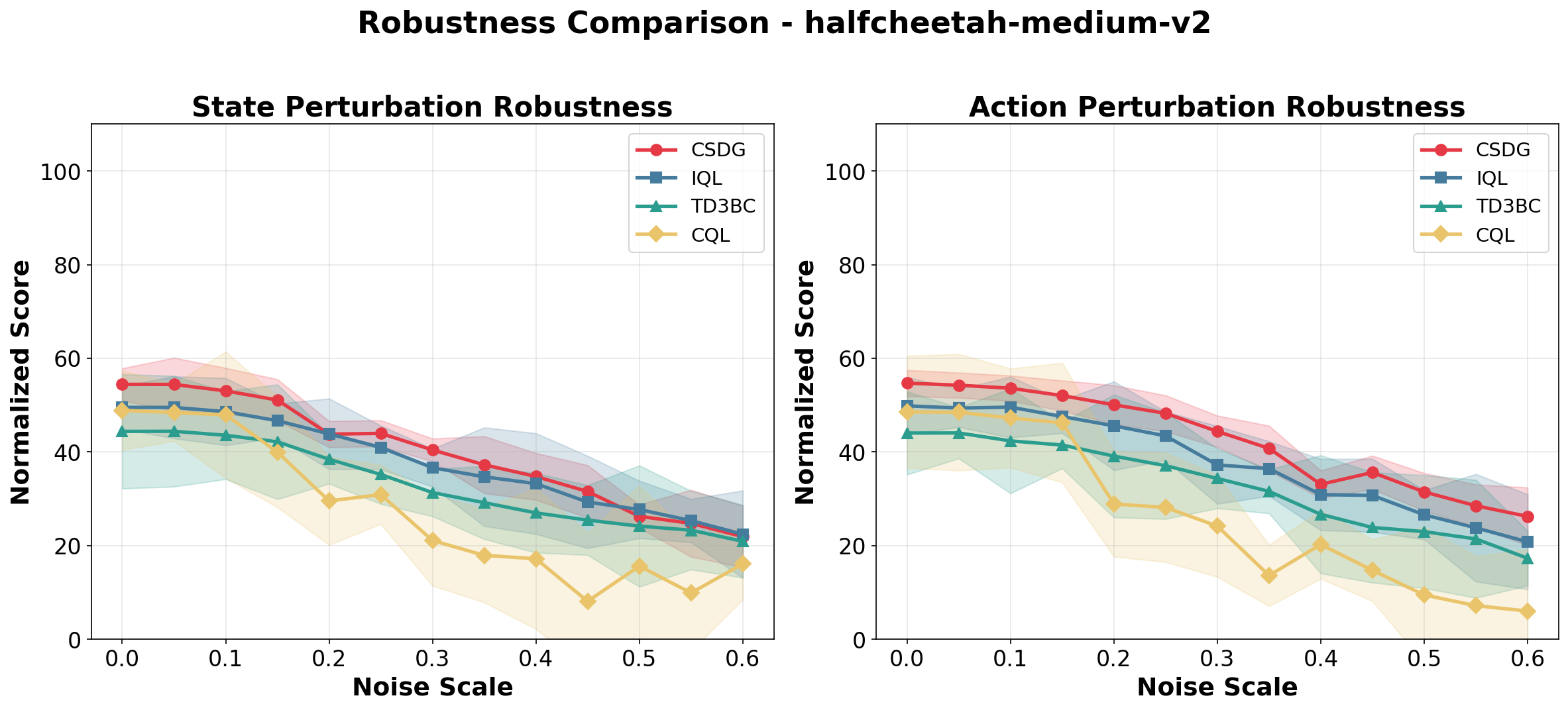}
  \end{subfigure}
  \begin{subfigure}[b]{\linewidth}
    \centering
    \includegraphics[width=\linewidth]{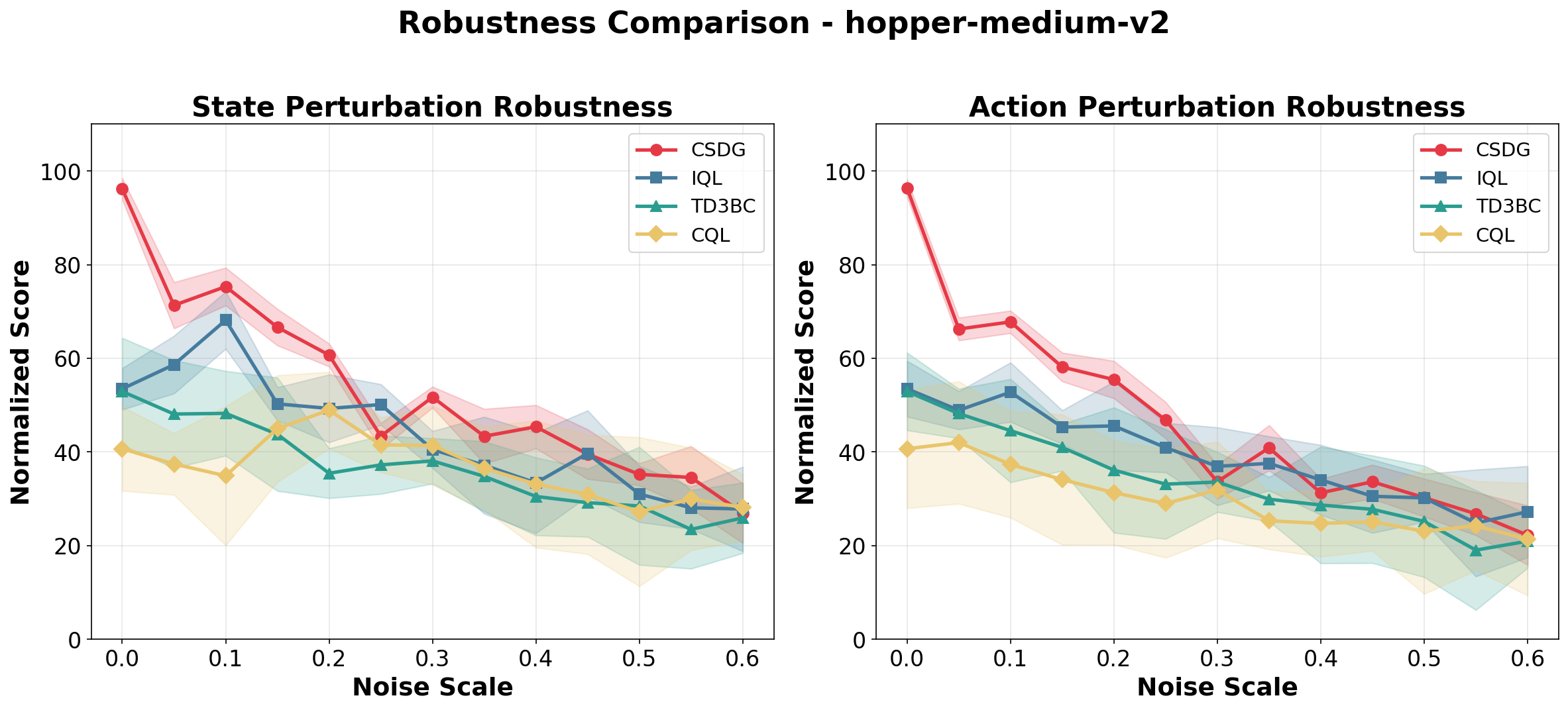}
  \end{subfigure}
  \begin{subfigure}[b]{\linewidth}
    \centering
    \includegraphics[width=\linewidth]{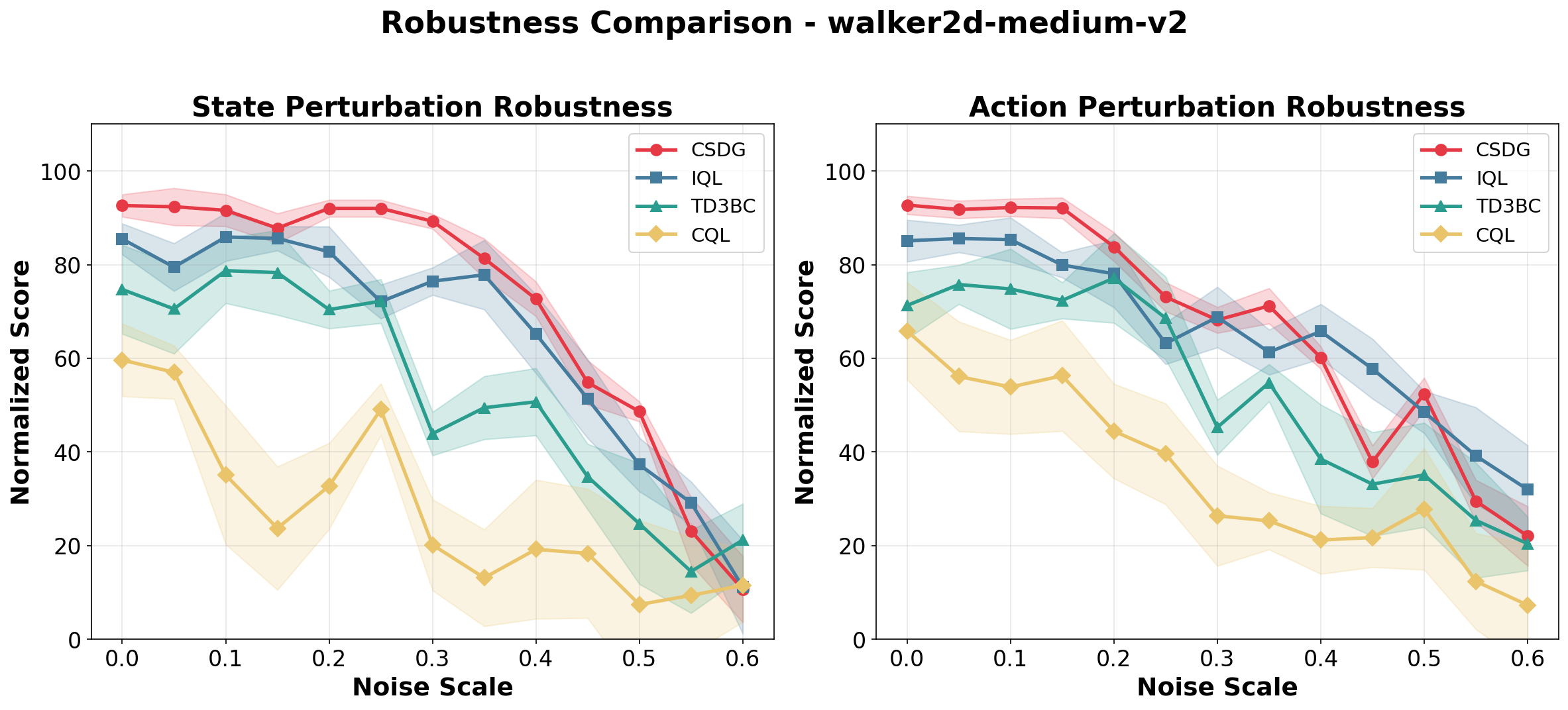}
  \end{subfigure}
  \caption{Normalized returns under observation and action perturbations on
  three Gym medium datasets. The policies are evaluated without retraining.}
  \label{fig:app_robustness}
\end{figure}
\begin{figure}[htb]
  \centering
  \begin{tabular}{ccccc}
    \includegraphics[width=0.18\textwidth]{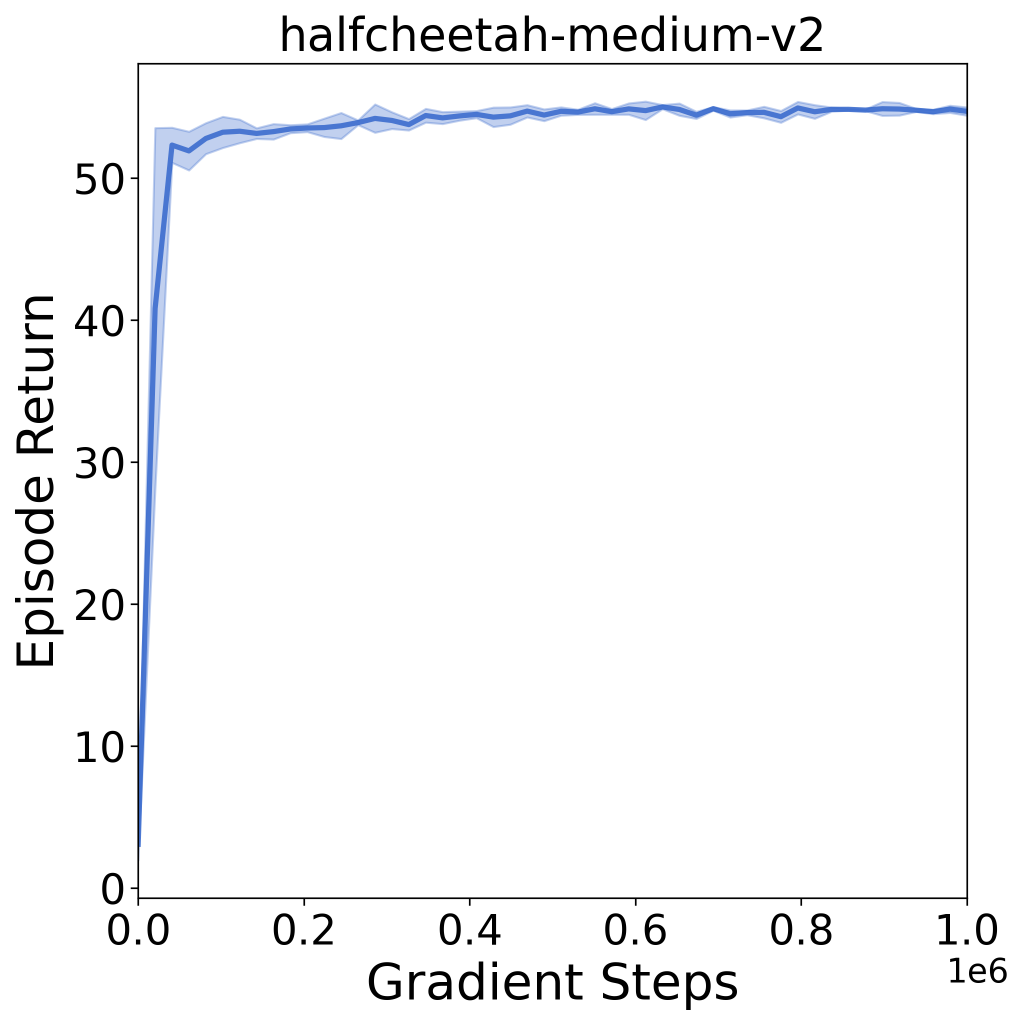} &
    \includegraphics[width=0.18\textwidth]{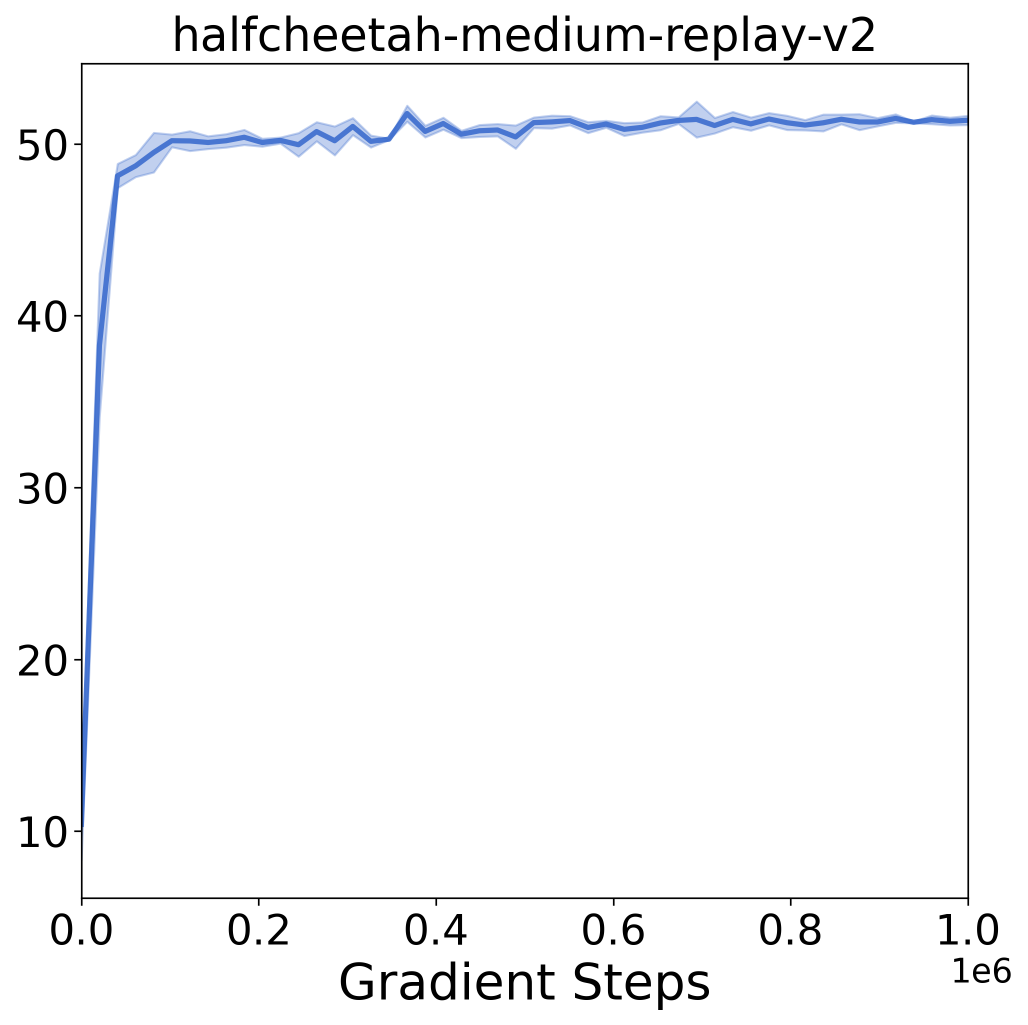} &
    \includegraphics[width=0.18\textwidth]{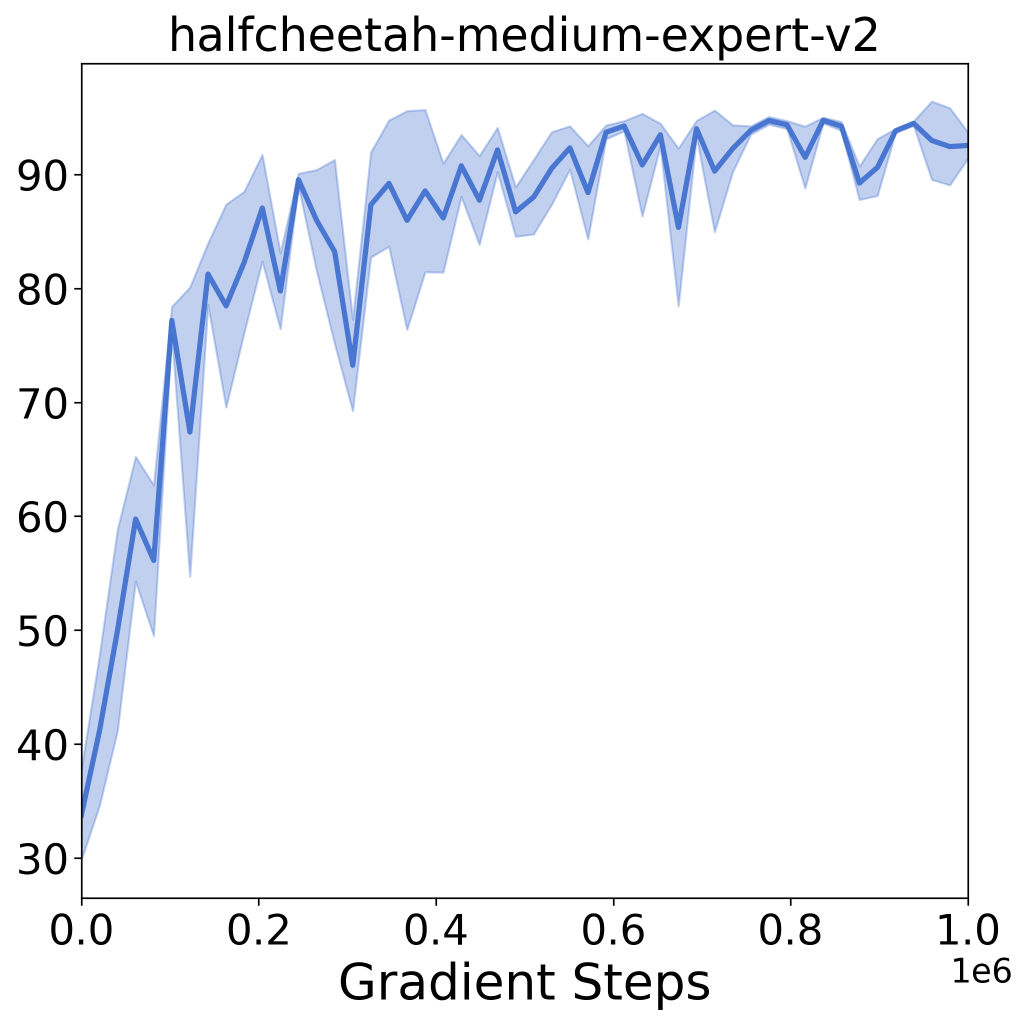} &
    \includegraphics[width=0.18\textwidth]{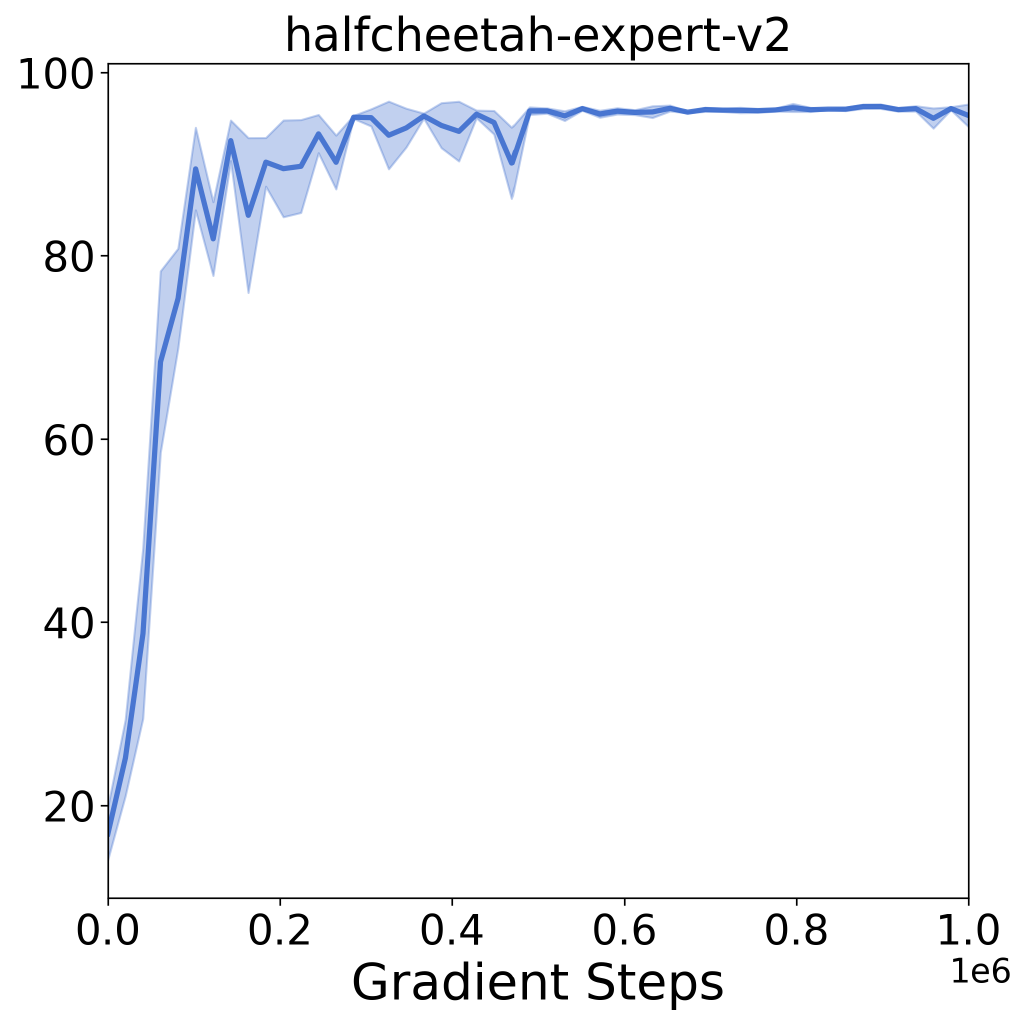} &
    \includegraphics[width=0.18\textwidth]{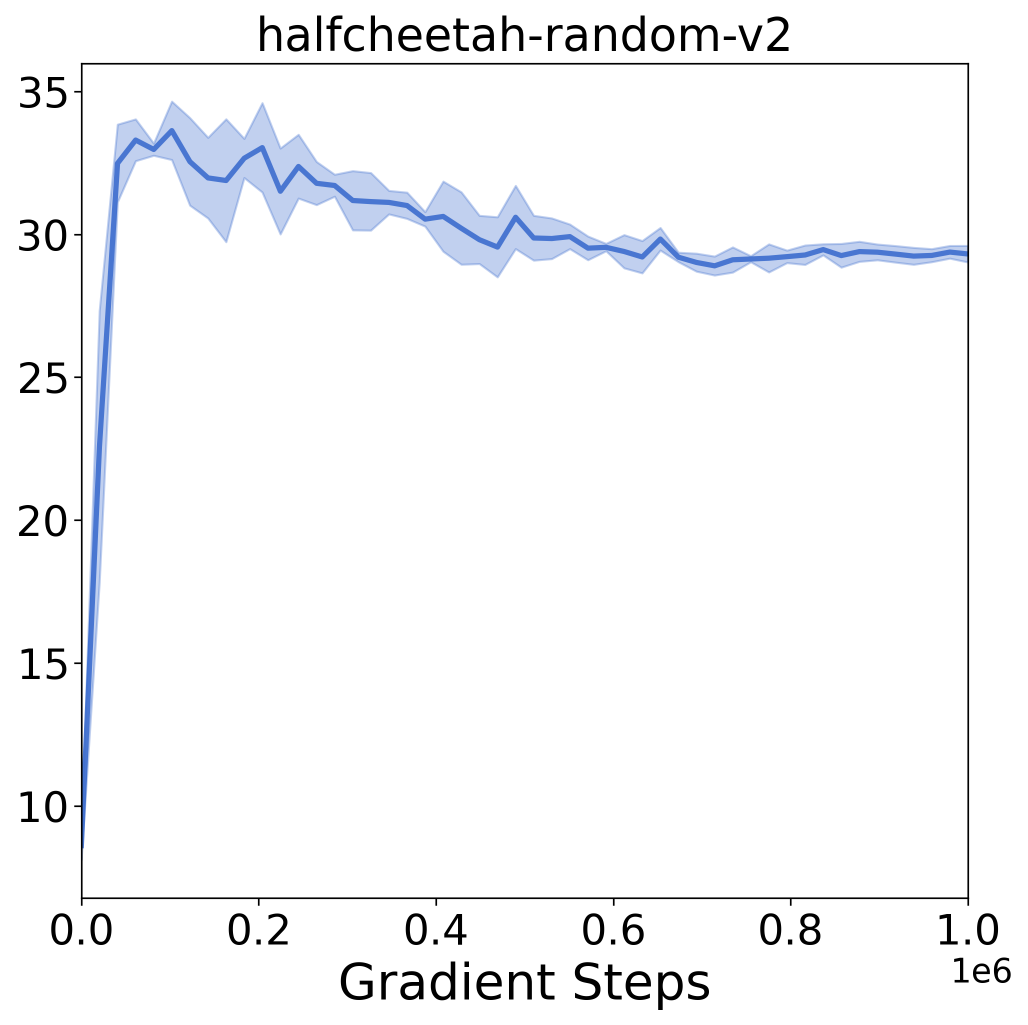} \\
    \includegraphics[width=0.18\textwidth]{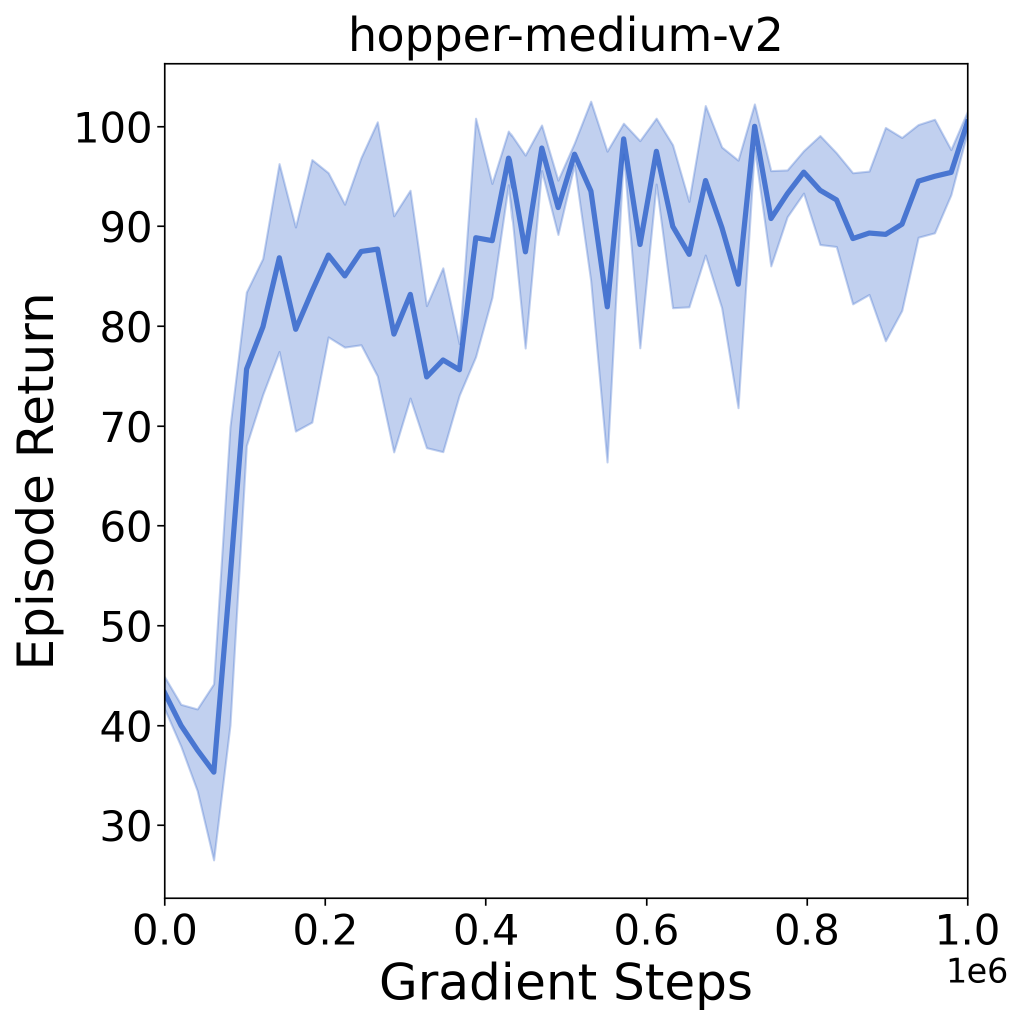} &
    \includegraphics[width=0.18\textwidth]{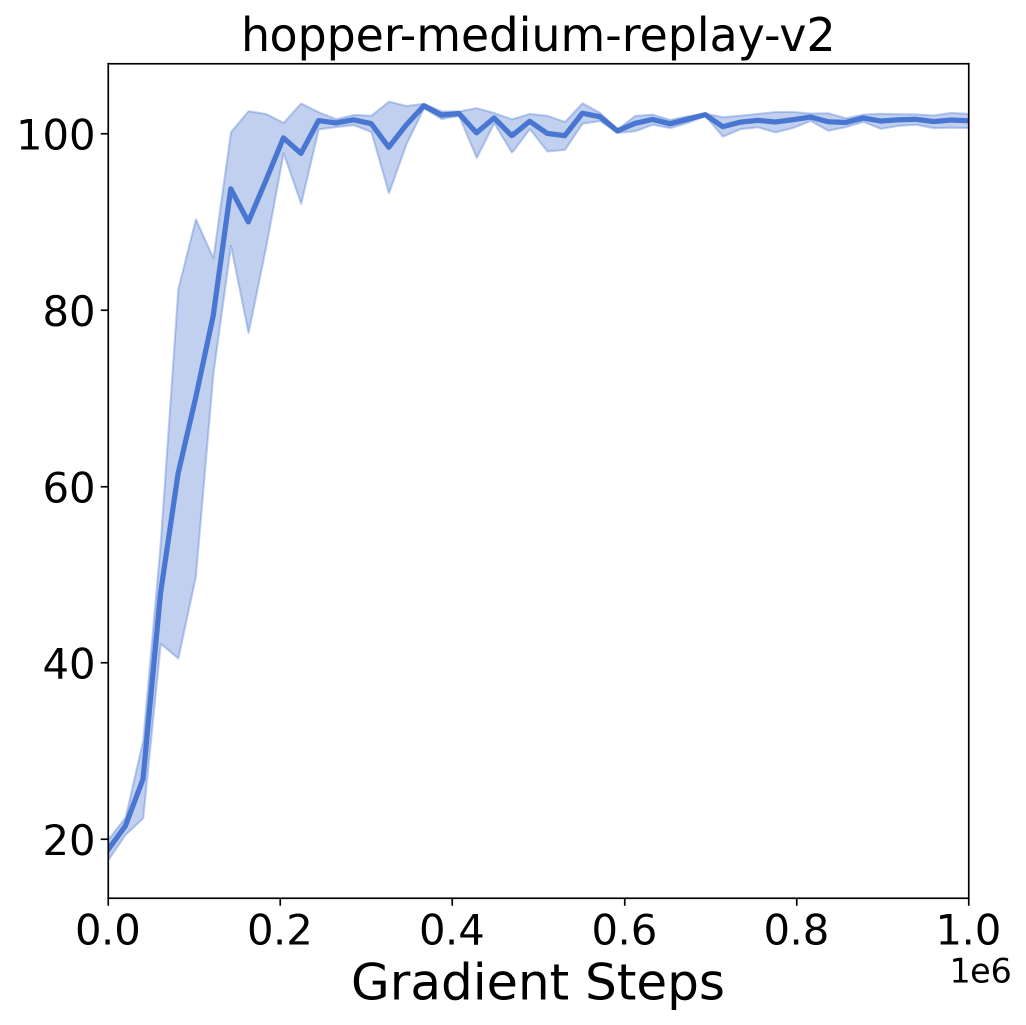} &
    \includegraphics[width=0.18\textwidth]{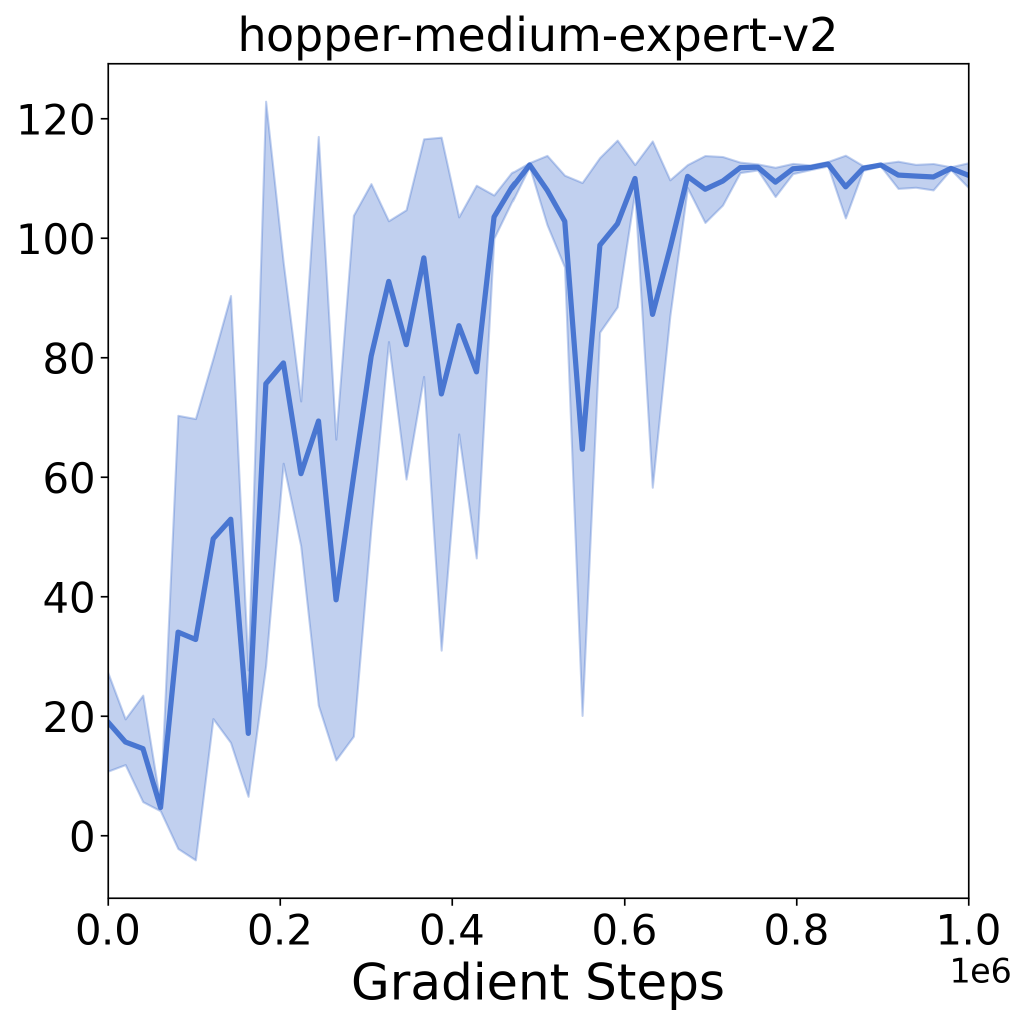} &
    \includegraphics[width=0.18\textwidth]{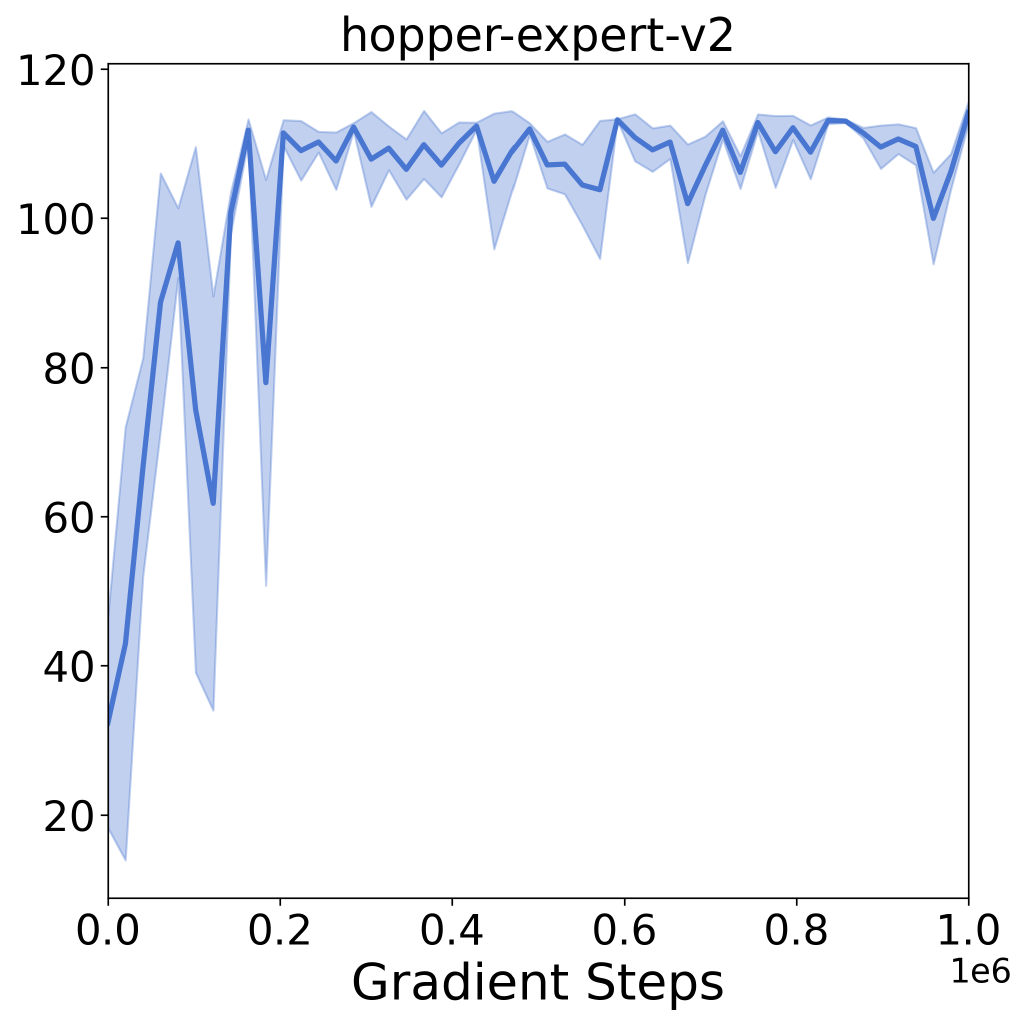} &
    \includegraphics[width=0.18\textwidth]{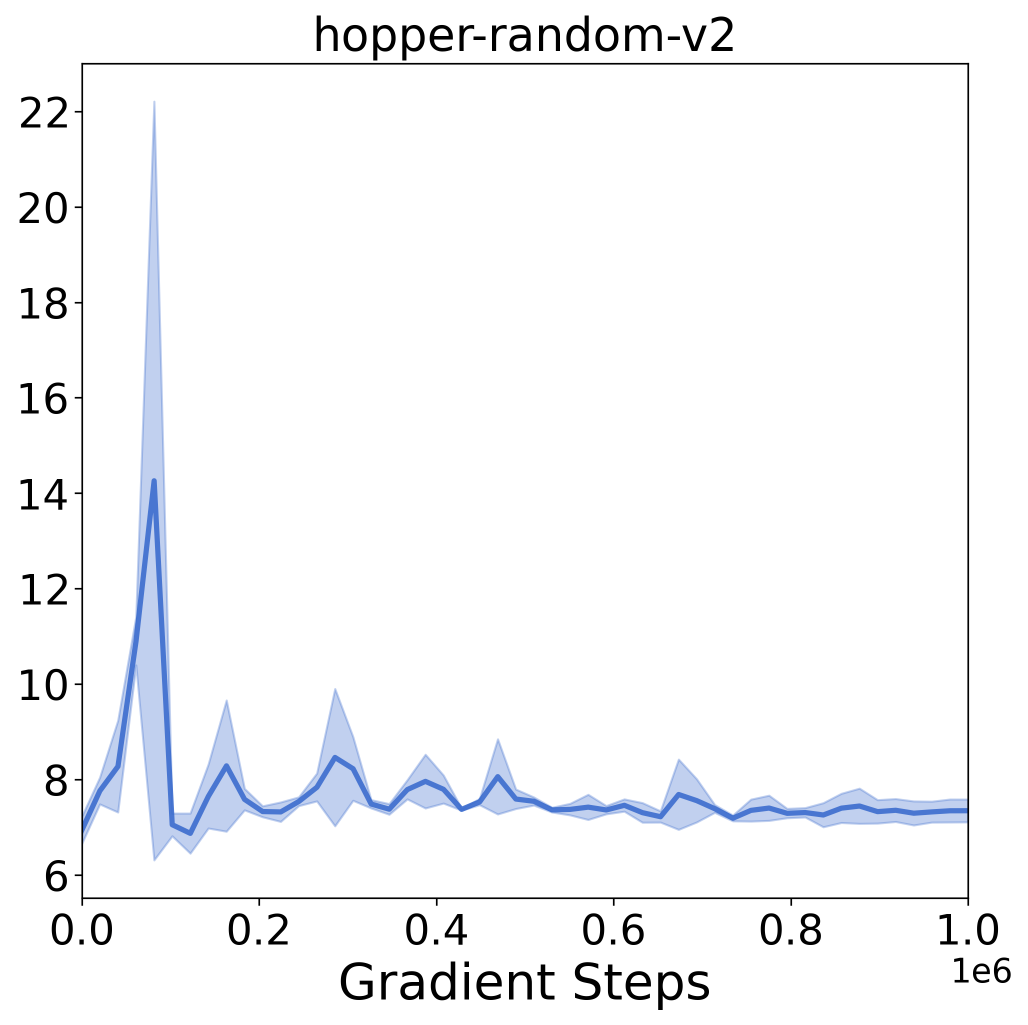} \\
    \includegraphics[width=0.18\textwidth]{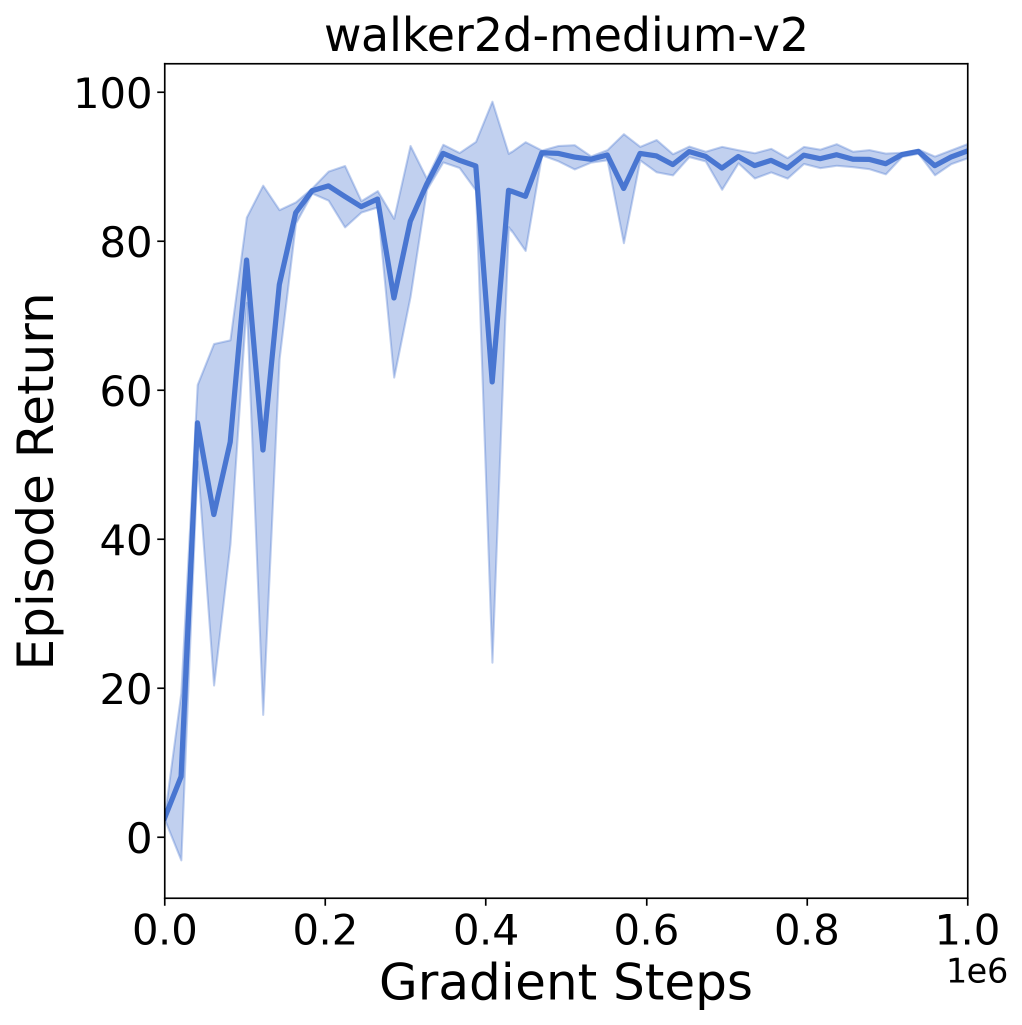} &
    \includegraphics[width=0.18\textwidth]{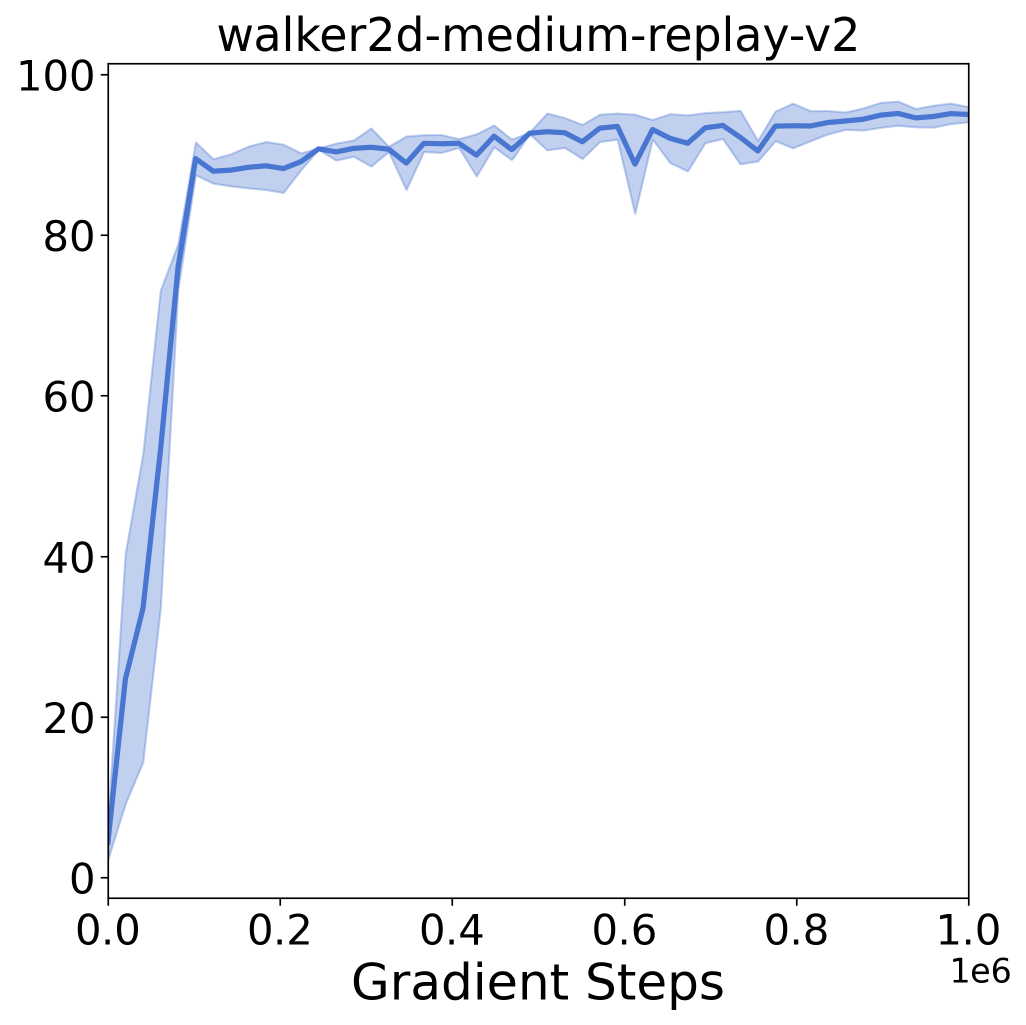} &
    \includegraphics[width=0.18\textwidth]{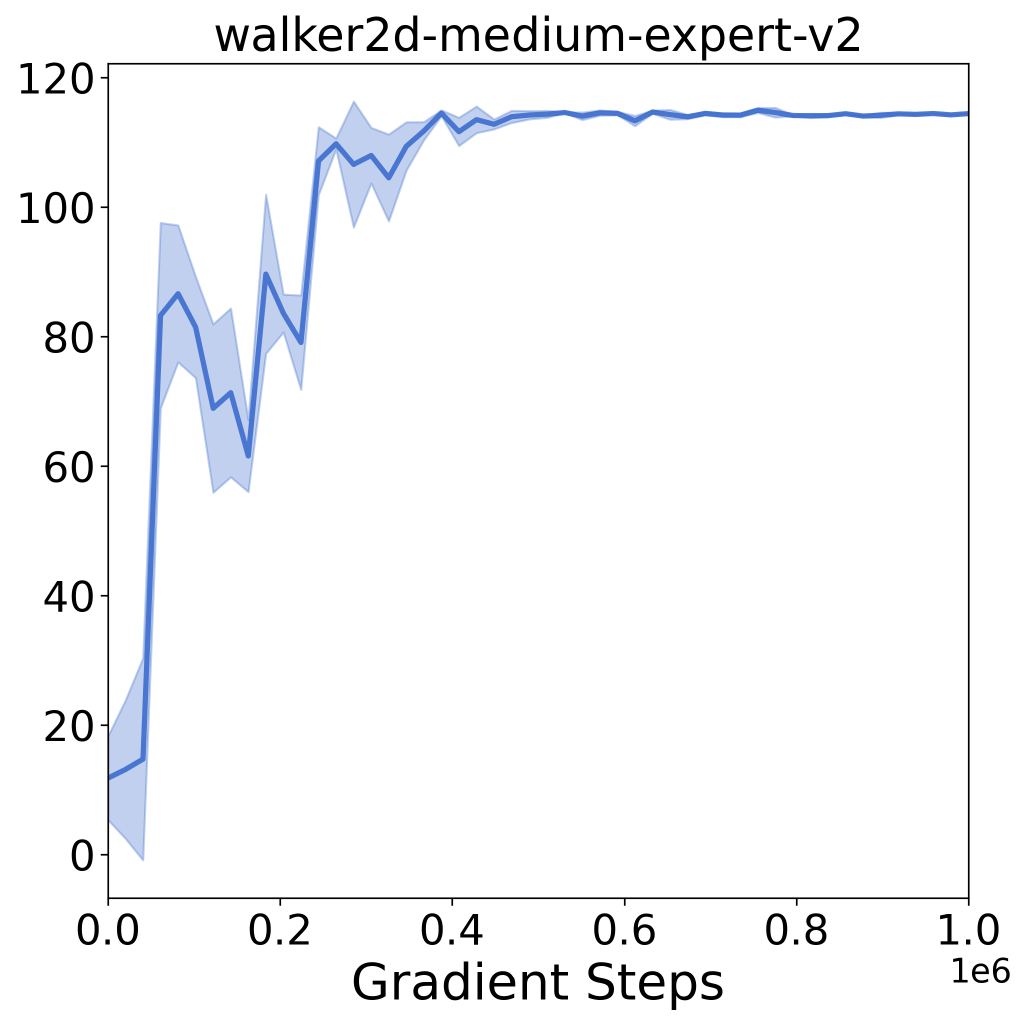} &
    \includegraphics[width=0.18\textwidth]{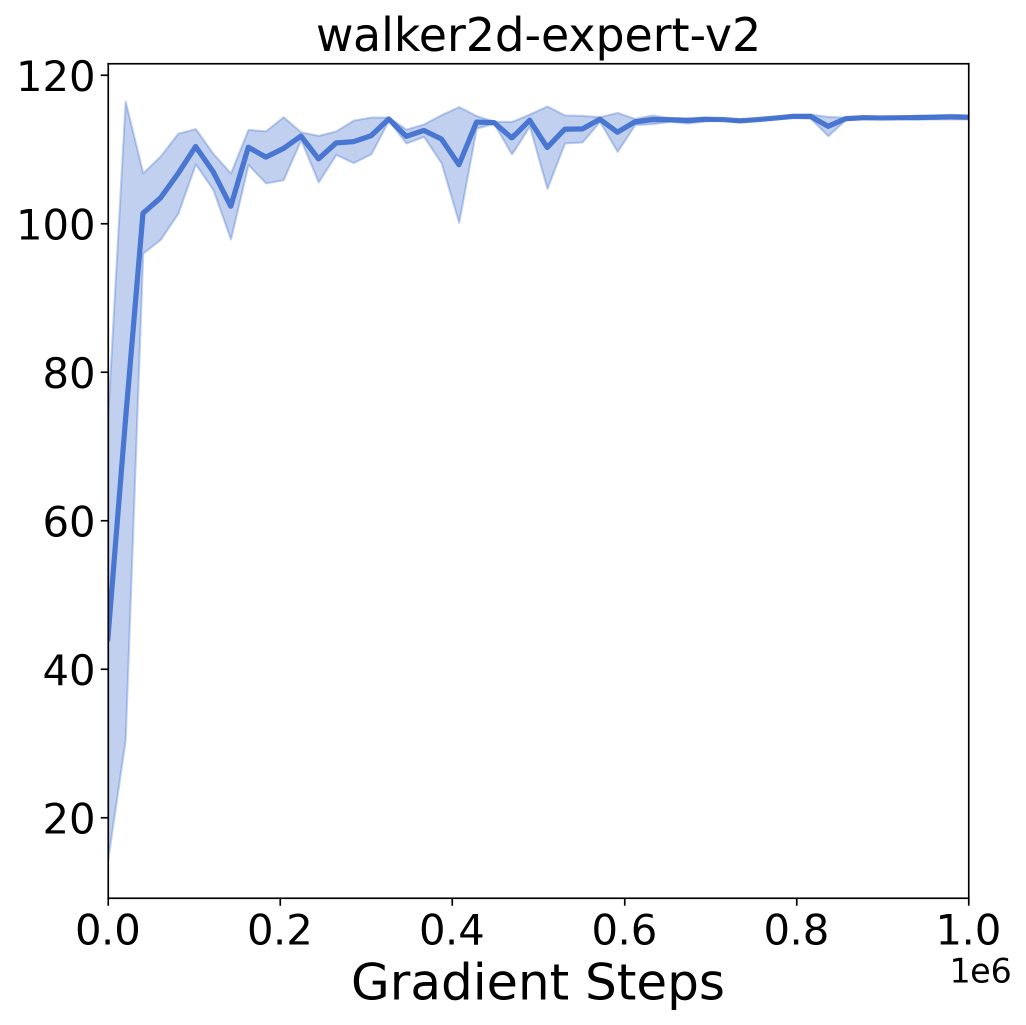} &
    \includegraphics[width=0.18\textwidth]{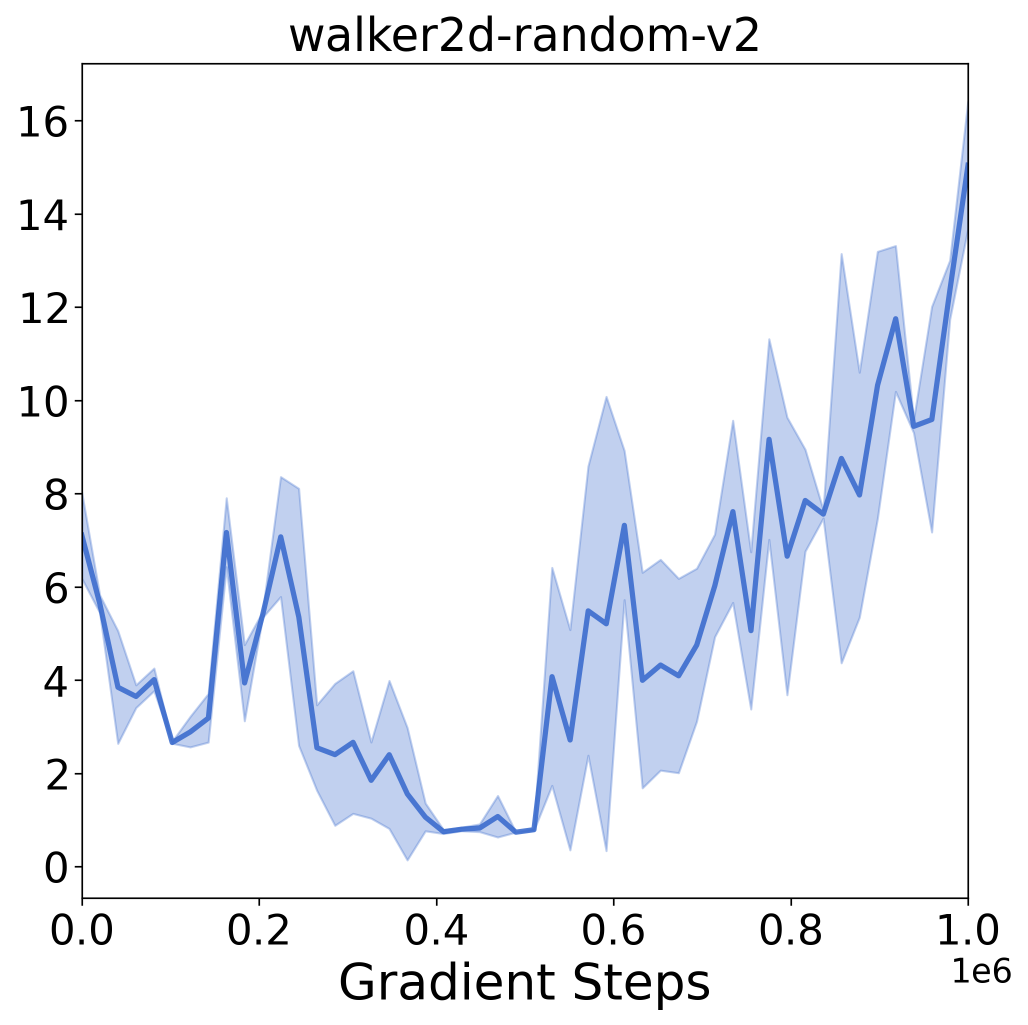}
  \end{tabular}
  \caption{Learning curves of CSDG on the Gym locomotion tasks.}
  \label{fig:dataset1}
\end{figure}

\begin{figure}[ht]
  \centering
  \begin{tabular}{ccc}
    \includegraphics[width=0.3\textwidth]{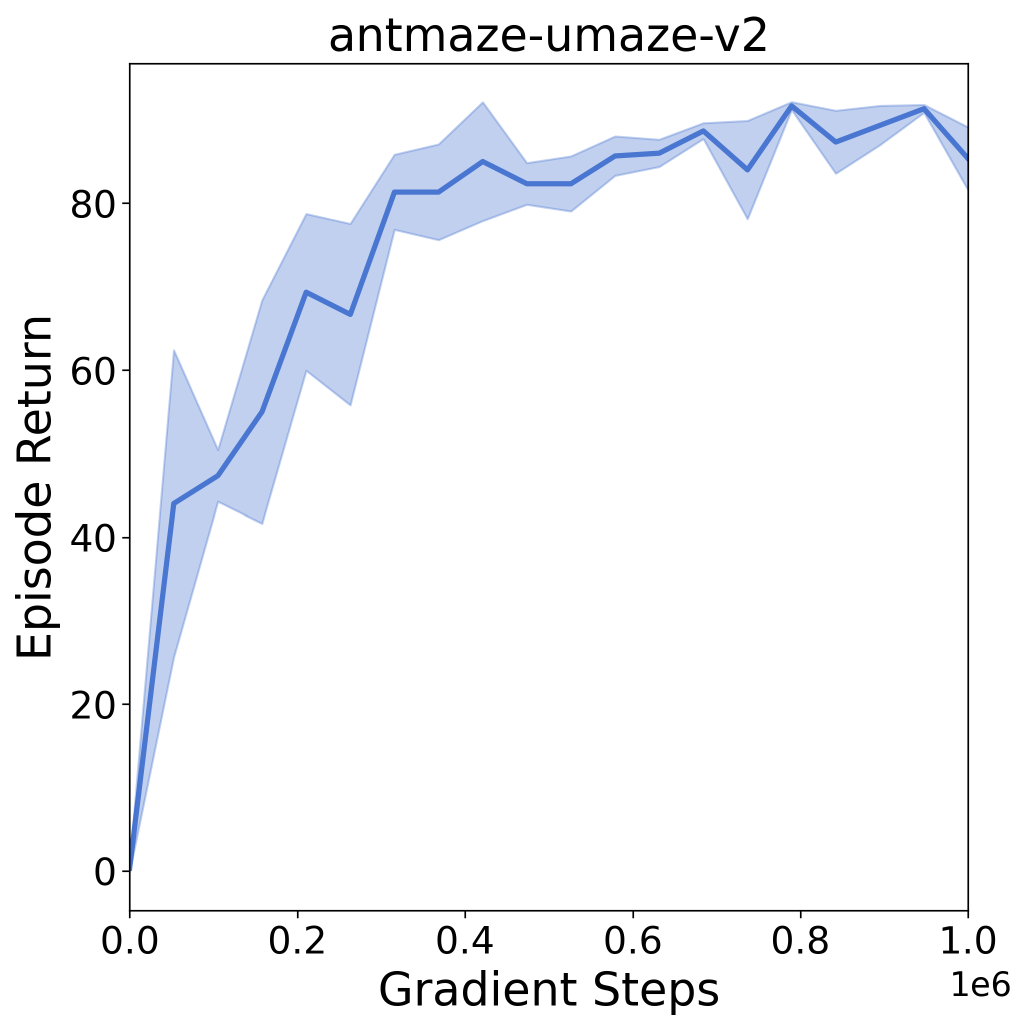} &
    \includegraphics[width=0.3\textwidth]{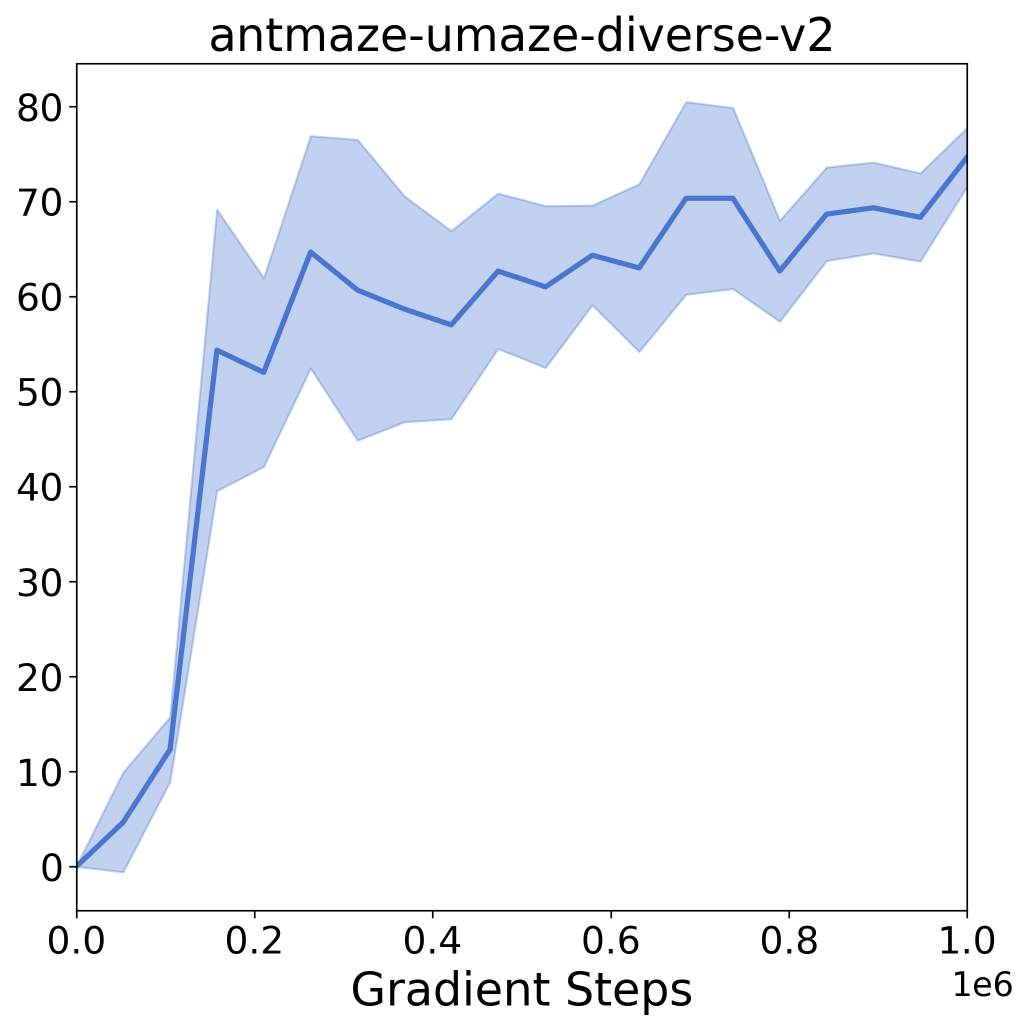} &
    \includegraphics[width=0.3\textwidth]{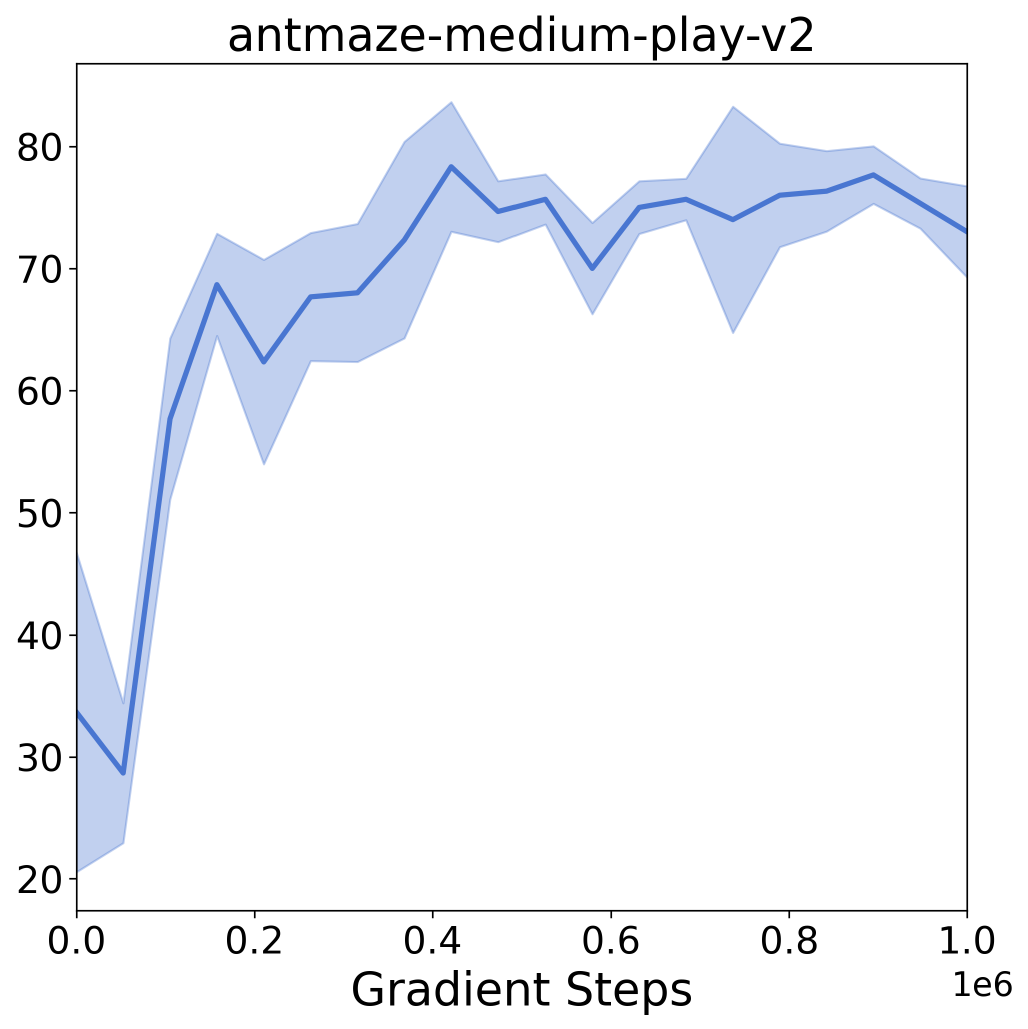} \\
    \includegraphics[width=0.3\textwidth]{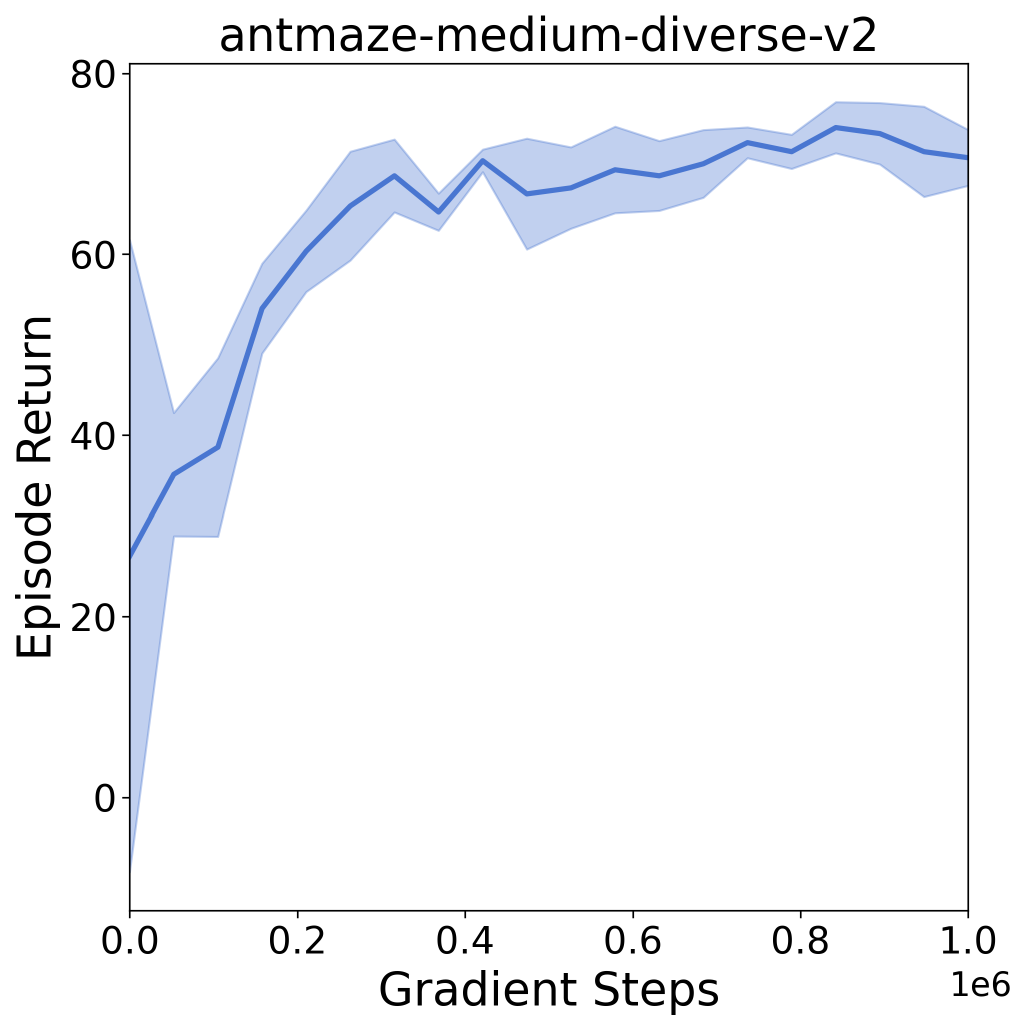} &
    \includegraphics[width=0.3\textwidth]{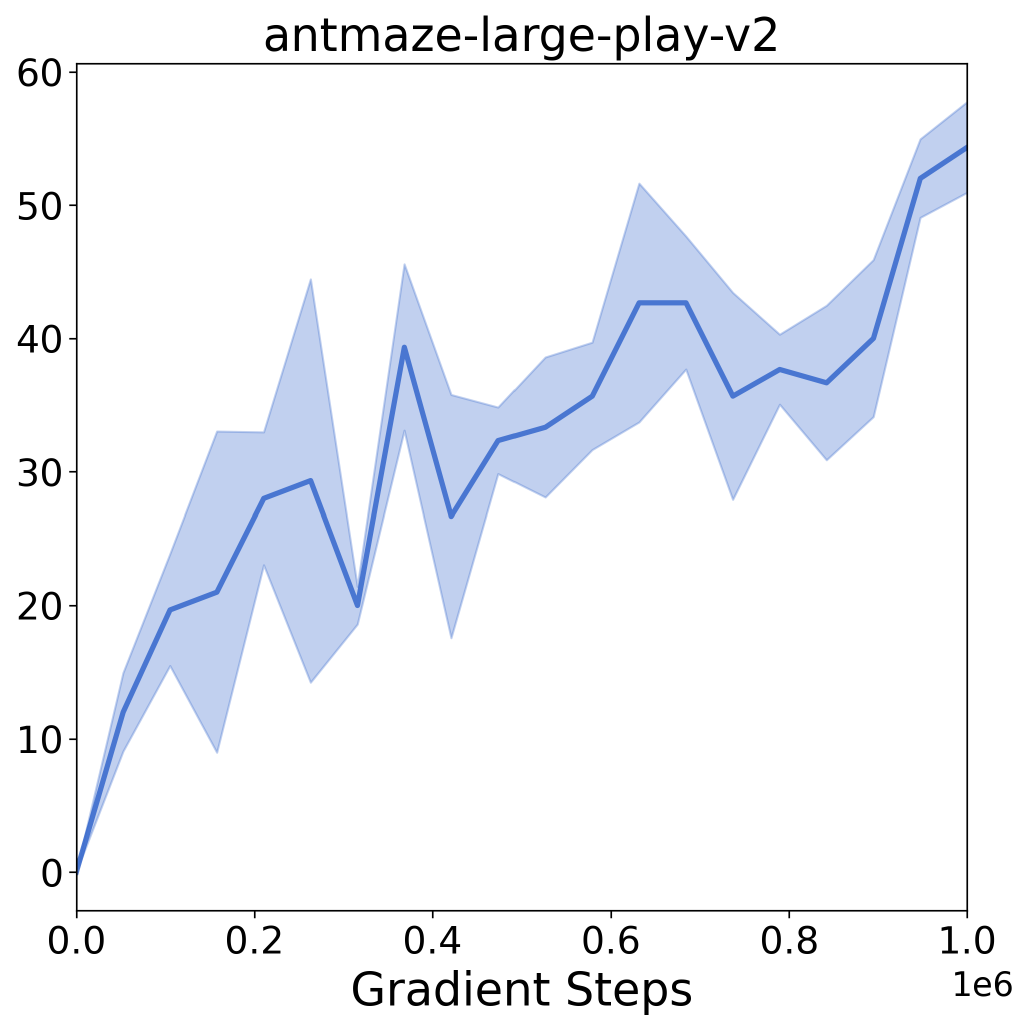} &
    \includegraphics[width=0.3\textwidth]{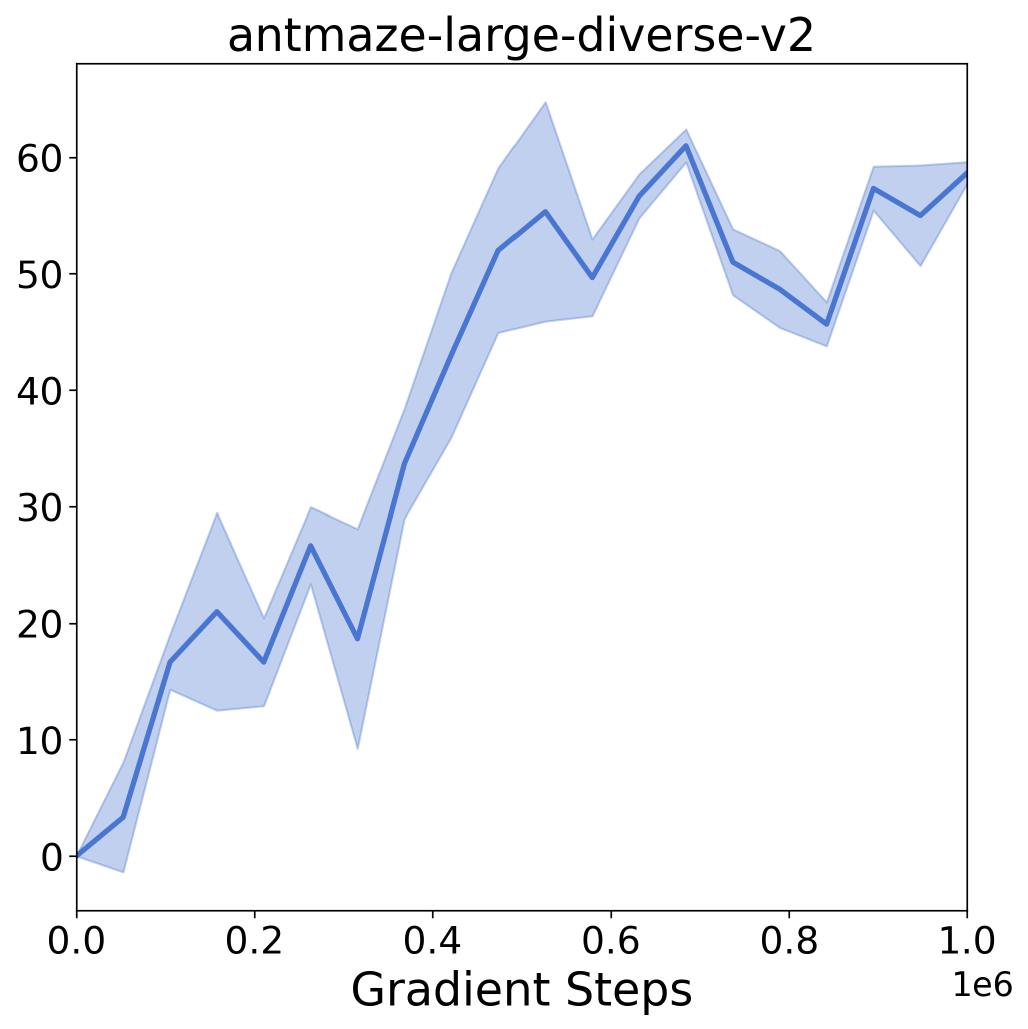}
  \end{tabular}
  \caption{Learning curves of CSDG on the AntMaze tasks.}
  \label{fig:dataset2}
\end{figure}

\paragraph{Robustness to observation and action perturbations.}
CSDG retains higher returns than IQL, TD3BC, and CQL as observation or action
noise increases in the three medium datasets. Figure~\ref{fig:app_robustness}
reports evaluations without retraining for Gaussian noise levels
$\sigma\in[0,0.6]$. The result complements the training-time noise study by
showing how the learned policy behaves under perturbations at evaluation time.
\newpage
\newpage
\section{Proofs of the Generalized Correction Analysis}
\label{app_sec:proofs}

This section proves the operator, fixed-point, and policy results in the main
paper. The target action selector and the two perturbation kernels remain fixed
throughout the analysis.

\subsection{Non-expansiveness and Contraction}

\paragraph{Well-posedness.}
For bounded measurable $Q$, define
\[
F_Q(s,v)
=
\mathbb E_{a\sim\hat\beta_k(\cdot|s)}
\bigl[h_\tau(Q(s,a)-v)\bigr],
\qquad
h_\tau(u)=w_\tau(u)u.
\]
The measurable nearest-neighbor selectors make $F_Q(\cdot,v)$ measurable for
every $v$, while $F_Q(s,\cdot)$ is continuous and strictly decreasing. Its
unique zero can be written as
$V_\tau^Q(s)=\inf\{q\in\mathbb Q:F_Q(s,q)\le0\}$, so $V_\tau^Q$ is measurable.
It lies between the smallest and largest local Q-values and is therefore
bounded. Kernel integration shows that $\mathcal S_\mu Q$ is also bounded and
measurable.
Moreover, the normalizer $Z_\tau^Q$ is positive and measurable, which makes
$\pi_\tau^Q$ a measurable action kernel. The fixed perturbation kernels also
define a measurable mixture kernel $\kappa_\mu$. Thus, both Bellman operators map
$\mathbb B_b(\mathcal S\times\mathcal A)$ into itself, and their induced
fixed-point policies are well defined.

\begin{theorem}[Non-expansive branch functionals]
\label{lem:app_nonexpansive}
For bounded $Q_1,Q_2$,
\[
\|\mathcal S_\mu Q_1-\mathcal S_\mu Q_2\|_\infty
\le \|Q_1-Q_2\|_\infty,
\qquad
\|V_\tau^{Q_1}-V_\tau^{Q_2}\|_\infty
\le \|Q_1-Q_2\|_\infty.
\]
Consequently, $\mathcal T_{\mathrm{CSDG}}$ and
$\mathcal T_{\mathrm{In}}^\tau$ are $\gamma$-contractions on
$\mathbb B_b(\mathcal S\times\mathcal A)$.
\end{theorem}

\begin{proof}
Let $\delta=\|Q_1-Q_2\|_\infty$. For the smoothed branch, both functions are
evaluated at the same candidates generated by the fixed perturbation kernels.
For every state $s$, the triangle inequality gives
\[
\begin{aligned}
&\left|\mathcal S_\mu Q_1(s)-\mathcal S_\mu Q_2(s)\right|\\
&\quad\le
\mu\,\mathbb E_{\eta_{\mathrm{In}}}
\left|Q_1(s,a_{\mathrm{In}})-Q_2(s,a_{\mathrm{In}})\right|\\
&\qquad+
(1-\mu)\,\mathbb E_{\eta_{\mathrm{OOD}}}
\left|Q_1(s,a_{\mathrm{OOD}})-Q_2(s,a_{\mathrm{OOD}})\right|\\
&\quad\le \delta.
\end{aligned}
\]
Taking the supremum over $s$ proves the first non-expansive bound.

For the expectile map, recall
$F_Q(s,v)=\mathbb E_{a\sim\hat\beta_k}
[h_\tau(Q(s,a)-v)]$, where $h_\tau(u)=w_\tau(u)u$. Since
$\tau\in(0,1)$, the function $h_\tau$ is continuous and strictly increasing.
Consequently, $F_Q(s,v)$ is strictly decreasing in $v$. If
$Q_1(s,\cdot)\le Q_2(s,\cdot)$, then monotonicity of $h_\tau$ gives
$F_{Q_1}(s,v)\le F_{Q_2}(s,v)$ for every $v$. Evaluating this inequality at
the zero $V_\tau^{Q_2}(s)$ yields
$F_{Q_1}(s,V_\tau^{Q_2}(s))\le0$. Since $F_{Q_1}(s,\cdot)$ is strictly
decreasing and vanishes at $V_\tau^{Q_1}(s)$, it follows that
$V_\tau^{Q_1}(s)\le V_\tau^{Q_2}(s)$. Thus, the expectile functional is
monotone.

The same first-order equation also gives translation equivariance. For any
constant $c$,
$F_{Q+c}(s,v+c)=F_Q(s,v)$, and hence
$V_\tau^{Q+c}(s)=V_\tau^Q(s)+c$. By the definition of $\delta$,
$Q_2-\delta\le Q_1\le Q_2+\delta$. Applying monotonicity and translation
equivariance gives
\[
V_\tau^{Q_2}(s)-\delta
\le V_\tau^{Q_1}(s)
\le V_\tau^{Q_2}(s)+\delta.
\]
Therefore,
$\|V_\tau^{Q_1}-V_\tau^{Q_2}\|_\infty\le\delta$.

Finally, the target functional of $\mathcal T_{\mathrm{CSDG}}$ is
$(1-\lambda)V_\tau^Q+\lambda\mathcal S_\mu Q$. Combining the two bounds and
using that transition expectations are non-expansive in the supremum norm,
\[
\begin{aligned}
\|\mathcal T_{\mathrm{CSDG}}Q_1
-\mathcal T_{\mathrm{CSDG}}Q_2\|_\infty
&\le
\gamma\bigl[(1-\lambda)
\|V_\tau^{Q_1}-V_\tau^{Q_2}\|_\infty\\
&\qquad\quad+
\lambda\|\mathcal S_\mu Q_1-\mathcal S_\mu Q_2\|_\infty\bigr]\\
&\le \gamma\|Q_1-Q_2\|_\infty.
\end{aligned}
\]
The same argument with only the expectile branch proves the contraction of
$\mathcal T_{\mathrm{In}}^\tau$.
\end{proof}

\subsection{Proof of Theorem~\ref{th:unified_bound}}

\begin{proof}
Using the operator definition, write the CSDG update as
\[
\mathcal T_{\mathrm{CSDG}}Q(s,a)
=R(s,a)+\gamma\,
\mathbb E_{s'\sim P(\cdot|s,a)}
\left[V_\tau^Q(s')+
\lambda\bigl(\mathcal S_\mu Q(s')-V_\tau^Q(s')\bigr)\right].
\]
The in-sample operator contains the same reward and expectile terms,
\[
\mathcal T_{\mathrm{In}}^\tau Q(s,a)
=R(s,a)+\gamma\,
\mathbb E_{s'\sim P(\cdot|s,a)}[V_\tau^Q(s')].
\]
Subtracting the second identity from the first cancels these shared terms and
leaves
\[
\mathcal T_{\mathrm{CSDG}}Q(s,a)
-\mathcal T_{\mathrm{In}}^\tau Q(s,a)
=
\gamma\lambda\,
\mathbb E_{s'\sim P(\cdot|s,a)}
\bigl[\mathcal S_\mu Q(s')-V_\tau^Q(s')\bigr].
\]
This is the signed one-step identity. Moreover,
\[
\begin{aligned}
&\left|\mathcal T_{\mathrm{CSDG}}Q(s,a)
-\mathcal T_{\mathrm{In}}^\tau Q(s,a)\right|\\
&\quad\le
\gamma\lambda\,
\mathbb E_{s'\sim P(\cdot|s,a)}
\left|\mathcal S_\mu Q(s')-V_\tau^Q(s')\right|\\
&\quad\le \gamma\lambda e(Q).
\end{aligned}
\]
Taking the supremum over $(s,a)$ proves
Eq.~\eqref{eq:one_step_correction_bound}.
\end{proof}

\subsection{Proof of Theorem~\ref{th:deviation}}

\begin{proof}
Set $D_k=\|Q_{\mathrm{CSDG}}^k-Q_{\mathrm{In}}^k\|_\infty$. Insert
$\mathcal T_{\mathrm{In}}^\tau Q_{\mathrm{CSDG}}^k$ between the two updates.
The triangle inequality gives
\[
\begin{aligned}
D_{k+1}
&=
\|\mathcal T_{\mathrm{CSDG}}Q_{\mathrm{CSDG}}^k
-\mathcal T_{\mathrm{In}}^\tau Q_{\mathrm{In}}^k\|_\infty\\
&\le
\|\mathcal T_{\mathrm{CSDG}}Q_{\mathrm{CSDG}}^k
-\mathcal T_{\mathrm{In}}^\tau Q_{\mathrm{CSDG}}^k\|_\infty\\
&\quad+
\|\mathcal T_{\mathrm{In}}^\tau Q_{\mathrm{CSDG}}^k
-\mathcal T_{\mathrm{In}}^\tau Q_{\mathrm{In}}^k\|_\infty\\
&\le \gamma\lambda e_k+\gamma D_k,
\end{aligned}
\]
where the last line uses Theorem~\ref{th:unified_bound} for the first term and
Theorem~\ref{lem:app_nonexpansive} for the second. Since $D_0=0$, the first two
steps satisfy
$D_1\le\gamma\lambda e_0$ and
$D_2\le\gamma\lambda(e_1+\gamma e_0)$. Repeated substitution therefore gives
\[
D_k
\le
\gamma\lambda
\sum_{t=0}^{k-1}\gamma^{k-1-t}e_t.
\]
This is Eq.~\eqref{eq:deviation_bound}.

For the fixed points, insert
$\mathcal T_{\mathrm{In}}^\tau Q_{\mathrm{CSDG}}^*$ and use the same two
bounds. Since each fixed point equals its corresponding Bellman update,
\[
\begin{aligned}
\|Q_{\mathrm{CSDG}}^*-Q_{\mathrm{In}}^*\|_\infty
&\le
\|\mathcal T_{\mathrm{CSDG}}Q_{\mathrm{CSDG}}^*
-\mathcal T_{\mathrm{In}}^\tau Q_{\mathrm{CSDG}}^*\|_\infty\\
&\quad+
\|\mathcal T_{\mathrm{In}}^\tau Q_{\mathrm{CSDG}}^*
-\mathcal T_{\mathrm{In}}^\tau Q_{\mathrm{In}}^*\|_\infty\\
&\le
\gamma\lambda e_*+
\gamma\|Q_{\mathrm{CSDG}}^*-Q_{\mathrm{In}}^*\|_\infty.
\end{aligned}
\]
Moving the final term to the left yields
$(1-\gamma)\|Q_{\mathrm{CSDG}}^*-Q_{\mathrm{In}}^*\|_\infty
\le\gamma\lambda e_*$, which proves the first bound in
Eq.~\eqref{eq:fixed_point_residual_bounds}.

For any bounded $\bar Q$, insert
$\mathcal T_{\mathrm{CSDG}}\bar Q$ and
$\mathcal T_{\mathrm{In}}^\tau\bar Q$ between $\bar Q$ and
$Q_{\mathrm{In}}^*$. More precisely,
\[
\begin{aligned}
\bar Q-Q_{\mathrm{In}}^*
={}&(\bar Q-\mathcal T_{\mathrm{CSDG}}\bar Q)\\
&+(\mathcal T_{\mathrm{CSDG}}\bar Q
-\mathcal T_{\mathrm{In}}^\tau\bar Q)\\
&+(\mathcal T_{\mathrm{In}}^\tau\bar Q
-\mathcal T_{\mathrm{In}}^\tau Q_{\mathrm{In}}^*).
\end{aligned}
\]
Taking norms, applying the residual definition and
Theorem~\ref{th:unified_bound}, and then using the contraction of
$\mathcal T_{\mathrm{In}}^\tau$ gives
\[
\|\bar Q-Q_{\mathrm{In}}^*\|_\infty
\le
r_{\mathrm{CSDG}}(\bar Q)
+\gamma\lambda e(\bar Q)
+\gamma\|\bar Q-Q_{\mathrm{In}}^*\|_\infty.
\]
After moving the last term to the left and dividing by $1-\gamma$, we obtain
the residual bound in Eq.~\eqref{eq:fixed_point_residual_bounds}.
\end{proof}

\subsection{Expectile-Induced Policies}

\begin{theorem}[Expectile policy representation]
\label{lem:app_expectile_policy}
For bounded $Q$, define
\[
Z_\tau^Q(s)
=
\mathbb E_{a\sim\hat\beta_k(\cdot|s)}
\bigl[w_\tau(Q(s,a)-V_\tau^Q(s))\bigr].
\]
Then $Z_\tau^Q(s)>0$, and the normalized policy
\[
\pi_\tau^Q(da|s)
=
\frac{w_\tau(Q(s,a)-V_\tau^Q(s))}{Z_\tau^Q(s)}
\hat\beta_k(da|s)
\]
satisfies
$V_\tau^Q(s)=\mathbb E_{a\sim\pi_\tau^Q(\cdot|s)}[Q(s,a)]$.
\end{theorem}

\begin{proof}
The weight $w_\tau$ is bounded below by $\min\{\tau,1-\tau\}>0$, so the
normalizer is positive. It is also finite because $w_\tau$ takes only the two
values $\tau$ and $1-\tau$. The proposed kernel is normalized since
\[
\int_{\mathcal A}\pi_\tau^Q(da|s)
=
\frac{1}{Z_\tau^Q(s)}
\mathbb E_{a\sim\hat\beta_k(\cdot|s)}
\bigl[w_\tau(Q(s,a)-V_\tau^Q(s))\bigr]
=1.
\]
Thus, $\pi_\tau^Q(\cdot|s)$ is a probability measure. The expectile
first-order condition is
\[
\mathbb E_{a\sim\hat\beta_k}
\bigl[w_\tau(Q(s,a)-V_\tau^Q(s))
(Q(s,a)-V_\tau^Q(s))\bigr]=0.
\]
Dividing by $Z_\tau^Q(s)$ and using the definition of $\pi_\tau^Q$ gives
\[
\begin{aligned}
0
&=
\mathbb E_{a\sim\pi_\tau^Q(\cdot|s)}
\bigl[Q(s,a)-V_\tau^Q(s)\bigr]\\
&=
\mathbb E_{a\sim\pi_\tau^Q(\cdot|s)}[Q(s,a)]-V_\tau^Q(s).
\end{aligned}
\]
Rearranging proves the stated expectation identity.
\end{proof}

\subsection{Proof of Theorem~\ref{th:policy_performance_comparison}}

\begin{proof}
At $Q_{\mathrm{In}}^*$, Theorem~\ref{lem:app_expectile_policy} gives
$V_\tau^{Q_{\mathrm{In}}^*}(s)
=\mathbb E_{a\sim\pi_{\mathrm{In}}^*(\cdot|s)}
[Q_{\mathrm{In}}^*(s,a)]$. Substituting this identity into the in-sample
fixed-point equation yields
\[
Q_{\mathrm{In}}^*(s,a)
=R(s,a)+\gamma\,
\mathbb E_{s'\sim P(\cdot|s,a),\,
b\sim\pi_{\mathrm{In}}^*(\cdot|s')}
[Q_{\mathrm{In}}^*(s',b)].
\]
This is the Bellman evaluation equation for $\pi_{\mathrm{In}}^*$.

For the CSDG fixed point, the expectile representation and
$\mathcal S_\mu Q(s)=
\mathbb E_{a\sim\kappa_\mu(\cdot|s)}[Q(s,a)]$ give
\[
\begin{aligned}
Q_{\mathrm{CSDG}}^*(s,a)
=R(s,a)+\gamma\,
\mathbb E_{s'\sim P(\cdot|s,a)}\bigl[&
(1-\lambda)
\mathbb E_{b\sim\pi_\tau^{Q_{\mathrm{CSDG}}^*}(\cdot|s')}
Q_{\mathrm{CSDG}}^*(s',b)\\
&+\lambda
\mathbb E_{b\sim\kappa_\mu(\cdot|s')}
Q_{\mathrm{CSDG}}^*(s',b)\bigr].
\end{aligned}
\]
By the definition of
$\pi_{\mathrm{CSDG}}^*=(1-\lambda)
\pi_\tau^{Q_{\mathrm{CSDG}}^*}+\lambda\kappa_\mu$, the bracketed expression
is the expectation under $\pi_{\mathrm{CSDG}}^*$. Hence, this is the Bellman
evaluation equation for $\pi_{\mathrm{CSDG}}^*$. Bellman evaluation is a
$\gamma$-contraction and has a unique bounded fixed point. Therefore,
$Q_{\mathrm{In}}^*=Q^{\pi_{\mathrm{In}}^*}$ and
$Q_{\mathrm{CSDG}}^*=Q^{\pi_{\mathrm{CSDG}}^*}$.

The normalized discounted occupancy used below is
\[
d_\pi^\gamma(B)
=
(1-\gamma)
\sum_{t=0}^{\infty}
\gamma^t
\Pr_\pi(s_t\in B\mid s_0\sim d_0).
\]
Apply the performance-difference lemma with
$\pi=\pi_{\mathrm{CSDG}}^*$ and reference policy
$\pi_{\mathrm{In}}^*$. With the normalized occupancy above, it gives
\[
J(\pi_{\mathrm{CSDG}}^*)-J(\pi_{\mathrm{In}}^*)
=
\frac{1}{1-\gamma}
\mathbb E_{\substack{
s\sim d_{\pi_{\mathrm{CSDG}}^*}^\gamma,\,
a\sim\pi_{\mathrm{CSDG}}^*(\cdot|s)}}
\bigl[A^{\pi_{\mathrm{In}}^*}(s,a)\bigr].
\]
For each state $s$, substituting the policy mixture separates the conditional
advantage expectation as
\[
\begin{aligned}
&\mathbb E_{a\sim\pi_{\mathrm{CSDG}}^*(\cdot|s)}
\bigl[A^{\pi_{\mathrm{In}}^*}(s,a)\bigr]\\
&\quad=
(1-\lambda)
\mathbb E_{a\sim\pi_\tau^{Q_{\mathrm{CSDG}}^*}(\cdot|s)}
\bigl[A^{\pi_{\mathrm{In}}^*}(s,a)\bigr]\\
&\qquad+
\lambda
\mathbb E_{a\sim\kappa_\mu(\cdot|s)}
\bigl[A^{\pi_{\mathrm{In}}^*}(s,a)\bigr].
\end{aligned}
\]
Averaging over
$s\sim d_{\pi_{\mathrm{CSDG}}^*}^\gamma$ gives
$(1-\lambda)\bar G_\tau+\lambda\bar G_\mu$. Substitution into the
performance-difference identity proves the equality in
Eq.~\eqref{eq:policy_performance_result}. Finally, the assumptions
$\bar G_\tau\ge-\xi$ and $\bar G_\mu\ge g_\mu$ imply
\[
\begin{aligned}
(1-\lambda)\bar G_\tau+\lambda\bar G_\mu
&\ge -(1-\lambda)\xi+\lambda g_\mu,\\
J(\pi_{\mathrm{CSDG}}^*)-J(\pi_{\mathrm{In}}^*)
&\ge
\frac{\lambda g_\mu-(1-\lambda)\xi}{1-\gamma}.
\end{aligned}
\]
The right-hand side is nonnegative whenever
$\lambda g_\mu\ge(1-\lambda)\xi$, which proves the final statement.
\end{proof}

\section{State-Conditional CHN Anchoring}
\label{chn}

The local action reference $\mathcal A_{\mathcal D,k}(s)$ is finite and
nonempty. Its convex hull
$C_k(s)=\operatorname{Conv}(\mathcal A_{\mathcal D,k}(s))$ is therefore
compact. Define the hull-to-reference coverage radius
\[
\kappa_k(s)
=
\sup_{c\in C_k(s)}
\operatorname{dist}
\bigl(c,\mathcal A_{\mathcal D,k}(s)\bigr).
\]

\begin{theorem}[CHN-to-anchor radius]
\label{prop:app_chn_anchor}
For every $a\in\mathcal A$,
\[
\operatorname{dist}
\bigl(a,\mathcal A_{\mathcal D,k}(s)\bigr)
\le
\operatorname{dist}(a,C_k(s))+\kappa_k(s).
\]
Consequently, $a\in\operatorname{CHN}_\delta(s)$ implies
$\operatorname{dist}(a,\mathcal A_{\mathcal D,k}(s))
\le\delta+\kappa_k(s)$.
\end{theorem}

\begin{proof}
Because $\mathcal A_{\mathcal D,k}(s)$ is finite and nonempty, its convex hull
$C_k(s)$ is nonempty and compact. The continuous map $c\mapsto\|a-c\|_2$
therefore attains its minimum on $C_k(s)$. Let $c^*$ be a minimizer, so that
$\|a-c^*\|_2=\operatorname{dist}(a,C_k(s))$.

The local action reference is finite, and hence the distance from $c^*$ to
$\mathcal A_{\mathcal D,k}(s)$ is also attained. Choose
$a_i^*\in\mathcal A_{\mathcal D,k}(s)$ such that
$\|c^*-a_i^*\|_2=
\operatorname{dist}(c^*,\mathcal A_{\mathcal D,k}(s))$. By the definition of
$\kappa_k(s)$, this distance is at most $\kappa_k(s)$. The triangle inequality
then gives
\[
\begin{aligned}
\operatorname{dist}
\bigl(a,\mathcal A_{\mathcal D,k}(s)\bigr)
&\le \|a-a_i^*\|_2\\
&\le \|a-c^*\|_2+\|c^*-a_i^*\|_2\\
&\le \operatorname{dist}(a,C_k(s))+\kappa_k(s).
\end{aligned}
\]
This proves the first claim. If $a\in\operatorname{CHN}_\delta(s)$, then
$\operatorname{dist}(a,C_k(s))\le\delta$. Substituting this bound proves
$\operatorname{dist}(a,\mathcal A_{\mathcal D,k}(s))
\le\delta+\kappa_k(s)$.
\end{proof}

The nearest-anchor selector in Definition~\ref{def:nearest_anchor_decomposition}
attains the distance on the left because the local reference is finite. The
main analysis evaluates the associated value difference directly through
$e_{\mathrm{geo}}(Q,s)$ and keeps the remaining anchor-to-expectile difference
in $e_{\mathrm{ref}}(Q,s)$. The radius above is an action-space statement;
$e_{\mathrm{geo}}(Q,s)$ also contains the state displacement and is therefore
retained explicitly in the operator bounds.

\section{Practical Noise Injection}
\label{PFNI}

The practical algorithm centers both perturbation branches at the target
action $\pi_{\phi'}(s)$. The following result relates their noise radii to the
state-conditional CHN and observed-action anchors.

\begin{theorem}[Perturbation and anchoring radii]
\label{prop:app_practical_anchoring}
Let $\bar a(s)\in\mathcal A$ satisfy
$\operatorname{dist}(\bar a(s),C_k(s))\le\rho_k(s)$. For
$j\in\{\mathrm{In},\mathrm{OOD}\}$, suppose $\|\eta_j\|_2\le\delta_j$ and
define $\tilde a_j=\Pi_\mathcal A(\bar a(s)+\eta_j)$. Then
\[
\operatorname{dist}(\tilde a_j,C_k(s))
\le\rho_k(s)+\delta_j,
\qquad
\operatorname{dist}
\bigl(\tilde a_j,\mathcal A_{\mathcal D,k}(s)\bigr)
\le\rho_k(s)+\delta_j+\kappa_k(s).
\]
\end{theorem}

\begin{proof}
Every $c\in C_k(s)$ lies in the convex action set $\mathcal A$, so
$\Pi_\mathcal A(c)=c$. Since Euclidean projection onto a closed convex set is
non-expansive, for every $c\in C_k(s)$,
$\|\Pi_\mathcal A(\bar a(s)+\eta_j)-\Pi_\mathcal A(c)\|_2
\le\|\bar a(s)+\eta_j-c\|_2$. Taking the infimum over $c$ gives
\[
\begin{aligned}
\operatorname{dist}(\tilde a_j,C_k(s))
&=
\inf_{c\in C_k(s)}
\|\Pi_\mathcal A(\bar a(s)+\eta_j)-\Pi_\mathcal A(c)\|_2\\
&\le
\inf_{c\in C_k(s)}\|\bar a(s)+\eta_j-c\|_2\\
&\le
\operatorname{dist}(\bar a(s),C_k(s))+\|\eta_j\|_2.
\end{aligned}
\]
The final line follows from the triangle inequality: for each $c\in C_k(s)$,
$\|\bar a(s)+\eta_j-c\|_2
\le\|\bar a(s)-c\|_2+\|\eta_j\|_2$, after which the infimum is taken over
$c$. Using
$\operatorname{dist}(\bar a(s),C_k(s))\le\rho_k(s)$ and
$\|\eta_j\|_2\le\delta_j$ proves
$\operatorname{dist}(\tilde a_j,C_k(s))
\le\rho_k(s)+\delta_j$.

Applying Theorem~\ref{prop:app_chn_anchor} to $\tilde a_j$ then gives
\[
\begin{aligned}
\operatorname{dist}
\bigl(\tilde a_j,\mathcal A_{\mathcal D,k}(s)\bigr)
&\le
\operatorname{dist}(\tilde a_j,C_k(s))+\kappa_k(s)\\
&\le \rho_k(s)+\delta_j+\kappa_k(s),
\end{aligned}
\]
which proves the second inequality.
\end{proof}

In the implementation, $\bar a(s')=\pi_{\phi'}(s')$ and
$\rho_k(s')=\operatorname{dist}(\pi_{\phi'}(s'),C_k(s'))$. If Gaussian noise
is clipped elementwise at magnitude $c_j$ in a $d_a$-dimensional action space,
then
\[
\|\eta_j\|_2^2
=\sum_{\ell=1}^{d_a}\eta_{j,\ell}^2
\le d_a c_j^2,
\qquad
\|\eta_j\|_2\le\sqrt{d_a}\,c_j.
\]
Thus, the theoretical radius can be taken as
$\delta_j=\sqrt{d_a}\,c_j$, and
$c_{\mathrm{In}}<c_{\mathrm{OOD}}$ gives the two ordered locality bounds used
by CSDG.

\end{document}